\documentclass{article} 

\usepackage{arxiv/arxiv}

\usepackage[dvipsnames]{xcolor}
\usepackage[utf8]{inputenc} 
\usepackage[T1]{fontenc}    
\usepackage{microtype}      
\usepackage{url}            
\usepackage{graphicx}
\usepackage{subcaption}
\usepackage{booktabs}       
\usepackage{multirow}
\usepackage{makecell}

\usepackage{floatrow}
\usepackage{algorithm}
\usepackage{algpseudocode}
\usepackage{hyperref}       
\usepackage{pifont}
\usepackage{enumitem}
\usepackage{nicefrac}       
\usepackage{caption}
\usepackage{subcaption}
\usepackage{tablefootnote}
\usepackage{soul}
\usepackage{wrapfig}
\usepackage{listings}
\usepackage{natbib}
\definecolor{codebg}{gray}{0.95}
\usepackage{tikz}
\usetikzlibrary{fit}
\usetikzlibrary{arrows.meta, positioning}
\tikzset{
  vtx/.style   = {circle, draw, inner sep=0pt, minimum size=5.4mm, font=\footnotesize},
  bnode/.style = {circle, draw, inner sep=0pt, minimum size=6.6mm, font=\footnotesize},
  term/.style  = {rectangle, draw, inner sep=0pt, minimum size=5.6mm, font=\footnotesize},
  lo/.style    = {dashed, -{Latex[length=1.5mm]}},
  hi/.style    = {-{Latex[length=1.5mm]}},
}

\usepackage{lib/researchpack}
\usepackage{lib/macros}

\hypersetup{
   colorlinks=true,
   linkcolor=MidnightBlue,
   citecolor=MidnightBlue,
   urlcolor=MidnightBlue
}

\usepackage[capitalize, nameinlink]{cleveref}
\usepackage{thmtools, thm-restate}
\declaretheorem[style=plain]{definition, theorem, proposition, lemma}

\theoremstyle{plain}

\newtheorem{corollary}{Corollary}

\newtheorem{example}{Example}

\newtheorem{assumption}{Assumption}

\crefname{algorithm}{Algorithm}{Alg.}
\crefname{definition}{Def.}{Defs.}
\Crefname{definition}{Def.}{Defs.}

\newif\ifshowcomments
\showcommentsfalse

\title{Beyond Trained Models: Compiling GNNs for a Sound Explainer Benchmark}

\author{
  Steve Azzolin\thanks{Equal contribution} \\ University of Trento \\ Trento, Italy \And
  Francesco Paolo Nerini\footnotemark[1]\\ Sapienza University \\ Rome, Italy \And
  Stefano Teso \\ University of Trento \\ Trento, Italy \And 
  Francesco Bonchi \\ Intesa Sanpaolo AI Research \\ Turin, Italy \And
  Bruno Lepri \\ Fondazione Bruno Kessler \\ Trento, Italy \And
  Andr\'e Panisson \\ Intesa Sanpaolo AI Research \\ Turin, Italy \And
  Andrea Passerini \\ University of Trento \\ Trento, Italy
}

\begin{document}

\maketitle

\begin{abstract}
    Explainers for Graph Neural Networks (GNNs) are commonly evaluated by their \textit{plausibility}, i.e., how well their explanations recover a predefined ground truth, such as a motif planted in the data.
    This protocol implicitly assumes that a GNN trained on such data relies on the intended motif.
    Although prior work has questioned this assumption, plausibility remains widespread.
    First, we show that the assumption is violated on several widely used benchmarks, where, e.g., degree statistics alone suffice to solve the task.
    Then, we remove this confounder by replacing training with \textit{compilation}.
    We achieve this by introducing $\mathsf{Gracr}$, the first compiler translating graded modal logic formulas into  GNN weights, yielding models that replicate the behaviour of the corresponding formulas.
    Since the behaviour of the model is now known by construction, we can define its ground truth explanation formally and compute it exactly.
    Building on this, we introduce $\mathsf{Gracr}\mathsf{Bench}$, a benchmark of compiled GNNs for the evaluation of explainers against this exact ground truth.
    Experiments on eleven explainers across six tasks show its effectiveness for fine-grained diagnostic evaluation: notably, we discover that most explainers are not robust to indirect influences or alternative implementations of the same formula.
    These results position $\mathsf{Gracr}\mathsf{Bench}$ as a novel, rigorous evaluation setting for graph post-hoc explainability.
\end{abstract}
\section{Introduction}
\label{sec:intro}

Explaining why a Machine Learning (ML) model produces a particular output is a core challenge for deploying these models in the real world.
Graph Neural Networks (GNNs) \citep{scarselli2008graph, bronstein2021gdl} are no exception, and a rich literature has emerged to develop methods for explaining such models \citep{yuan2022explainability, amara2022graphframex, longa2024explaining}.
One of the most widely used techniques for evaluating the \textit{quality} of GNN explainers is \textbf{\textit{plausibility}}\footnote{A.k.a. explanation accuracy \citep{agarwal2023evaluating}}, i.e., how well the explanation recovers a ``ground-truth'' pattern. It follows this blueprint: \textit{First}, identify or create a dataset with a known discriminative pattern, e.g., graphs containing a house-like motif versus graphs containing a cycle-like motif; \textit{then}, train a GNN on this dataset; \textit{last}, evaluate how well the explanations produced by the explainer under evaluation recover the known discriminative pattern.

Previous work has already identified shortcomings of this blueprint \citep{faber2021comparing, himmelhuber2021demystifying, agarwal2023evaluating, fontanesi2024xai}.
Yet, despite these findings, our extensive literature review (see \cref{fig:eval-papers}) shows that plausibility remains widely used: strikingly, roughly $25\%$ of all $155$ papers we surveyed adopt plausibility as the \textit{only} criterion for evaluating explanations. 
The fundamental issue is that the trained GNN may learn other, potentially spurious patterns present in the data that differ from the intended discriminative pattern, as illustrated by the example in \cref{fig:counterpoint-example}. 
Consequently, we scrutinize the definition and usage of plausibility closely.

\begin{wrapfigure}[25]{r}{0.4\textwidth}
    \centering
    \includegraphics[width=\linewidth]{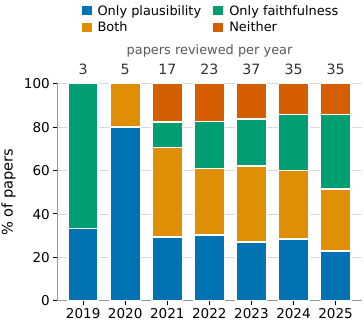}
    \caption{
        We manually reviewed $155$ published GNN explainability papers ($2019$--$2025$), recording how they evaluate explanation quality. \textbf{Plausibility remains a core criterion}, used alone in at least $20\%$ of papers each year. Numbers above bars indicate the number of papers analysed; details in \cref{app:survey}.
    }
    \label{fig:eval-papers}
\end{wrapfigure}

\textbf{In our work}, we first extend previous findings by identifying specific, widely used datasets where the above shortcoming is likely to occur, and then propose a new strategy for evaluating plausibility that replaces GNN training with a \textit{compilation step}.
%
%
We begin by identifying a pattern of interest, represented as a Boolean logic formula. 
Then, we introduce a compiler to transform this logic formula into a GNN that, by construction, shares the same behaviour and explanations as the formula.
%
This raises the question of what the ``ground truth'' explanation of a formula actually is.
Common practice answers with a designated motif, but already for simple patterns -- such as \textit{a node having more red than blue neighbours} -- several subgraphs justify the prediction equally well. 
%
We therefore need a notion of ground truth that generalizes beyond simple motifs and accounts for alternative explanations.
To do so, we borrow from formal XAI the definition of \textit{prime implicant} \citep{darwiche2020reasons, marques2022delivering} -- \ie a minimal subgraph that alone suffices to determine the prediction -- and take the union of all prime implicants as the ground truth explanation.
Finally, the explainer is evaluated based on how well its explanations recover such ground truth.
%
%
Crucially, the compilation step allows us to control \textit{exactly} the behaviour of the GNN, as it is equivalent to the logic formula from which it was derived:
effects that are usually inextricably confounded on a trained model can be studied in controlled experiments, allowing specific hypotheses to be tested rigorously.
%

Summarizing, our contributions are as follows:
\begin{itemize}[leftmargin=1.25em]
    \item We show that popular benchmarks are not a reliable test bed for evaluating plausibility, as the GNN over which explainers are evaluated may not behave as expected.

    \item We propose a compiler capable of translating a logic formula into the weights of a GNN. This yields GNNs where the exact behaviour is known and controllable, addressing the issue above.

    \item To derive a principled notion of what characterizes a ground truth explanation, we build on formal XAI and adopt \textit{prime implicants}, and develop an efficient procedure to extract them.
    %

    \item We exploit this compiler to propose a new, sound benchmark for evaluating plausibility. Our experiments recover previously known findings while uncovering new ones, exposing limitations of explainers in understudied scenarios beyond motif-based tasks.
\end{itemize}

\begin{figure}
    \centering
    \includegraphics[width=1.\linewidth]{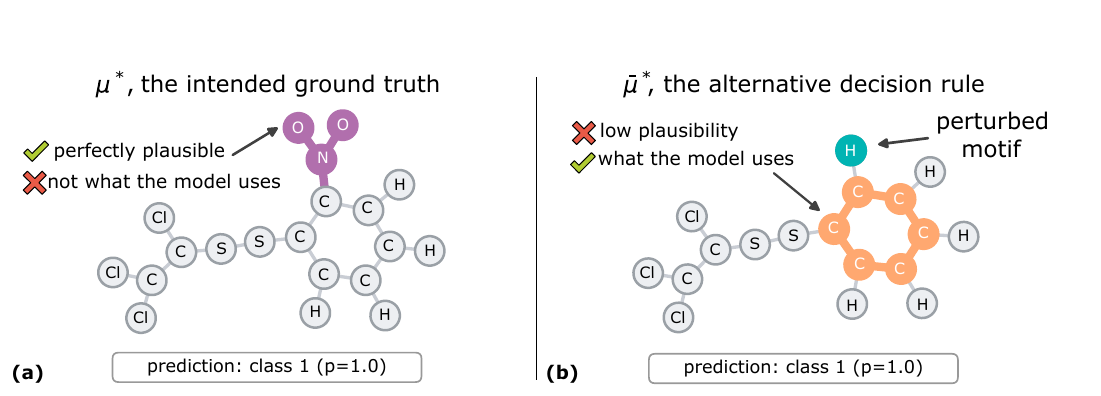}
    \caption{
    \textbf{Evaluating the plausibility of explanations is misleading when the model is not behaving as expected.} 
    Example of a model trained on \Mutagenicity, where plausibility penalizes well-behaved explainers. 
    a) A correctly predicted mutagenic molecule, with the intended ground truth $\mathrm{NO_2}$ pattern highlighted as the explanation $\mu^*$.
    b) The same molecule, where we substituted the original $\mu^*$ with a single hydrogen atom. Crucially, the model prediction does not change, suggesting it has learned an alternative decision rule $\bar{\mu}^*$ (e.g., the carbon ring).
    As a result, a good explainer correctly providing $\bar{\mu}^*$as the explanation would be penalized in favour of a bad one providing $\mu^*$.
    In this work, we fix this issue by building GNNs whose behaviour is known, as described in \cref{sec:compiled-gnn}.
    }
    \label{fig:counterpoint-example}
\end{figure}

\section{Background}

%
%
%
\textbf{Notation.}
A \textit{\textbf{graph}} is a tuple $G=(\calV,\calE)$ with nodes $\calV$ and directed edges $\calE \subseteq \calV \times \calV$, where an edge $e$ is an ordered pair of nodes $e=(u,v)$.
%
%
Throughout, a \textit{\textbf{subgraph}} $R\subseteq G$ is a \textit{spanning} subgraph obtained from $G$ by retaining all nodes in $\calV$ and a subset of its edges. 
We write $E(u,v)$ for the predicate holding iff $(u,v) \in \calE$, and $\calN(v)$ for the neighbours of node $v$.
We focus on binary classification tasks: a \textit{\textbf{graph classifier}} is a function $g$ which maps an input graph into $\{0,1\}$; a \textit{\textbf{node classifier}}, instead, maps each node of a graph into $\{0,1\}$.

\textbf{Problem statement.}
A \textbf{\textit{graph explainer}} is a function $F$ mapping a classifier and an input graph to an \textbf{\textit{explanation}}, i.e., a vector of edge importance scores $F(g, G) = \mu \in \bbR^{|\calE|}$, with $\mu_e$ the score of edge $e$. For node classification, $F$ also takes the node to explain.
%
%
The \textbf{\textit{ground truth explanation}} is a binary mask $\mu^* \in \{0,1\}^{|\calE|}$.
%
Finally, a \textbf{\textit{plausibility}} score is any scoring function -- like accuracy or variants thereof -- between the predicted explanation $\mu$ and the ground truth $\mu^*$, i.e., $\mathrm{sim}(\mu, \mu^*) \in [0,1]$ \citep{yuan2022explainability, agarwal2023evaluating}.
%
%
The choice of $\mu^*$ is crucial but often heuristic.
Next, we describe cases in which an inappropriate choice can make plausibility misleading.

\section{Issues with previous evaluations of plausibility}
\label{sec:issues}



Evaluating the plausibility of an explainer typically boils down to the following blueprint \citep{ agarwal2023evaluating}. First, identify or create a dataset $\calD$ of graphs containing a known discriminative pattern, which acts as the ground truth explanation $\mu^*$.
For example, in the \Mutagenicity dataset \citep{kazius2005derivation}, the presence of a $\mathrm{NO_2}$ group is usually associated with the mutagenicity that distinguishes the two classes and therefore serves as their ground-truth explanation \citep{ying2019gnnexplainer, luo2020parameterized}.
%
Then, the plausibility score between the post-hoc explanation $\mu$ and $\mu^*$ is computed.

As pointed out by previous work \citep{faber2021comparing, yuan2022explainability, agarwal2023evaluating, monti2024a}, this blueprint makes the following \textit{implicit} assumption, which may not hold in practice.

\begin{assumption}[informal]
\label{ass:shortcut}
    A GNN trained to achieve high accuracy on a dataset with a known ground-truth explanation $\mu^*$ will learn only the discriminative pattern represented by $\mu^*$.
\end{assumption}

In fact, ML models are known to be susceptible to learning shortcuts in the data \citep{geirhos2020shortcut}, which invalidates \cref{ass:shortcut}.
Arguably, explainers are most useful exactly when they can reveal that the model is not behaving as expected \citep{lapuschkin2019unmasking, teso2023leveraging}.
Nonetheless, \textbf{\textit{evaluating plausibility when \cref{ass:shortcut} does not hold can be misleading}}, as illustrated in \cref{fig:counterpoint-example}:
If \cref{ass:shortcut} held and the model had learned to distinguish molecule mutagenicity based solely on the ground truth $\mathrm{NO_2}$ pattern, we would expect its prediction to change when this pattern is removed, provided that the resulting molecule remains sufficiently close to the data distribution.
However, as exemplified in \cref{fig:counterpoint-example} and analysed in \cref{appx:chemical-datasets}, removing the $\mathrm{NO_2}$ pattern often leaves the prediction of the model unchanged, suggesting that \cref{ass:shortcut} does not hold, and that the model learned an alternative pattern.
%
In this case, a well-behaved explainer outputting the alternative pattern would be penalized in favour of an explainer outputting ${\mu^*}$. 

Crucially, the shortcoming highlighted in \cref{fig:counterpoint-example} is likely to appear in both synthetic and real datasets. In \cref{appx:plaus-data-analyses}, we identify several of these cases on some of the most popular graph explainability benchmarks, extending \citet{fontanesi2024xai}'s analysis to previously unexamined datasets.
Specifically, we discover novel shortcuts in \BASHAPES \citep{ying2019gnnexplainer} and \BAMULTISHAPES \citep{azzolin2022global}, where the task can be solved by averaging node degrees rather than using the intended motif. 
Additionally, on molecular benchmarks with explicit motif rules like \Benzene, \AlkaneCarbonyl, and \FluorideCarbonyl \citep{sanchez2020evaluating}, we observe that GNNs consistently fail to learn the correct patterns exactly, resorting instead to spurious correlations.
Notably, we argue that this issue cannot be solved by dataset design alone.
\citet{faber2021comparing} and \citet{agarwal2023evaluating} propose handcrafted datasets meant to be more reliable, e.g., by generating unambiguous ground truth and by requiring perfect model accuracy.
Yet, they do not offer guarantees that a GNN trained on them implements a uniquely identified behaviour, 
and verifying a posteriori that the model relies on the intended pattern may itself be as hard as explaining the model.
Hence, rather than devising better datasets, we provide guarantees on the model itself in \cref{sec:method}.

\textbf{Evaluating the faithfulness of explanations.}
An orthogonal and widely used (c.f. \cref{fig:eval-papers}) approach to evaluating explanation quality is faithfulness, assessing how well an explanation reflects the model behaviour \citep{agarwal2024faithfulness}. 
Nonetheless, faithfulness comes with its own limitations:
%
there is no universal definition, available metrics often conflict \citep{azzolin2025reconsidering, miro2025comprehensive},
%
%
and these metrics are frequently validated by their correlation with human-defined explanations \citep{fang2023evaluating, zheng2023robust} -- the same issue that plausibility suffers from. 
%
Despite these limitations, faithfulness remains the only viable option to assess a given explanation in downstream applications, and the evaluation setting we will propose in \cref{sec:experiments} offers a reliable testbed to validate faithfulness metrics as well, which we leave to future work. 

%
In the remainder of the paper, we address the issue of plausibility by building GNNs whose behaviour is known by construction. This is achieved by replacing GNN training with \textit{compilation}.
\section{From logic formulas to ground truth explanations}
\label{sec:method}

\subsection{\METHODNAME: Compiling a logic formula into a GNN}
\label{sec:compiled-gnn}

Rather than hoping that training has induced the intended behaviour, we construct GNNs whose behaviour can be controlled exactly.
%
%
We obtain it by \textit{compiling}, in the spirit of \citet{lindner2023tracr}, a target classifier expressed in logic directly into the weights of a standard message-passing architecture, so that the two agree on every graph.
Our construction follows the constructive expressiveness results of \citet[Prop. 4.1]{barcelo2020logical}, which we extend to cover also graph-level tasks.

\textbf{Preliminaries.}
We will consider node features drawn from a finite list $\calC$ of colours, which act as atomic predicates over nodes: a node $u$ is described by the indicator vector $\vc_u\in\{0,1\}^{|\calC|}$, and $\mathrm{Col}(u)=1$ iff $(\vc_u)_{\mathrm{Col}}=1$, with $\mathrm{Col}$ any colour drawn from $\calC$.
We compile classifiers expressible in \textit{graded modal logic} \citep{de2000note}, the fragment of first-order logic generated by
\[
\label{eq:gml}
    \varphi(u) \;::=\; \top \;\mid\; \bot \;\mid\; \mathrm{Col}(u) \;\mid\; \neg\varphi(u) \;\mid\; \varphi(u)\wedge\psi(u) \;\mid\; \exists^{\ge N} v. \big(E(u,v)\wedge\varphi(v)\big), 
\]
where $\varphi$ and $\psi$ are themselves graded modal logic formulas defined over a free variable.
%
The last construct holds at $u$ when at least $N$ neighbours satisfy $\varphi$, and will be shortened to $\Diamond_{\ge N}\,\varphi$. 
We reserve $\Phi$ for the classifier we wish to compile, keeping $\varphi$ and $\psi$ for arbitrary formulas and, later, $\varphi_\ell$ for the subformulas of $\Phi$.
Throughout, we write $(G,u) \models \varphi$ to say that a formula $\varphi$ is satisfied at the node $u$ of $G$.
Other connectives, like $\vee$ or $\Diamond_{=N}$, are obtained by manipulations of \cref{eq:gml} as described in \cref{appx:desugaring}.
%
A node classifier is then a formula $\Phi(u)$ with one free variable, predicting $g(G,u)\defeq\Ind{(G,u) \models \Phi}$ at each node $u$ of $G$.
%
A graph classifier, instead, closes a node predicate $\varphi$ with a counting quantifier, $\Phi = \exists^{[\mathrm{lo}, \mathrm{hi}]}u.\,\varphi(u)$, and thus asks how many nodes of $G$ satisfy $\varphi$, predicting $g(G)\defeq\Ind{G\models\Phi}$. 
For instance, $\exists^{[N, \inf]}$ holds when at least $N$ of them do.

\textbf{The GNN architecture.}
We compile into an aggregate-combine GNN \citep{barcelo2020logical} with sum aggregation, per-layer weights, and truncated \textsc{ReLU} activation $\sigma(z)\defeq\min(\max(0,z),1)$,
\[
\label{eq:acgnn}
    \vh^{(\ell)}_u \;=\; \sigma\Big(\vh^{(\ell-1)}_u C^{(\ell)} \;+\; \sum_{v\in\calN(u)}\vh^{(\ell-1)}_v A^{(\ell)} \;+\; b^{(\ell)}\Big),
    \qquad \ell=1,\dots,L
\]
with $C^{(\ell)},A^{(\ell)}\in\bbR^{d\times d}$ and $b^{(\ell)}\in\bbR^{d}$.
Compilation is possible for this fragment precisely because every graded modal logic classifier is captured by a GNN \citep{barcelo2020logical}.

\textbf{GRAph Neural Network CompileR (\METHODNAME).}
Let $\mathrm{sub}(\Phi)=[\varphi_1,\dots,\varphi_d]$ be the distinct subformulas of $\Phi$, ordered so that each one precedes those containing it and $\varphi_d=\Phi$, and let $\mathrm{depth}(\varphi)$ denote the number of nested connectives in $\varphi$, so that atoms have depth $0$ and $\mathrm{depth}(\Phi)=L$. 
Our construction reserves one coordinate $i \in \{1,\dots,d\}$ of the hidden space to encode the truth value of each subformula $\varphi_i$, maintaining the invariant for successive layers, as follows:
\[
\label{eq:invariant}
    \big(\vh^{(\ell)}_u\big)_i = 1 \iff (G,u) \models \varphi_i ,
    \qquad\text{for every } \mathrm{depth}(\varphi_i) \le \ell \le L.
\]

This is achieved via three components.
\textit{First,} we initialise $\vh^{(0)}_u$ from the node features $\vc_u$, setting coordinate $i$ to $1$ exactly when $\varphi_i$ is an atom satisfied at $u$.
\textit{Second,} layer $\ell$ carries forward the subformulas already settled by setting $C^{(\ell)}_{ii}=1$ whenever $\mathrm{depth}(\varphi_i)<\ell$, and evaluates each subformula $\varphi_i$ with $\mathrm{depth}(\varphi_i)=\ell$ according to \cref{tab:weights}, all other entries being zero, with $j,k < i$.
\textit{Last,} after $L$ layers a node classifier reads off coordinate $d$, while a graph classifier $\exists^{[\mathrm{lo}, \mathrm{hi}]}u.\,\varphi(u)$ first sum-pools that coordinate into the count $m=|\{u \in \calV: (G,u) \models \varphi\}|$, and then thresholds it \eg $\sigma(m-N+1)$ for $\exists^{[\mathrm{N}, \infty]}$.
%
%
The correctness of this construction follows from \citet{barcelo2020logical}.

\begin{wraptable}[8]{r}{0.45\linewidth}
    \centering
    \caption{Weights realising each connective.}
    \label{tab:weights}
    \footnotesize
    \scalebox{0.9}{
    \begin{tabular}{lll}
        \toprule
        Subformula & \multicolumn{2}{c}{Compiled weights} \\
        \midrule
        $\varphi_i=\neg\varphi_k$             & $C^{(\ell)}_{ki}=-1$                & $b^{(\ell)}_i = 1$ \\
        $\varphi_i = \varphi_j \wedge \varphi_k$  & $C^{(\ell)}_{ji}=C^{(\ell)}_{ki}=1$ & $b^{(\ell)}_i = -1$ \\
        $\varphi_i = \Diamond_{\ge N}\varphi_k$ & $A^{(\ell)}_{ki}=1$                 & $b^{(\ell)}_i = 1-N$ \\
        \bottomrule
    \end{tabular}
    }
\end{wraptable}

Further implementation details are provided in \cref{appx:compiler-details}, along with a practical example of a compiled GNN and its instantiation using our Python codebase in \cref{appx:example-compiled}. 
Taken together, the construction above allows us to build a GNN that encodes a graded modal logic statement.
This allows practitioners to have full knowledge of the inner mechanisms of the GNN to be explained, constituting a powerful laboratory for research in graph explainability.

\subsection{Defining ground truth explanations formally}
\label{sec:gt}

Having obtained a GNN reliably implementing a target formula, the next stage is to define a suitable notion of ground truth explanations, \ie what we consider as the \textit{desired} explanation.
In previous studies, this usually corresponds to the implanted target motif of interest, e.g., the house motif in \BATWOMOTIF \citep{luo2020parameterized}.
Nonetheless, the case where ground truth explanations are not \textit{unique}\footnote{
E.g., counting tasks like \textit{a node has more red than blue neighbours}: with $n$ red and $m < n$ blue neighbours, any $m+1$ red edges suffice and no fewer do, giving $\binom{n}{m+1}$ valid explanations.} is often overlooked, and a precise definition of what a ground truth explanation looks like is often left to task-specific heuristics \citep{schlichtkrull2021interpreting}.
%
In this work, we will build on the formal notion of graph explanation introduced in \citet{azzolin2025formal}, described next:
\begin{definition}(Prime Implicant explanation)
\label{def:piexpl}
   Let $g$ be a classifier and $G$ be an instance with predicted label $g(G)$, then $R\subseteq G$ is a Prime Implicant explanation (PI) for $g(G)$, written as $PI(g, G)$, if:
    \begin{enumerate}            
        \item (\textit{sufficiency}) \label{PI2} For all $R'$, such that $R \subseteq R' \subseteq G$, we have that ${g(G) = g(R')}$.
        
        \item (\textit{minimality}) \label{PI3} No other $R' \subset R$ satisfies \eqref{PI2}.
    \end{enumerate}
    %
    PIs are not unique,  and we denote the set containing the Union of all PIs (\UPI) for $g(G)$ as $\UPI(g, G) \defeq \cup_i PI(g, G)_i$.
    We will shorten this to $\UPI$ when referring to a generic set of all PIs.
    For node-level tasks, we will write $PI(g, G, u)$ to indicate the PI explanation for $g(G, u)$. 
\end{definition}

PIs are a reasonable target for the following reasons:
\textit{First,} they are a crisp, rigorous solution for explaining a wide range of models, like logic classifiers and tabular models \citep{shih2018symbolicapproach, marques2022delivering, marques2023logic};
\textit{Second,} they were shown to encode a more informative and trustworthy notion of explanation for graph tasks \citep{weinbauer2025prime, azzolin2026degenerate};
\textit{Third,} they naturally handle alternative explanations by collecting the set of all possible PIs \citep{darwiche2020reasons, darwiche2023complete};
\textit{Last,} they were shown to recover the implanted target motif for motif-based graph tasks \citep{azzolin2025formal}.
%
Then, our definition of ground truth explanation is:
\begin{definition}(Ground truth explanation)
\label{def:gt}
    Let $g$ be a classifier and $G$ be an instance with predicted label $g(G)$, then the ground truth explanation $\mu^*$ for $g(G)$ is $\UPI(g, G)$, such that $\mu^*_e = \Ind{e \in \UPI}$.
\end{definition}

The definition above indicates the edges we consider to be relevant for a given task, and reconciles with existing definitions of feature relevancy from the formal explainability literature for tabular data \citep{huang2023feature}.
Having established a precise notion of ground truth explanation, we describe in the next section how $\UPI$ can be computed in practice.



\subsection{Computing $\UPI$ for a compiled GNN}
\label{sec:upi}

Computing PIs is a known computationally hard problem, with a few exceptions \citep{marquessilva20naivebayes, bassan2026unifying}.
This complexity arises because for every candidate explanation $R$, one has to query the black-box model with \textit{every} extension $R \subseteq R' \subseteq G$, and then iterate to find a subset-minimal candidate (cf. \cref{def:piexpl}).
Thanks to our compilation step, however, we can avoid the cost of invoking the black-box model by relying on the symbolic representation from which the GNN was derived.
%
Next, we will describe a knowledge compilation algorithm relying on Reduced Ordered Binary Decision Diagram (\ROBDD) \citep{lee1959representation, bryant1992symbolic} that allows an efficient computation of $\UPI$.
Although this is not a novel technique in itself \citep{shih2018symbolicapproach, darwiche2023complete}, it is to the best of our knowledge the first time it is applied in the context of graph explainability.
%


At a high level, an \ROBDD is a tractable representation of a Boolean function $f:\{0,1\}^{n} \rightarrow \{0,1\}$ over a set of Boolean variables $X = \{x_1, \ldots ,x_n\}$.
This is achieved by building a directed acyclic graph where nodes represent the Boolean variables, and edges represent Boolean assignments.
This graph-based representation of a Boolean function supports efficient operations and queries on the function, including conjoining and disjoining \ROBDDs, and this tractability will be critical for efficiently extracting \UPI (see \citet{meinel1998obdds} for more on \ROBDDs).
Due to space constraints, we defer a more detailed introduction to \ROBDDs to \cref{appx:robdd-example}.

\textbf{Symbolic encoding of Boolean formulas over graphs.}
In our setting, given a target graph $G$ and a Boolean function $f$ representing a logic formula defined over $G$, we instantiate an \ROBDD with one node for each directed edge of $G$, such that $X=\{x_{e}\}_{e \in \calE}$.
Variables are ordered by edge index, and node and edge features are kept constant.
A subgraph $R \subseteq G$ is therefore identified with the assignment assigning $x_e=1$ when $e\in R$ and $x_e=0$ otherwise, representing a bijection between the $2^{|\calE|}$ boolean assignments of $X$ and the $2^{|\calE|}$ subgraphs present in $G$.
%
Thanks to this bijection, \textbf{\textit{an \ROBDD compactly encodes a Boolean function over all possible subgraphs of $G$}}.
Hence, we will write the \ROBDD as a function over Boolean variables $X$ or subgraphs $R$ interchangeably.
We report an example of such a construction in \cref{appx:robdd-example}.

Having established how input graphs are encoded in Boolean variables, the construction of the \ROBDD follows standard construction and reduction algorithms, for which we refer the reader to \citet{bryant1986graph} and \citet{wegener2000branching}. We refer to this procedure as $\mathrm{BUILD}$, described in \cref{appx:build-alg}.

So far, $f$ has been a generic function mapping subgraphs to $\{0,1\}$, hence a binary classifier.
In our experiments, $f$ is given by the modal logic sentence from which the GNN to be explained is derived, as discussed in \cref{sec:compiled-gnn}.
Consequently, the procedure described next for extracting \UPI from $f$ yields a valid ground truth explanation (cf. \cref{def:gt}) of the GNN we will explain.


\begin{figure}[t]
    \centering
    \begin{minipage}[t]{0.315\textwidth}
    \begin{algorithm}[H]\footnotesize
    \caption{\proc{UnionPI}} 
    \label{alg:main}
    \begin{algorithmic}[1]
    \Require classifier $\Phi$; graph $G$
    \Ensure $\mu^*_e=\Ind{e\in\UPI}$
    \State $f \gets \proc{Build}(\Phi,G)$
    \State $\hat y \gets f(G)$
    \State $f \gets f$ \textbf{if} $\hat y=1$ \textbf{else} $\neg f$ \label{line:ok}
    \State $\fext \gets \proc{Extend}(f)$ \label{line:call2}
    \State $\mu^* \gets \proc{Extract}(\fext)$ \label{line:call3}
    \State \Return $\mu^*$
    \end{algorithmic}
    \end{algorithm}
    \end{minipage}\hfill
    \begin{minipage}[t]{0.315\textwidth}
    \begin{algorithm}[H]\footnotesize
    \caption{\proc{Extend}$(f)$}
    \label{alg:Extend}
    \begin{algorithmic}[1]
    \Require $f$ is an \ROBDD
    \Ensure $\fext$ maps $R$ to $1$ iff $R$ is sufficient (Def. \ref{def:piexpl})
      \State $\fext \gets f$
      \For{$e = 1,\dots,|\calE|$}
        \State $\fext \gets \fext \wedge \cof{\fext}{e}{1}$ \label{line:step}
      \EndFor
    \State \Return $\fext$
    \end{algorithmic}
    \end{algorithm}
    \end{minipage}\hfill
    \begin{minipage}[t]{0.315\textwidth}
    \begin{algorithm}[H]\footnotesize
    \caption{\proc{Extract}$(\fext)$}
    \label{alg:extract}
    \begin{algorithmic}[1]
    \Require $\fext$ as given by Alg. \ref{alg:Extend}
    \Ensure mask $\mu^*\in\{0,1\}^{|\calE|}$
    \For{$e = 1,\dots,|\calE|$}
      \State $\kappa \gets \fext \wedge x_e \wedge \neg\cof{\fext}{e}{0}$ \label{line:crit}
      \State $\mu^*_e \gets \Ind{\kappa \neq 0}$ \label{line:sat}
    \EndFor
    \State \Return $\mu^*$
    \end{algorithmic}
    \end{algorithm}
    \end{minipage}
\end{figure}

\textbf{Extracting \UPI from an \ROBDD.}
\proc{UnionPI}, described in \cref{alg:main}, presents the procedure for extracting \UPI from an input graph $G$ for the classifier encoded by an \ROBDD.
Once the \ROBDD has been constructed, \proc{UnionPI} invokes two routines:

\proc{Extend} (\cref{alg:Extend}) is responsible for finding those subgraphs that are sufficient (cf. \cref{def:piexpl}), albeit possibly not minimal.
It does so by modifying the function $f$ that the \ROBDD encodes from
\[
\textstyle
f: R \mapsto f(R)
\qquad\text{to}\qquad
\fext: R \mapsto f(R) \land \bigwedge_{R \subset R' \subseteq G} f(R')
\]
%
by recursively conjoining $\fext$ with $\cof{\fext}{e}{1}$, which corresponds to a copy of $\fext$ whose value at $R$ is the value of $\fext$ at $R \cup \{e\}$.
After unrolling the recursion, each subgraph $R$ will be conjoined to every possible extension of itself inside $G$, and the truth value at $R$ will be true iff $R$ satisfies the sufficiency property.
We show an example of how this recursion unrolls in Appendix \ref{appx:example-unrolling}.

\proc{Extract} (\cref{alg:extract}), instead, takes as input the \ROBDD as given by \proc{Extend} and returns a mask $\mu^*$ over the edges of $G$ such that $\mu^*_e = \Ind{e \in \UPI(f,G)}$.
For each edge, it asks whether there is a sufficient edge set that contains it ($\fext \land x_e$) and that stops being sufficient the moment the edge is removed ($\dots \land \neg\cof{\fext}{e}{0}$). If the resulting formula admits a witness, i.e., $\kappa \ne 0$, then that edge is indispensable to at least one prime implicant, and therefore belongs in \UPI.
This is formalized next:

\begin{restatable}{theorem}{algupi}
\label{thm:extract}
    Let $\Phi$ be a classifier and $G$ a graph such that $G \models \Phi$, let $f=\proc{Build}(\Phi,G)$ be its \ROBDD encoding, and let $\fext$ be the output of $\proc{Extend}(f)$ so that $\fext$ maps a subgraph $R$ to $1$ iff $R$ is sufficient.
    Then, for any edge $e$ in $G$, it holds that
    $  ( \fext \wedge x_e \wedge \neg\cof{\fext}{e}{0} \; \ne 0 ) \iff e \in \UPI(f, G). $
\end{restatable}

This result also applies to negative instances, as detailed in the proof; see \cref{appx:proof-extract}.
%
Note that ad hoc, logic-specific knowledge compilation approaches are available \citep{ignatiev2019abduction, marquessilva20naivebayes, huang2021efficient}, but we wanted a general approach that generalizes easily beyond the logic fragment we are currently considering.
For completeness, \cref{appx:complexity} provides a computational complexity analysis, and \cref{appx:cost-vs-explainers} compares the runtime of \cref{alg:main} with that of the explainers evaluated next, showing that the runtime of our algorithm is negligible.

\section{\BENCHNAME: Benchmarking explainers on compiled GNNs}
\label{sec:experiments}

Having obtained a GNN implementing a target behaviour, we now use compiled models as a testbed for graph explainers.
%
Our pipeline is as follows: we write a formula, compile it into the weights of a GNN using \METHODNAME, generate a dataset of graphs whose label is consistent with the formula, compute the ground truth explanation \UPI with \cref{alg:main}, and compute the plausibility score of explainers against the ground truth.
%
No training is involved, so we are guaranteed the GNN reflects the ground truth we evaluate the explainer against, thus avoiding the confounder identified in \cref{sec:issues}.
Formulas and datasets are chosen to highlight specific shortcomings of explainers, as will be detailed later.

\textbf{Explainers.}
We benchmark eleven explainers spanning several methodological categories. These include
four \textit{gradient-based} attributions: \saliency \citep{pope2019explainability}, deconvolution \citep[\deconv,][]{zeiler2014visualizing}, guided backpropagation \citep[\guidedbp,][] {springenberg2014striving}, and integrated gradients \citep[\intgrad,][]{sundararajan2017axiomatic}; three \textit{perturbation-based} methods: \gnnexplainer \citep{ying2019gnnexplainer}, \graphmask \citep{schlichtkrull2021interpreting} and \subgraphx \citep{yuan2021explainability}; two \textit{surrogate-based} methods: \pgmexplainer \citep{vu2020pgm} and \graphsvx \citep{duval2021graphsvx}; and two \textit{decomposition-based} methods: \flowx \citep{gui2023flowx} and \lrp \citep{baldassarre2019explainability, bach2015pixel}. We additionally include a random baseline (\randombaseline) that assigns each edge in the computational field a score drawn uniformly at random from $\{0,1\}$.

\textbf{Computing the plausibility score.}
We score an explainer by separately reporting the precision and recall of the predicted $\mu$ with respect to $\mu^*$. Given that \UPI is a union of alternative and individually valid explanations, this choice allows us to distinguish two cases: 
\textit{i)} an explainer may precisely highlight truly relevant edges but fail to recover alternative explanations, resulting in high precision but low recall; and \textit{ii)} an explainer may recover many relevant edges across alternative explanations while also including irrelevant edges, resulting in high recall but low precision.
We provide additional motivation for this choice over alternatives in \cref{appx:metrics}. Operationally, for each target, we sweep over all thresholds induced by the scores, keep the one maximising $F_1$ between $\mu$ and $\mu^*$, and report the corresponding precision and recall.
To avoid empty explanations, we only explain positive targets.

\subsection{Tasks and Datasets}
\label{sec:experiments-gc}

We consider three graph-classification and three node-classification tasks.
For each task, we generate graphs with random backgrounds and implant the target pattern in positive ones, while expressing the task as a logical formula that \METHODNAME converts into a GNN. Details are in \cref{appx:datasets}.

\textbf{Graph classification.}
\textit{i)} In positive \REDBLUELEAF instances, the graph contains a chain of four red nodes, each with exactly two blue leaves. This structure mirrors classical motif-based benchmarks in which the implanted motif constitutes the ground truth.
\textit{ii)} In positive \REDUNDANCY instances, the graph contains a black node with six red neighbours, any one of which is sufficient to satisfy the property. The task therefore admits six alternative PI explanations, while \UPI comprises all six corresponding edges.
\textit{iii)} In positive \COMPREDUNDANCY instances, the graph contains a black node satisfying at least one of three conditions: having one red, two blue, or three green neighbours.
The task therefore admits three PI explanations, while \UPI again comprises six edges.

\textbf{Node classification.}
\textit{i)} \REDBLUELEAFNC is a node-level variant of \REDBLUELEAF. A target node is positive if it is red, has exactly two blue leaves, and is adjacent to another node with the same properties. Consequently, target nodes belonging to the same motif may have different explanations.
\textit{ii)} In \MAJONEHOP, a target node is positive if it has strictly more red neighbours than blue neighbours.
\textit{iii)} \MAJTWOHOP applies this comparison recursively. A target node is positive if it has strictly more red-ish neighbours than blue-ish neighbours, where a node is red-ish according to the \MAJONEHOP rule. This corresponds to two rounds of label propagation \citep{nerini2025true}.

\begin{table}[t]
    \centering
    \small
    \setlength{\tabcolsep}{4pt}
    \scalebox{0.80}{
    \begin{tabular}{l ccc ccc ccc}
        \toprule
         & \multicolumn{3}{c}{\REDBLUELEAF} & \multicolumn{3}{c}{\REDUNDANCY} & \multicolumn{3}{c}{\COMPREDUNDANCY} \\
        \cmidrule(lr){2-4}\cmidrule(lr){5-7}\cmidrule(lr){8-10}
        Explainer & precision & recall & $F_1$ & precision & recall & $F_1$ & precision & recall & $F_1$ \\
        \midrule
        \saliency                & $\nentry{1.00}{0.00}$ & $\nentry{1.00}{0.00}$ & $\nentry{1.00}{0.00}$ & $\nentry{1.00}{0.00}$ & $\nentry{1.00}{0.00}$ & $\nentry{1.00}{0.00}$ & $\nentry{0.04}{0.00}$ & $\nentry{1.00}{0.00}$ & $\nentry{0.09}{0.00}$ \\
        \intgrad                 & $\nentry{1.00}{0.00}$ & $\nentry{1.00}{0.00}$ & $\nentry{1.00}{0.00}$ & $\nentry{1.00}{0.00}$ & $\nentry{1.00}{0.00}$ & $\nentry{1.00}{0.00}$ & $\nentry{1.00}{0.00}$ & $\nentry{0.50}{0.00}$ & $\nentry{0.67}{0.00}$ \\
        \deconv                  & $\nentry{0.22}{0.00}$ & $\nentry{1.00}{0.00}$ & $\nentry{0.37}{0.00}$ & $\nentry{1.00}{0.00}$ & $\nentry{1.00}{0.00}$ & $\nentry{1.00}{0.00}$ & $\nentry{0.04}{0.00}$ & $\nentry{1.00}{0.00}$ & $\nentry{0.09}{0.00}$ \\
        \guidedbp                & $\nentry{1.00}{0.00}$ & $\nentry{1.00}{0.00}$ & $\nentry{1.00}{0.00}$ & $\nentry{1.00}{0.00}$ & $\nentry{1.00}{0.00}$ & $\nentry{1.00}{0.00}$ & $\nentry{0.04}{0.00}$ & $\nentry{1.00}{0.00}$ & $\nentry{0.09}{0.00}$ \\
        \gnnexplainer            & $\nentry{1.00}{0.00}$ & $\nentry{1.00}{0.00}$ & $\nentry{1.00}{0.00}$ & $\nentry{1.00}{0.00}$ & $\nentry{1.00}{0.00}$ & $\nentry{1.00}{0.00}$ & $\nentry{0.04}{0.00}$ & $\nentry{1.00}{0.00}$ & $\nentry{0.09}{0.00}$ \\
        \graphmask               & $\nentry{0.25}{0.03}$ & $\nentry{0.93}{0.15}$ & $\nentry{0.38}{0.02}$ & $\nentry{0.04}{0.00}$ & $\nentry{1.00}{0.00}$ & $\nentry{0.09}{0.00}$ & $\nentry{0.04}{0.00}$ & $\nentry{1.00}{0.00}$ & $\nentry{0.09}{0.00}$ \\
        \subgraphx               & $\nentry{0.40}{0.01}$ & $\nentry{0.84}{0.01}$ & $\nentry{0.38}{0.00}$ & $\nentry{0.19}{0.01}$ & $\nentry{0.69}{0.01}$ & $\nentry{0.15}{0.00}$ & $\nentry{0.19}{0.01}$ & $\nentry{0.69}{0.01}$ & $\nentry{0.15}{0.00}$ \\
        \pgmexplainer            & $\nentry{0.33}{0.16}$ & $\nentry{0.89}{0.19}$ & $\nentry{0.45}{0.15}$ & $\nentry{0.15}{0.17}$ & $\nentry{0.49}{0.16}$ & $\nentry{0.19}{0.16}$ & $\nentry{0.35}{0.16}$ & $\nentry{0.56}{0.08}$ & $\nentry{0.39}{0.15}$ \\
        \graphsvx                & $\nentry{1.00}{0.00}$ & $\nentry{1.00}{0.00}$ & $\nentry{1.00}{0.00}$ & $\nentry{0.45}{0.05}$ & $\nentry{0.89}{0.07}$ & $\nentry{0.59}{0.06}$ & $\nentry{0.47}{0.03}$ & $\nentry{0.74}{0.16}$ & $\nentry{0.56}{0.06}$ \\
        \flowx                   & $\nentry{0.90}{0.03}$ & $\nentry{0.99}{0.01}$ & $\nentry{0.94}{0.01}$ & $\nentry{1.00}{0.00}$ & $\nentry{1.00}{0.00}$ & $\nentry{1.00}{0.00}$ & $\nentry{1.00}{0.00}$ & $\nentry{0.97}{0.07}$ & $\nentry{0.98}{0.04}$ \\
        \lrp                     & $\nentry{1.00}{0.00}$ & $\nentry{1.00}{0.00}$ & $\nentry{1.00}{0.00}$ & $\nentry{1.00}{0.00}$ & $\nentry{1.00}{0.00}$ & $\nentry{1.00}{0.00}$ & $\nentry{0.04}{0.00}$ & $\nentry{1.00}{0.00}$ & $\nentry{0.09}{0.00}$ \\
        \midrule
        \randombaseline          & $\nentry{0.23}{0.02}$ & $\nentry{0.93}{0.16}$ & $\nentry{0.37}{0.00}$ & $\nentry{0.06}{0.01}$ & $\nentry{0.73}{0.15}$ & $\nentry{0.10}{0.01}$ & $\nentry{0.06}{0.01}$ & $\nentry{0.73}{0.15}$ & $\nentry{0.10}{0.01}$ \\
        \bottomrule
    \end{tabular}}
    \caption{
        \textbf{Plausibility on graph classification}, averaged over $5$ seeds.
    }
    \label{tab:explainers}
\end{table}

\subsection{Results \& Key takeaways}
\label{sec:takeaway}

The results are reported in \cref{tab:explainers} and \cref{tab:explainers-nc-majority} for graph and node classification, respectively.
Results highlight that we can recover previous insights while
uncovering new ones, described next.

\subsubsection{Recovering existing insights}
\label{sec:existing-insights}

\textbf{\textit{\gnnexplainer does not sparsify redundant edges.}}
Among the well-performing explainers, \gnnexplainer is the only one specifically designed to provide sparse explanations, due to its $\mathrm{L}_1$ penalty on explanation scores \citep{ying2019gnnexplainer}. 
Our results on \REDUNDANCY, however, reveal the opposite: \gnnexplainer achieves perfect recall across a wide range of $\mathrm{L}_1$ strengths (see \cref{fig:gnnexplainer-sweep} in the Appendix), highlighting \textit{all} alternative explanations rather than a sparse subset, a contradiction already observed empirically \citep{funke2022zorro} and explained theoretically \citep{he2025explaining}.

\textbf{\textit{\deconv can highlight irrelevant-by-design edges.}}
\deconv is the only gradient-based explainer faring poorly on the motif-based \REDBLUELEAF task. 
After a careful investigation, we found that \deconv can pathologically assign spurious relevance to irrelevant-by-design edges.
%
In fact, we can vary its precision by recolouring background nodes with a colour that does not constitute a valid explanation, as shown in \cref{fig:deconv-recolour} in the Appendix.
This recovers the finding of \citet{nie2018theoretical} that \deconv can be insensitive to what the model actually computes.

\begin{table}[t]
    \centering
    \small
    \setlength{\tabcolsep}{4pt}
    \scalebox{0.80}{
    \begin{tabular}{l ccc ccc ccc}
        \toprule
         & \multicolumn{3}{c}{\REDBLUELEAFNC} & \multicolumn{3}{c}{\MAJONEHOP} & \multicolumn{3}{c}{\MAJTWOHOP} \\
        \cmidrule(lr){2-4}\cmidrule(lr){5-7}\cmidrule(lr){8-10}
        Explainer & precision & recall & $F_1$ & precision & recall & $F_1$ & precision & recall & $F_1$ \\
        \midrule
        \saliency                & $\nentry{1.00}{0.00}$ & $\nentry{1.00}{0.00}$ & $\nentry{1.00}{0.00}$ & $\nentry{1.00}{0.00}$ & $\nentry{1.00}{0.00}$ & $\nentry{1.00}{0.00}$ & $\nentry{1.00}{0.00}$ & $\nentry{0.87}{0.02}$ & $\nentry{0.92}{0.02}$ \\
        \intgrad                 & $\nentry{1.00}{0.00}$ & $\nentry{1.00}{0.00}$ & $\nentry{1.00}{0.00}$ & $\nentry{1.00}{0.00}$ & $\nentry{1.00}{0.00}$ & $\nentry{1.00}{0.00}$ & $\nentry{0.96}{0.01}$ & $\nentry{1.00}{0.00}$ & $\nentry{0.97}{0.01}$ \\
        \deconv                  & $\nentry{0.87}{0.00}$ & $\nentry{0.58}{0.00}$ & $\nentry{0.69}{0.00}$ & $\nentry{1.00}{0.00}$ & $\nentry{1.00}{0.00}$ & $\nentry{1.00}{0.00}$ & $\nentry{0.94}{0.02}$ & $\nentry{1.00}{0.00}$ & $\nentry{0.96}{0.01}$ \\
        \guidedbp                & $\nentry{1.00}{0.00}$ & $\nentry{1.00}{0.00}$ & $\nentry{1.00}{0.00}$ & $\nentry{1.00}{0.00}$ & $\nentry{1.00}{0.00}$ & $\nentry{1.00}{0.00}$ & $\nentry{1.00}{0.00}$ & $\nentry{0.87}{0.02}$ & $\nentry{0.92}{0.02}$ \\
        \gnnexplainer            & $\nentry{1.00}{0.00}$ & $\nentry{1.00}{0.00}$ & $\nentry{1.00}{0.00}$ & $\nentry{1.00}{0.00}$ & $\nentry{1.00}{0.00}$ & $\nentry{1.00}{0.00}$ & $\nentry{0.96}{0.01}$ & $\nentry{0.99}{0.00}$ & $\nentry{0.97}{0.01}$ \\
        \graphmask               & $\nentry{0.01}{0.00}$ & $\nentry{0.92}{0.16}$ & $\nentry{0.01}{0.00}$ & $\nentry{0.00}{0.00}$ & $\nentry{1.00}{0.00}$ & $\nentry{0.00}{0.00}$ & $\nentry{0.02}{0.01}$ & $\nentry{0.78}{0.12}$ & $\nentry{0.03}{0.01}$ \\
        \subgraphx               & $\nentry{0.19}{0.00}^\dagger$ & $\nentry{0.74}{0.01}^\dagger$ & $\nentry{0.28}{0.00}^\dagger$ & $\nentry{0.36}{0.00}$ & $\nentry{0.88}{0.01}$ & $\nentry{0.49}{0.01}$ & \multicolumn{3}{c}{\textit{OOT}} \\
        \pgmexplainer            & $\nentry{0.82}{0.02}$ & $\nentry{0.97}{0.01}$ & $\nentry{0.88}{0.01}$ & $\nentry{0.43}{0.01}$ & $\nentry{0.99}{0.00}$ & $\nentry{0.59}{0.01}$ & $\nentry{0.30}{0.02}$ & $\nentry{0.83}{0.01}$ & $\nentry{0.41}{0.02}$ \\
        \graphsvx                & $\nentry{0.89}{0.00}$ & $\nentry{1.00}{0.00}$ & $\nentry{0.94}{0.00}$ & $\nentry{0.50}{0.00}$ & $\nentry{1.00}{0.00}$ & $\nentry{0.66}{0.00}$ & $\nentry{0.56}{0.02}$ & $\nentry{0.74}{0.02}$ & $\nentry{0.61}{0.02}$ \\
        \flowx                   & $\nentry{0.99}{0.00}$ & $\nentry{1.00}{0.00}$ & $\nentry{0.99}{0.00}$ & $\nentry{1.00}{0.00}$ & $\nentry{1.00}{0.00}$ & $\nentry{1.00}{0.00}$ & $\nentry{0.91}{0.01}$ & $\nentry{0.95}{0.01}$ & $\nentry{0.92}{0.01}$ \\
        \lrp                     & $\nentry{1.00}{0.00}$ & $\nentry{1.00}{0.00}$ & $\nentry{1.00}{0.00}$ & $\nentry{1.00}{0.00}$ & $\nentry{1.00}{0.00}$ & $\nentry{1.00}{0.00}$ & $\nentry{1.00}{0.01}$ & $\nentry{0.89}{0.02}$ & $\nentry{0.93}{0.01}$ \\ 
        \midrule
        \randombaseline          & $\nentry{0.19}{0.01}$ & $\nentry{0.49}{0.03}$ & $\nentry{0.24}{0.01}$ & $\nentry{0.33}{0.17}$ & $\nentry{0.81}{0.06}$ & $\nentry{0.41}{0.18}$ & $\nentry{0.22}{0.01}$ & $\nentry{0.51}{0.04}$ & $\nentry{0.29}{0.01}$ \\
        \bottomrule
 
     \end{tabular}}
    \caption{
        \textbf{Plausibility on node classification,} averaged over $5$ seeds.
        A `\textit{OOT}' marks an explainer that exceeded the per-explanation time budget (see \cref{appx:impl-details}) on every seed; a $\dagger$ marks one whose mean uses only seeds completed within the budget.
    }
    \label{tab:explainers-nc-majority}
\end{table}

\subsubsection{Uncovering new limitations of explainers}

\textbf{\textit{Only \flowx, \graphsvx, and \intgrad are reliable across tasks, while others fluctuate.}}
Reading \cref{tab:explainers,tab:explainers-nc-majority} jointly, \flowx, \graphsvx, and \intgrad are the only explainers that never fall to chance level; e.g., \flowx keeps an F1 of at least $0.92$ on all six tasks, and \intgrad never scores below $0.96$ in precision.
Every other explainer is either uniformly poor or strongly task-dependent. 
%
We stress that this distinction is invisible on the motif-based tasks alone -- where \saliency, \lrp, and \gnnexplainer perform on par with \flowx and \intgrad\;-- and that it is only on settings departing from well-known motif-based tasks that the difference emerges, such as the two discussed next.

\textbf{\textit{Most explainers are not robust to the implementation of formulas.}}
A desirable property of explainers is implementation invariance \citep{sundararajan2017axiomatic, nauta2023anecdotal}, i.e., explainers should yield the same -- or similar thereof -- explanations 
for equivalent models.
%
Since \METHODNAME enables us to compile GNNs with a controllable behaviour, it provides a controlled setting in which to systematically evaluate this property.
As shown in \cref{tab:explainers}, several explainers that achieve perfect scores on the other tasks perform at chance level on \COMPREDUNDANCY.
To investigate this discrepancy, we evaluate an alternative formulation of the same task; see \cref{appx:compilation-choices}.
%
Only \intgrad, \flowx, and \graphsvx maintain consistent performance across the two formulations, suggesting greater robustness to these implementation changes. The discrepancy is associated with the original implementation of the ``$\lor$'' operator, which introduces double negations through De Morgan's laws.
%
Similar effects arise from other formula rewritings, as illustrated by \REDBLUELEAF in \cref{appx:compilation-choices}.

\textbf{\textit{Explainers struggle to recover indirect influences.}}
While most explainers achieve perfect \UPI recovery on \MAJONEHOP, none attain perfect results on \MAJTWOHOP.
We first observe that gradient-based explainers, with the exception of \deconv and \intgrad, maintain high precision ($>0.9$) at the expense of recall.
This is because they consistently miss edges whose removal would increase the number of bluish neighbours, indicating that they fail to capture indirect, multi-hop effects even in simple settings: although removing such edges does not directly change the target node’s number of red neighbours, they can change the state of an intermediate grey node, which indirectly affects the target’s prediction.
%
%
In contrast, on \MAJTWOHOP \gnnexplainer, \deconv, and \flowx attain a recall $\ge0.95$ but at the expense of precision, meaning they cover most indirect interactions whilst producing noisier explanations. See \cref{fig:ind-influence} in Appendix for visual examples of these problems.
This insight is robust across implementation choices for most explainers; see \cref{appx:ablation-comparative}.


\section{Conclusion}
\label{sec:conclusions}

We formalized a common failure mode when evaluating the plausibility of GNN explainers: the model being explained may not behave as expected. 
First, we showed that several popular datasets used in the literature suffer from this (\cref{sec:issues}).
Then, we proposed \METHODNAME, the first compiler translating graded modal logic sentences into the weights of a GNN (\cref{sec:method}). 
\METHODNAME yields compiled GNNs whose behaviour is known exactly, constituting a rigorous testbed for evaluating explainers.
To define a reliable ground truth, we adopt prime implicants and define our ground truth explanation to be the union of all prime implicants, for which we developed an exact, efficient extraction algorithm.
Finally, we proposed \BENCHNAME, which evaluates graph explainers over compiled GNNs based on how well they recover the respective ground truth explanation (\cref{sec:experiments}). 
Results on graph and node classification recover previously known findings while uncovering new ones, demonstrating how \BENCHNAME enables a sound, fine-grained analysis of the behaviour of post-hoc GNN explainers.

\textbf{Limitations.}
%
Compiled GNNs produced by \METHODNAME rely on binary embeddings and noiseless aggregations, and therefore have limitations in terms of realism compared to GNNs trained from data.
This is, however, a deliberate design choice allowing us to reason \textit{exactly} about the internal workings of both the model and the explainer, at a level of detail previously not possible.
Nonetheless, experiments conducted on compiled models should be interpreted carefully, and evaluations based on them should be intended as a minimum bar rather than a proxy of in-production evaluations.
Furthermore, just like for trained models, our results are to be interpreted in conjunction with the model, as different model configurations can yield different results. Nevertheless, \BENCHNAME allows us to rigorously control these configurations -- which are otherwise inextricably confounded on trained models -- and assess the robustness of explainers across these.
Last, our current implementation is restricted to binary classification and categorical node features.


\textbf{Future work.}
Since evaluating all available explainers was infeasible due to their large number, we exclusively tested popular instance-level explainers. Nonetheless, our analysis can be extended to include counterfactual and model-level explainers, especially those aiming to explain GNNs by extracting a logic formula; see \cref{appx:extended-rel-work} for an overview.
%
Also, we could make compiled GNNs more realistic, e.g., by fine-tuning after compilation. This is challenging, as fine-tuning may invalidate their guarantees, and ensuring that the expected behaviour is preserved is left to future work.
%
%


\bibliography{ref/explanatory-supervision, ref/references}
\bibliographystyle{unsrtnat}

\clearpage
\appendix

\begin{figure}[t]
    \centering
    \includegraphics[width=0.7\linewidth]{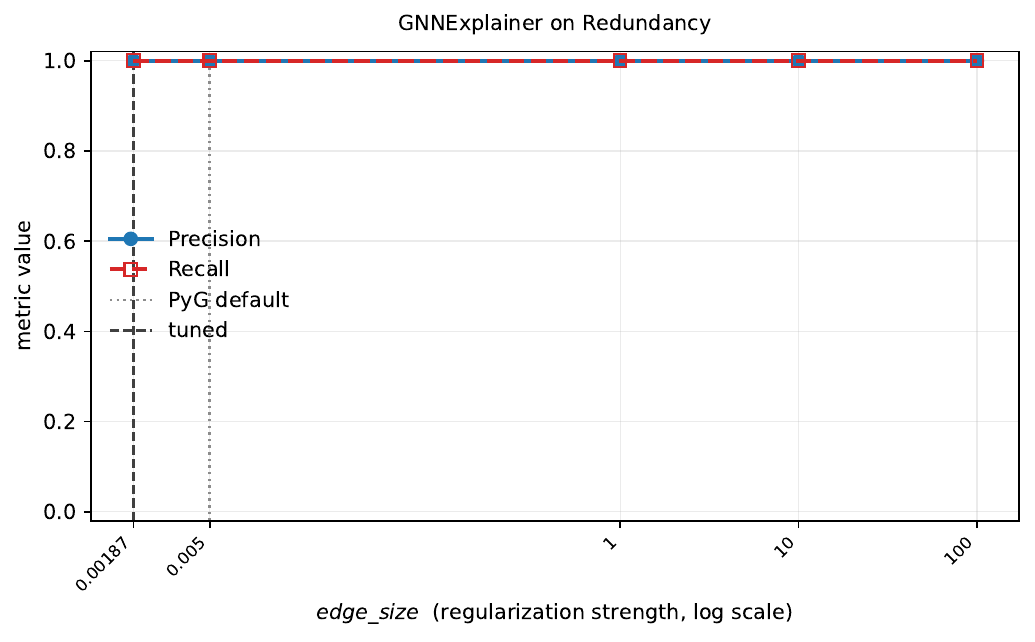}
    \includegraphics[width=0.7\linewidth]{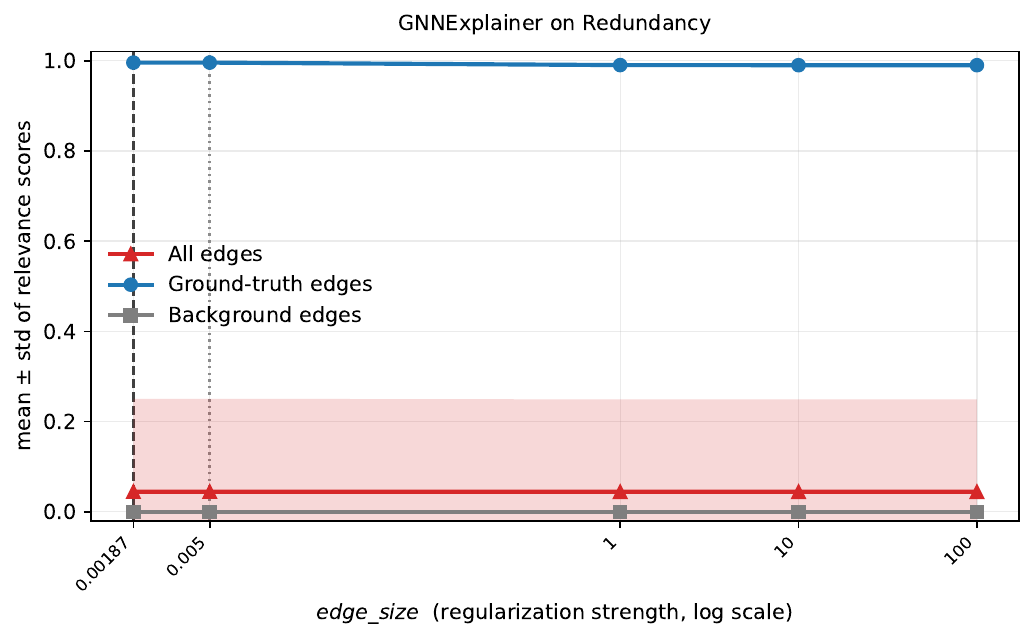}
    \caption{Ablation study on \gnnexplainer over \REDUNDANCY. The plots report: \textbf{Top:} metric values for different values of the $\mathrm{L}_1$ penalty regularization term controlling the sparsity strength; across all strength values, \gnnexplainer maintains perfect precision and recall. \textbf{Bottom:} average edge-wise relevance scores for different values of the $\mathrm{L}_1$ penalty regularization term, confirming that despite the penalty strength getting stronger, \gnnexplainer does not sparsify redundant edges.}
    \label{fig:gnnexplainer-sweep}
\end{figure}

\begin{figure}[t]
    \centering
    \includegraphics[width=0.99\linewidth]{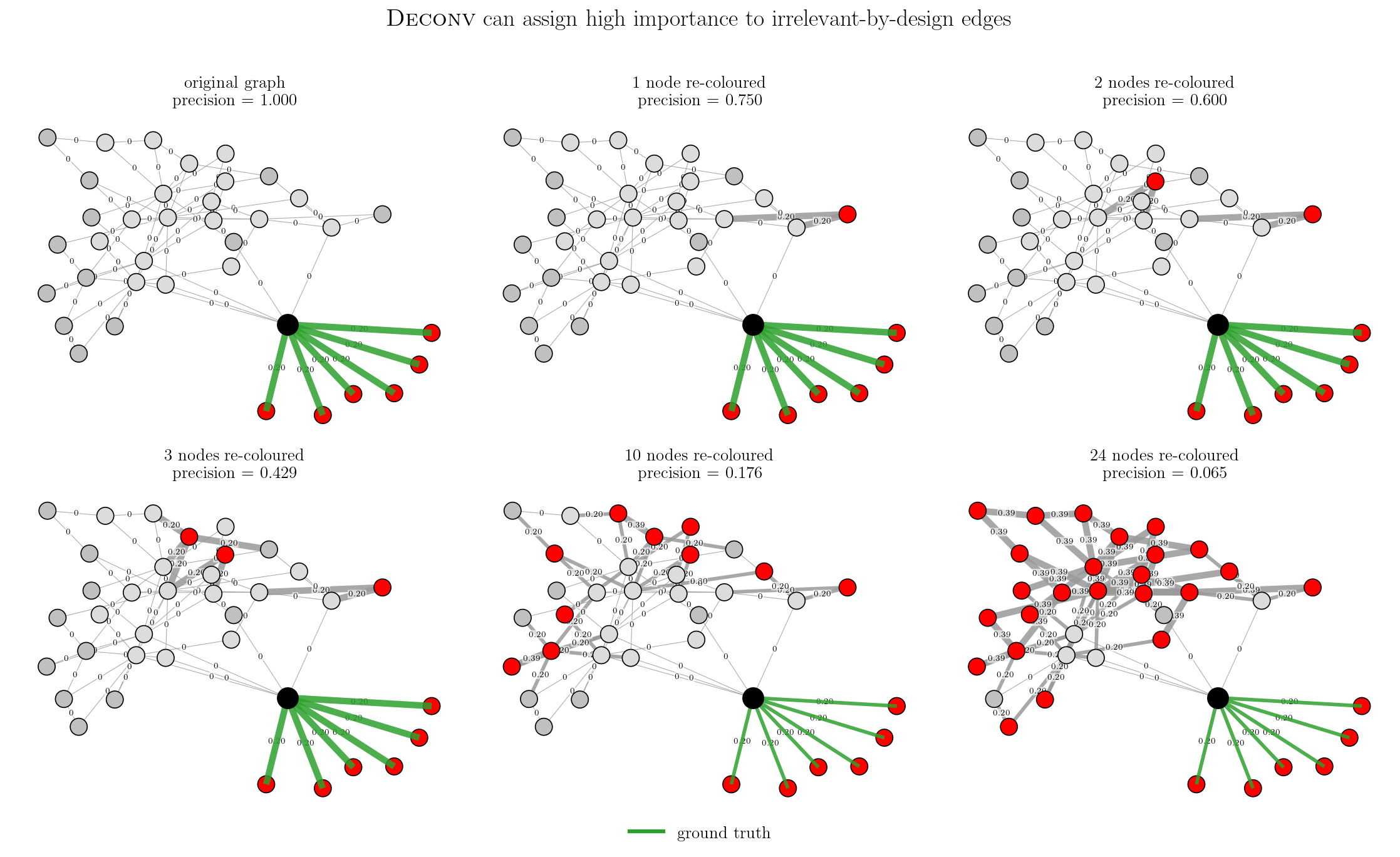}
    \caption{\textbf{\deconv can assign high importance to irrelevant-by-design edges}. To show this, we take a graph from \REDUNDANCY and progressively re-colour background nodes into red. Recall the underlying formula is $\exists u. \mathrm{black}(u) \wedge \Diamond_{\geq 1}\,\mathrm{red}$. In the original graph, grey-like colours play no role in the formula, and are correctly assigned zero importance. Hence, \deconv achieves perfect precision, as indicated also in \cref{tab:explainers}. 
    However, as red nodes are added to the graph, precision steadily drops. This is because red nodes, although they do not contribute to the formula when present alone, do appear in the formula, which is enough for \deconv to highlight them as important.
    This is the same mechanisms that brings \deconv to perform at chance-level on \REDBLUELEAF, as the background contains nodes firing intermediate sub-formulas (and intermediate sub-formulas only).
    }
    \label{fig:deconv-recolour}
\end{figure}

\begin{figure}
    \centering
    \includegraphics[width=\linewidth]{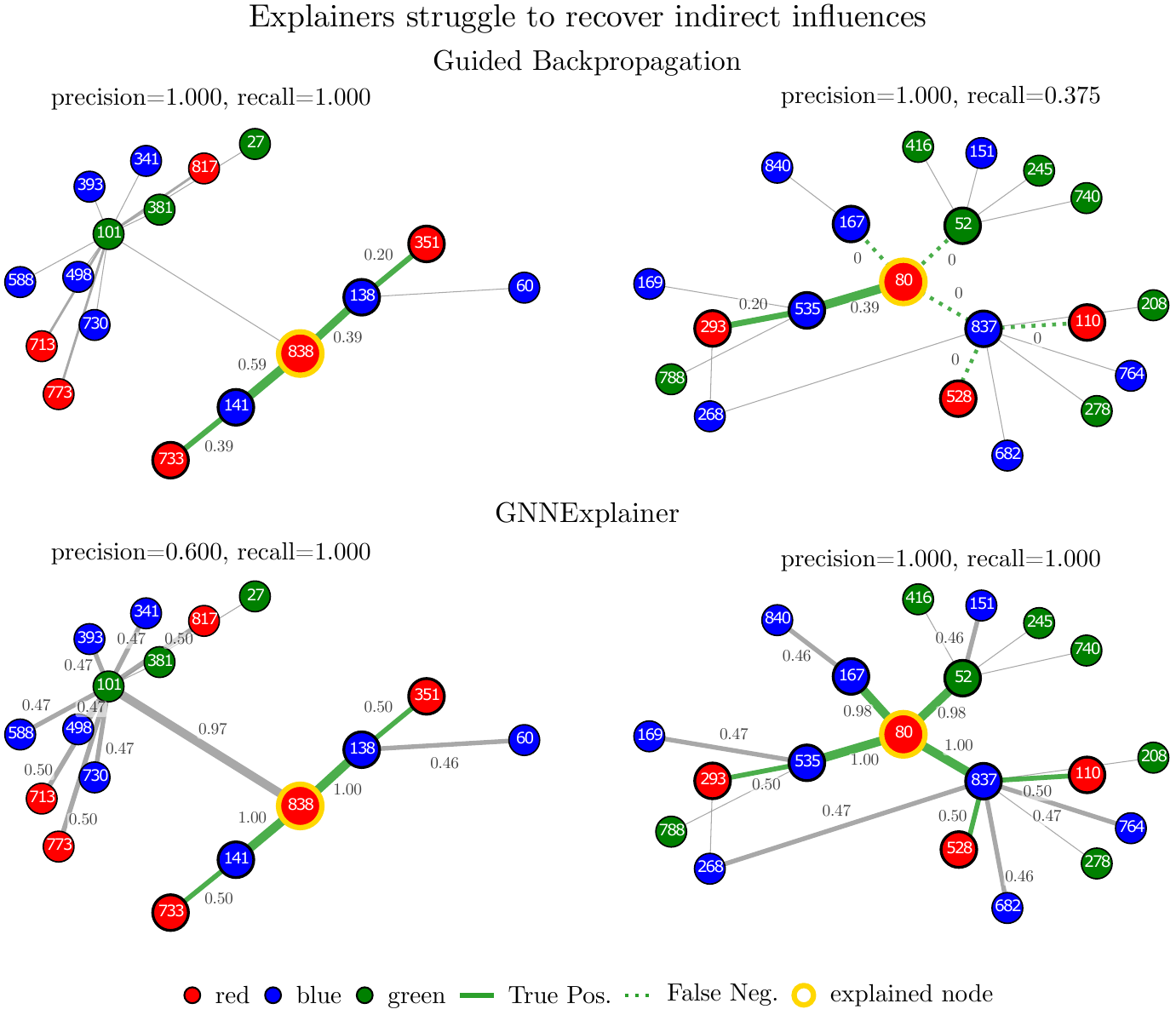}
    \caption{
    \textbf{Explainers struggle to recover indirect influences.} We illustrate this limitation using two representative examples from \MAJTWOHOP evaluated across \guidedbp and \gnnexplainer. 
    \MAJTWOHOP requires the target node to have strictly more 1-hop neighbors with a majority of red neighbors than neighbors with a majority of blue neighbors, with node colors reflecting their original features.
    On the left example graph,  \gnnexplainer assigns high importance ($\ge 0.5$) to edges $(101,713)$, $(101,817)$, and $(101,773)$; however, because node $101$ has a blue-majority neighborhood, it does not contribute to the positive prediction of the target. 
    On the right example graph, instead, \guidedbp assigns zero importance to edges $(80,167)$, $(80,52)$, $(80,837)$, $(110,837)$, and $(528,837)$. Yet, these edges ensure that nodes $167$, $52$, and $837$ maintain an equal balance of red and blue neighbors, which prevents a higher number of blue-ish nodes in the neighborhood of node $80$: if they were removed, the prediction at target node $80$ would flip.
    In both scenarios, the misattributed edges correspond to messages aggregated two hops away from the target. Notably, the true importance of $(80,52)$, $(80,837)$, and $(80,167)$ is directional (from node $80$ outward): while baseline explainers and our \UPI algorithm evaluate directed edges, we display undirected attributions here for visual clarity.}
    \label{fig:ind-influence}
\end{figure}

\clearpage

\section{Extended related work}
\label[appendix]{appx:extended-rel-work}


\textbf{From Logic to Neural Networks.} Integrating logical rules into machine learning models has been a long-standing line of research \citep[e.g.,][]{pazzani1992utility, TOWELL1994119, avila1999connectionist}. In this domain, researchers typically investigated methods to translate input propositional logic rules into neural networks, which could subsequently be refined through backpropagation. More recently, starting with the work of \citet{barcelo2020logical}, the equivalence between formal logic and GNNs has been deeply analysed to characterize the expressivity of the latter. Specifically, \citet{barcelo2020logical} established a fundamental correspondence between Aggregate-Combine GNNs and graded modal logic. \citet{Nunn2024logic} extended these results by proposing the logic $K^\#$, an expansion of graded modal logic that accounts for linear combinations of counting modalities, enabling bidirectional translation between specific GNNs and logical formulas. These results were later expanded by \citet{grau2026correspondence} and \citet{benedikt2026decidability}: the former systematically maps different classes of bounded GNNs to corresponding fragments of modal and first-order logic, whereas the latter broadens the class of GNNs translatable into logical formulas to analyse their expressiveness and decidability. Closest to our work are \citet{monti2024a} and \citet{nerini2025true}, which likewise replace trained black boxes with white-box GNNs of known logic. However, unlike our approach, they do not support general graded modal logic formulas. \citet{monti2024a} evaluates only node features, leaving edge importance entirely undefined. While \citet{nerini2025true} does target edges using a fixed label-propagation GNN (whose tasks we reuse as \MAJONEHOP and \MAJTWOHOP), their manual derivation of the ground-truth edges is flawed, as it crucially misses edges responsible for "indirect influences" in the classification. In contrast, Gracr compiles \emph{any} graded modal logic formula and automatically extracts the exact, complete \UPI (\cref{alg:main}) for both node and graph tasks.

\textbf{From GNN to Logic.}
A growing line of work explains GNNs post hoc with logic formulas rather than subgraphs.
Global explainers fit Boolean formulas over learned or mined concepts to the predictions of a trained model \citep{azzolin2022global,armgaan2024graphtrail,geng2026logicxgnn,ragno2025faithful}, or learn description logic class expressions for heterogeneous graphs \citep{kohler2024utilizing}.
Similarly, \citet{pluska2024logical} distil GNNs into decision trees, and \citet{morris2026sound} extract rules that are provably sound, at the price of restricting the GNN to non-negative weights.
A complementary line of work designs self-explainable GNN architectures that expose logic rules by construction rather than post hoc \citep{cranmer2020discovering,azzolin2025formal,ragno2026logic}.
Rather than extracting a logic formula out of a GNN, we go in the opposite direction: the formula is fixed first and compiled into the GNN, so the two coincide by construction.
Hence, \BENCHNAME offers a reliable testbed for testing also logic-based explainability methods, which we plan to investigate in future work.

\section{Analyses of popular datasets for evaluating plausibility}
\label[appendix]{appx:plaus-data-analyses}

We provide an extended analysis of \cref{sec:issues}. In short, for the dataset below, we investigate whether alternative solutions exist, or whether training a GNN on the dataset may yield a model that does not reliably implement the intended ground truth.\footnote{By ground truth here we mean the expected ground truth that the authors of these datasets use to evaluate plausibility in their setting.}
Since all of these datasets require identifying the correct motifs, the alignment between models and the intended ground truth is only possible if the models achieve nearly perfect accuracy, as discussed in \citet{faber2021comparing}.
For datasets already covered in \cref{sec:issues}, we only provide the experimental details and additional results.
In total, our analysis covers $74$ out of $94$ of the plausibility-related papers we considered.

\subsection{Synthetic datasets}

We start our analysis by looking at synthetic motif datasets. For all of them, we explicitly check the presence of a "shortcut", i.e., if there exists a valid alternative explanation $\bar{\mu}^*$ to the assumed ground truth $\mu^*$, invalidating \cref{ass:shortcut}. On these datasets, we use GIN \citep{xu2018powerful}, as a provably powerful GNN.

We observe that some of these "shortcuts" are already known in the literature: \citet{fontanesi2024xai} noticed that the datasets \BATWOMOTIF \citep{luo2020parameterized}, \BATWOGRID \citep{longa2024explaining}, and \GRIDHOUSE \citep{longa2024explaining} all contain alternative ground truth explanations that can be used to solve the task. We do not further comment on these datasets, and focus on novel shortcuts instead.

\textbf{\BASHAPES \citep{ying2019gnnexplainer}.}  As described in the main text, it is possible to obtain an accuracy of $100\%$ from a 2-layer GIN \citep{xu2018powerful} with $64$ hidden dimensions, trained for $1000$ epochs with an Adam optimizer (learning rate of $0.001$). Each GIN layer has an MLP of 2 layers, and between layers we use ReLU activations.
This suggests that an alternative ground truth explanation is present in the data, as the original implanted house-like motif requires $3$ hops in order to cover all of its edges.
In fact, the shortcut can be identified by plotting the degree distribution of each node (\cref{fig:bashapes-degree-distribution}, left panel) and of its neighbours (\cref{fig:bashapes-degree-distribution}, right panel), showing that all classes are perfectly separable by node degrees without the need to learn the house-like motif.

\begin{figure}[h!]
    \centering
    \includegraphics[width=0.45\linewidth]{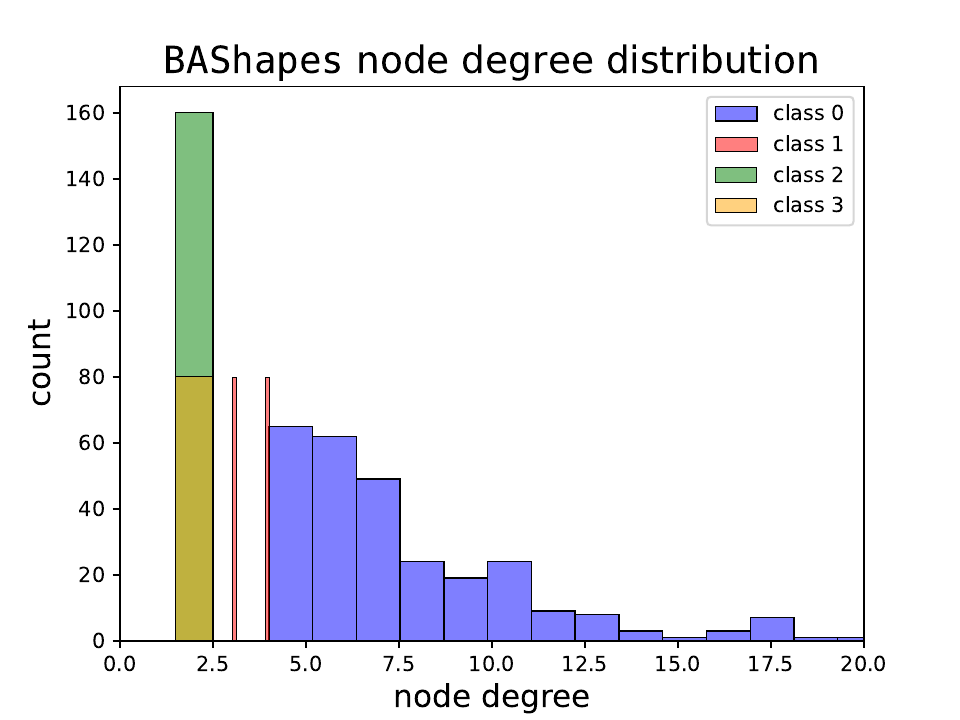}
    \includegraphics[width=0.45\linewidth]{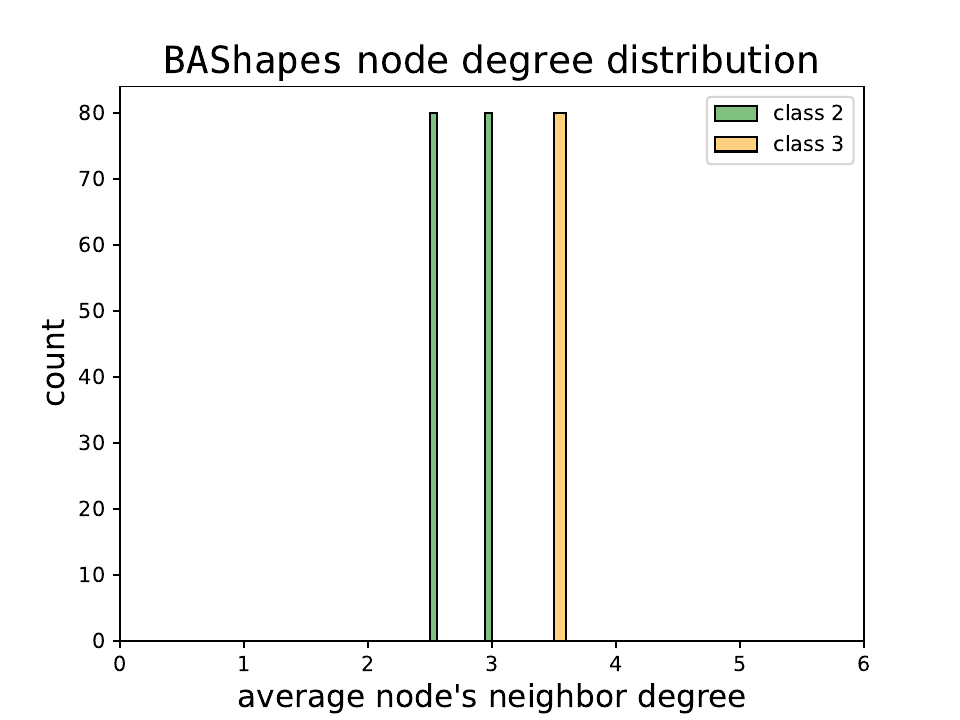}
    \caption{
            \textbf{Left:} Degree distribution for each node class in \BASHAPES. Classes $0$ and $1$ are fully distinguishable, whereas classes $2$ and $3$ overlap. \textbf{Right:} Distribution of average neighbour degrees for classes $2$ and $3$, showing that 2-hop degree statistics cleanly separate the remaining two classes.
        }
    \label{fig:bashapes-degree-distribution}
\end{figure}

\begin{figure}[h!]
    \centering
    \includegraphics[width=0.45\linewidth]{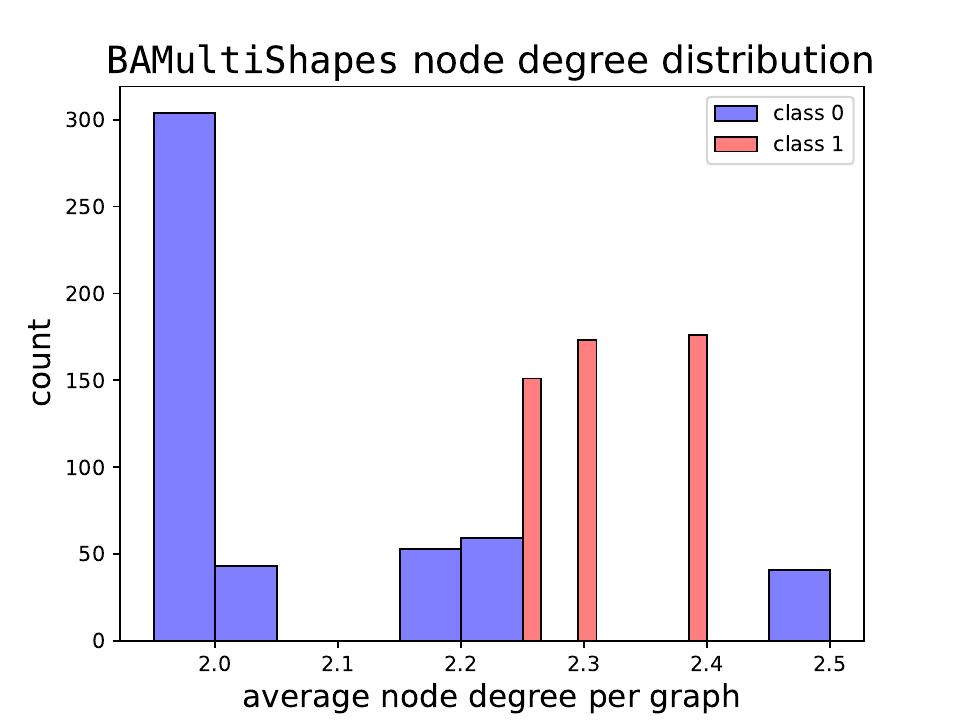}
    \caption{Distribution of average node degree in \BAMULTISHAPES, showing a clear separation between graphs in class $0$ and class $1$.}
    \label{fig:bamultishapes-degree-distribution}
\end{figure}

\textbf{\BAMULTISHAPES \citep{azzolin2022global}.}
Similarly to \BASHAPES, we notice that it is possible to separate the two classes perfectly just by computing the average node degree of graphs, plotted in \cref{fig:bamultishapes-degree-distribution}.
In fact, training a $1$-layer GIN with a 2-layer MLP for $200$ epochs can reach $100\%$ test accuracy. The GIN model was trained with a learning rate of $ 0.001$, mean global readout, $128$ hidden size, and an MLP classifier head.
This, again, indicates that the task can be solved by learning an alternative ground truth explanation to the one originally intended, providing a counterexample to \cref{ass:shortcut}.

\textbf{\TREECYCLES \citep{ying2019gnnexplainer}} is a node classification dataset where the task is to distinguish between nodes belonging to a $6$-cycle (the ground truth explanation) and those in the background tree.
To verify whether GNNs must learn the intended ground truth explanation, we trained a GIN \citep{xu2018powerful} for a different number of layers. Each model is trained with a $2$-layer MLP, a ReLU activation function, and a linear classifier. Also, each model was trained over $5$ random seeds for $1500$ epochs under identical conditions, using the Adam optimizer with a learning rate of $0.005$.
As shown in \cref{tab:TreeCycle-table}, however, no configuration could reach a test accuracy of $100\%$ robustly across random seeds, indicating that it may be hard for GINs to reliably learn the intended motif \citep{faber2021comparing}.
%
%
%
We then analyze the only model to achieve a test accuracy of $100\%$, although exclusively on few random seeds: a 4-layer GIN.
Supposedly, this is the only candidate GIN that may have learned the intended ground truth explanation.
Nonetheless, by performing local perturbations in the background tree, we observe a prediction change for a node \textit{inside} the cycle; see \cref{fig:treecycles-exampleanalysis} for more details.
This indicates that, despite the perfect accuracy, the model is still attending to unexpected information outside of the intended ground truth explanation.

\begin{figure}[h!]
    \centering
    \includegraphics[width=0.45\linewidth]{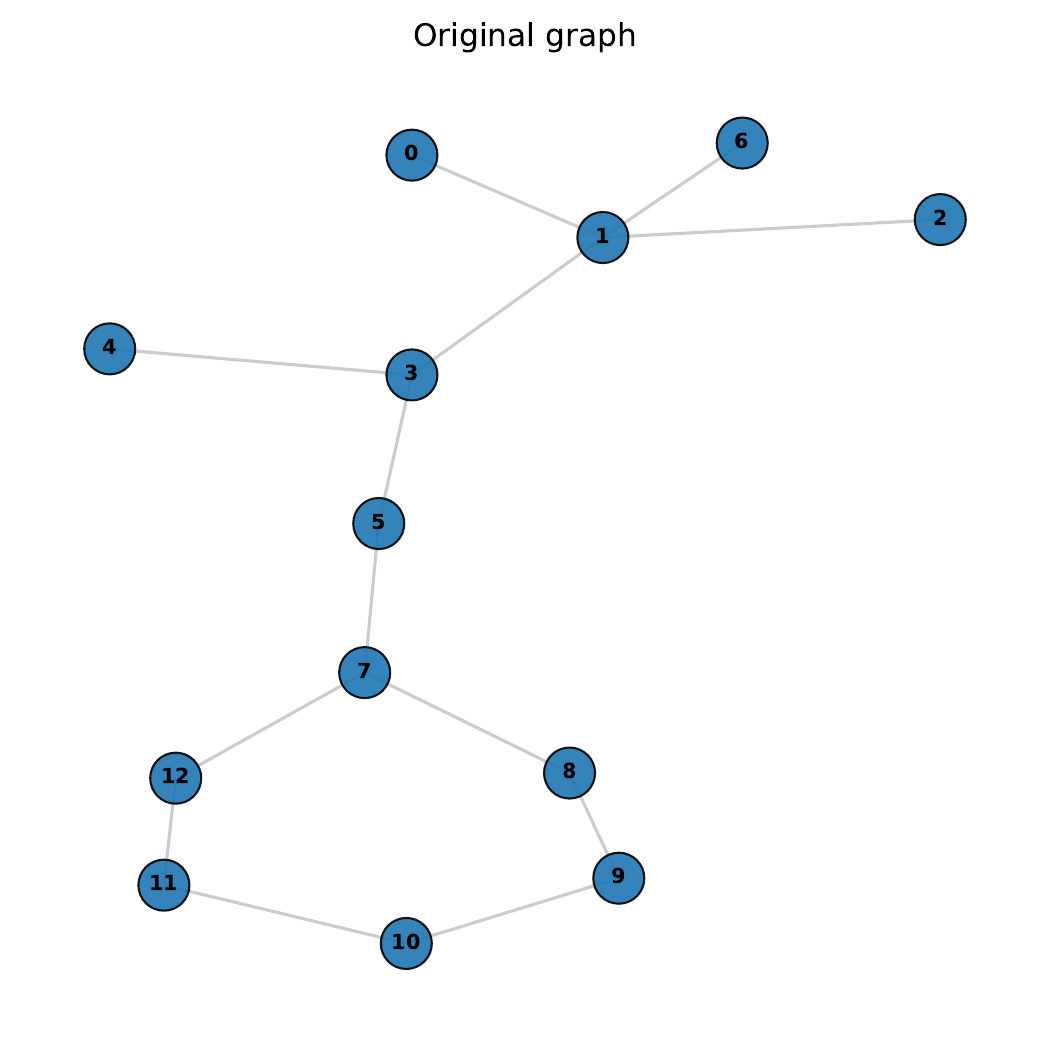}
    \includegraphics[width=0.45\linewidth]{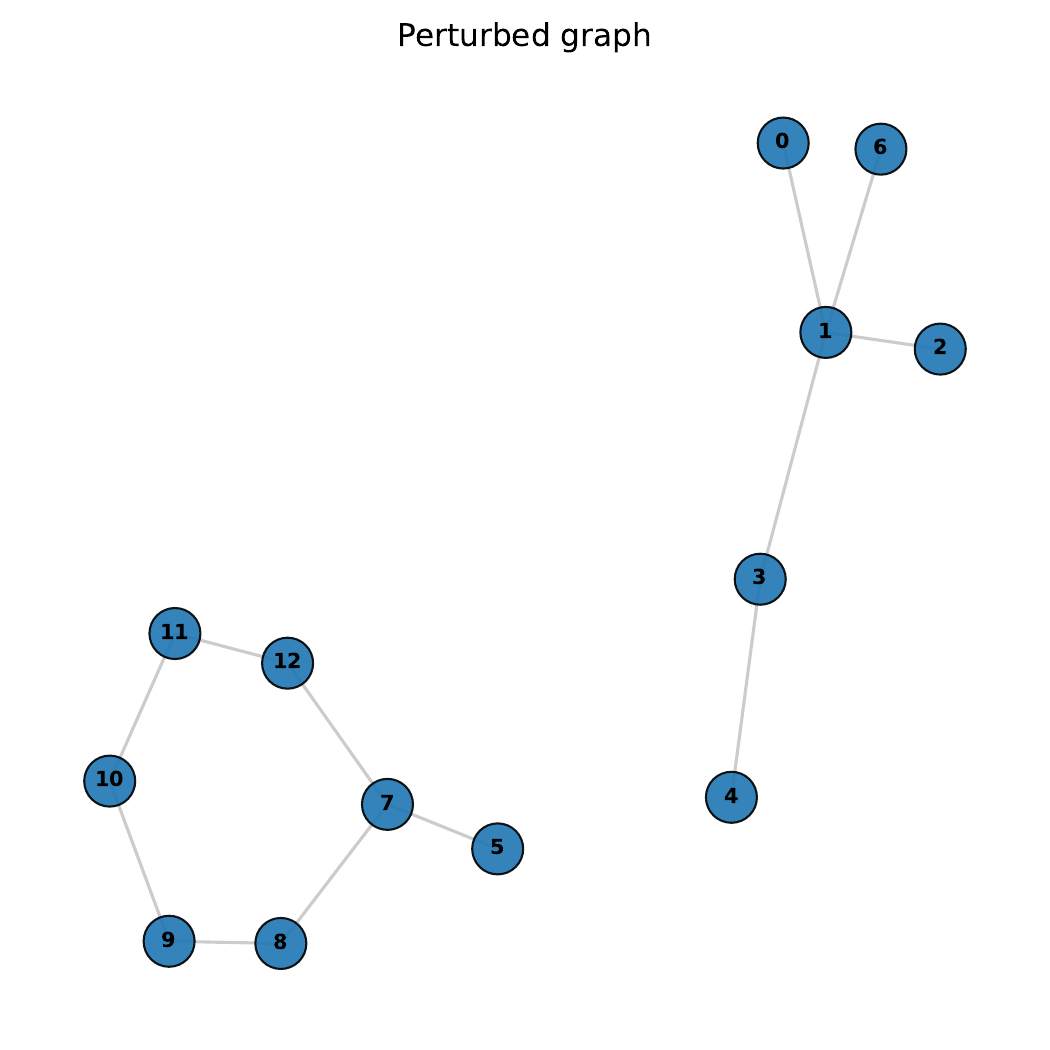}
    \caption{$4$-hop neighbourhood of node $7$ in \TREECYCLES. Original graph has all nodes predicted correctly. However, removing an edge outside of the ground truth explanation yields a prediction change for node $7$.}
    \label{fig:treecycles-exampleanalysis}
\end{figure}

\begin{table}[h!]
    \centering
    \caption{Test and validation accuracy averaged across five random seeds for \TREECYCLES.}
    \label{tab:TreeCycle-table}
    \begin{tabular}{ccc}
        \toprule
        Layers & Test Accuracy & Validation Accuracy \\
        \midrule
        $2$ & $0.9720 \pm 0.0045$ & $0.9818 \pm 0.0045$ \\
        $3$ & $0.9460 \pm 0.0141$ & $0.9939 \pm 0.0045$ \\
        $4$ & $0.9600 \pm 0.0245$ & $0.9859 \pm 0.0090$ \\
        \bottomrule
    \end{tabular}
\end{table}

\paragraph{\TREEGRID \citep{ying2019gnnexplainer}.} Analogous to \TREECYCLES, the task requires classifying nodes as belonging either to the background tree or to a $3\times3$ grid, which serves as the ground-truth motif.
Interestingly, $3$-layer GINs achieve the best accuracy on this task despite having a receptive field too small to cover the entire ground truth (\cref{tab:TreeGrid-table}), suggesting that the model is learning some shortcut in the node classification. The model was configured identically to \TREECYCLES ($64$ hidden channels, 2-layer MLPs, ReLU activations, linear classifier) and trained across $5$ seeds for $1500$ epochs with Adam (learning rate$=0.005$).
Crucially, removing edges strictly outside the ground-truth motif causes the model to flip its predictions on some target nodes, providing further evidence of a possible shortcut; see \cref{fig:treegrid-exampleanalysis} for more details.

\begin{figure}[t]
    \centering
    \includegraphics[width=0.45\linewidth]{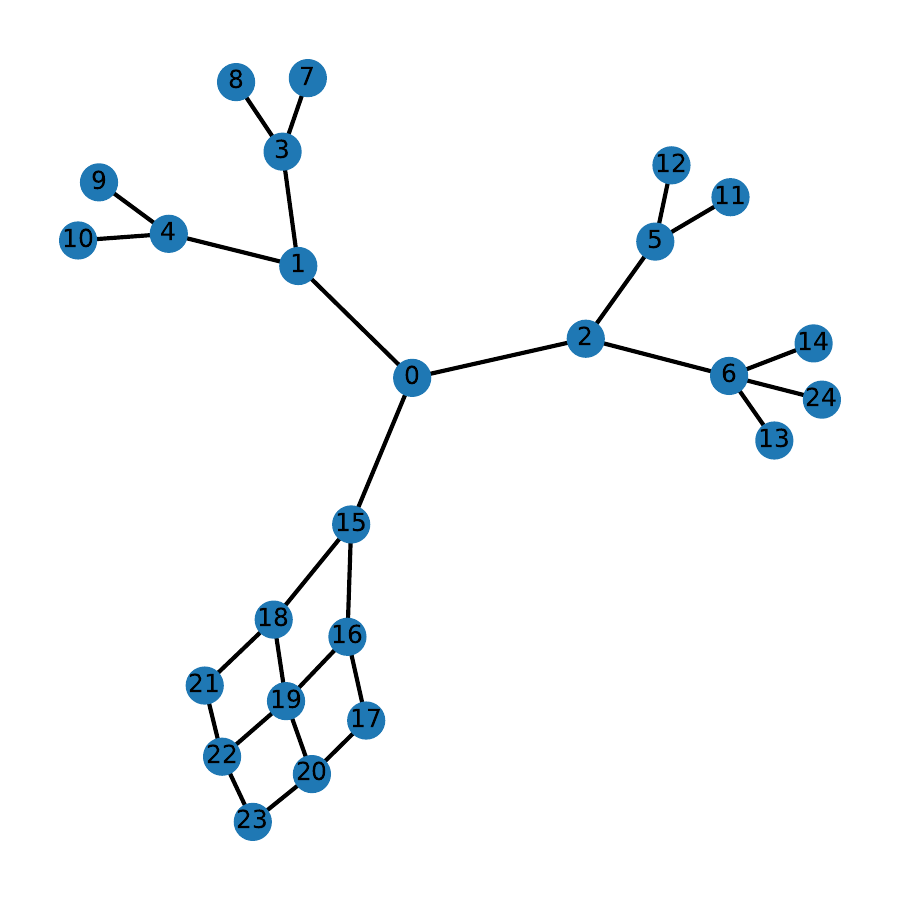}
    \includegraphics[width=0.45\linewidth]{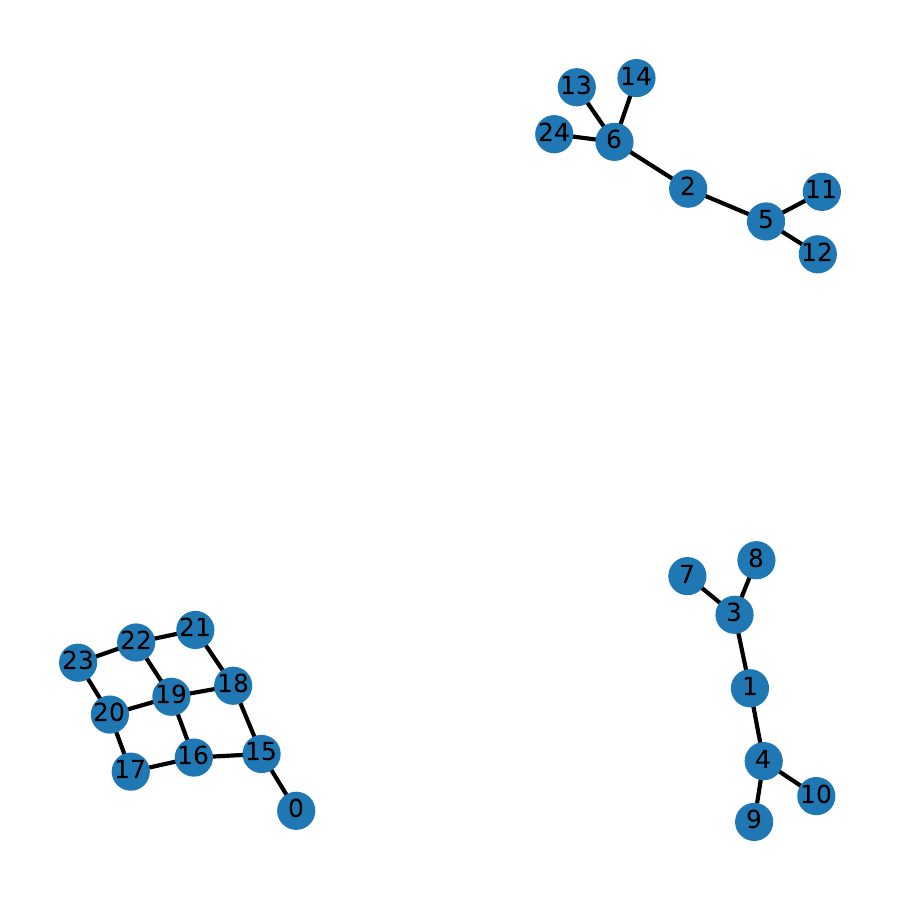}
    \caption{$4$-hop neighbourhood of node $15$ on \TREEGRID. 
    Given a high-accuracy model, all nodes are predicted correctly in the original graph. Yet, removing edges outside of the ground truth explanation (right) yields a prediction change for node $15$.}
    \label{fig:treegrid-exampleanalysis}
\end{figure}

\begin{table}[h!]
    \centering
    \caption{Test and validation accuracy averaged across five random seeds for \TREEGRID.}
    \label{tab:TreeGrid-table}
    \begin{tabular}{ccc}
        \toprule
        Layers & Test Accuracy & Validation Accuracy \\
        \midrule
        2 & $0.8774 \pm 0.0060$ & $0.9073 \pm 0.0040$ \\
        3 & $0.9774 \pm 0.0094$ & $0.9870 \pm 0.0040$ \\
        4 & $0.9694 \pm 0.0236$ & $0.9740 \pm 0.0130$ \\
        5 & $0.9581 \pm 0.0139$ & $0.9707 \pm 0.0151$ \\
        \bottomrule
    \end{tabular}
\end{table}

\subsection{Molecular datasets}
\label[appendix]{appx:chemical-datasets}

\begin{table*}[t]
    \centering
    \caption{Mean 5-fold cross-validation test performance (accuracy for \Mutagenicity and \Benzene; $F_1$ score for \AlkaneCarbonyl and \FluorideCarbonyl) for the best hyperparameter configuration across each architecture and dataset for the additional analysis provided in \cref{appx:chemical-datasets}.}
    \scalebox{0.9}{
    \begin{tabular}{lcccc}
    \toprule
     & \multicolumn{2}{c}{Accuracy} & \multicolumn{2}{c}{$F_1$ score} \\
    \cmidrule(lr){2-3} \cmidrule(lr){4-5}
     & \Mutagenicity & \Benzene & \AlkaneCarbonyl & \FluorideCarbonyl \\
    \midrule
    GCN       & 0.77 $\pm$ 0.05 & 0.85 $\pm$ 0.07 & 0.42 $\pm$ 0.31 & 0.61 $\pm$ 0.14 \\
    GIN       & 0.80 $\pm$ 0.02 & 0.91 $\pm$ 0.03 & 0.70 $\pm$ 0.17 & 0.71 $\pm$ 0.10 \\
    GraphSAGE & 0.77 $\pm$ 0.06 & 0.85 $\pm$ 0.04 & 0.43 $\pm$ 0.28 & 0.63 $\pm$ 0.12 \\
    PNA       & 0.80 $\pm$ 0.02 & 0.99 $\pm$ 0.02 & 0.71 $\pm$ 0.18 & 0.77 $\pm$ 0.11 \\
    GAT       & 0.77 $\pm$ 0.04 & 0.93 $\pm$ 0.06 & 0.49 $\pm$ 0.27 & 0.71 $\pm$ 0.09 \\
    \bottomrule
    \end{tabular}
    }
    \label{tab:all-datasets}
\end{table*}

We consider four widely adopted molecular datasets: \Mutagenicity, \Benzene, \AlkaneCarbonyl, and \FluorideCarbonyl \citep{kazius2005derivation, sanchez2020evaluating}\footnote{For our analysis, we used the dataset implementations provided by \textsc{GraphXAI} \citep{agarwal2023evaluating}, \url{https://github.com/mims-harvard/GraphXAI}. Note that in the public repository, the folder names for \FluorideCarbonyl and \AlkaneCarbonyl are inverted; we corrected this mismatch in our pipeline.}. \Mutagenicity is a graph-classification dataset annotated according to the mutagenic effect of several organic molecules. Previous studies have identified the $\mathrm{NO}_2$ and $\mathrm{NH}_2$ functional groups as the ground truth explanation $\mu^*$ \citep{ying2019gnnexplainer, luo2020parameterized}\footnote{In the literature, two separate datasets are referred to by this name although they originate from different studies, \MUTAG \citep{debnath1991structure} and \Mutagenicity \citep{kazius2005derivation}. Their primary difference lies in dataset size: \MUTAG consists of $188$ graphs, whereas \Mutagenicity contains $4337$. Since these names are frequently used interchangeably, we specify that all our reported results correspond to the larger dataset.}.
\Benzene, \FluorideCarbonyl, and \AlkaneCarbonyl, instead, are datasets where the task is to classify graphs based on the presence of specific chemical substructures that serve as the ground-truth motifs defining the positive class: a six-carbon aromatic ring for \Benzene; a fluoro group ($\mathrm{F}$) paired with a carbonyl group ($\mathrm{C{=}O}$) for \FluorideCarbonyl; and a carbonyl group alongside a three-carbon single-bonded chain for \AlkaneCarbonyl.

For each dataset, we test whether it is possible to train a model completely aligned with the intended ground truth. If that was the case, the model would never fail in the task of recognising the motif. Therefore, we check that it obtains near-perfect accuracy ($\ge0.99)$ on the task at hand, a test previously justified in \citet{faber2021comparing}. For all datasets, we performed a grid search over five backbones (GCN \citep{kipf2016semi}, GIN \citep{xu2018powerful}, GraphSAGE \citep{hamilton2017inductive}, PNA \citep{corso2020principal}, and GAT \citep{velivckovic2018graph}) across $1$ to $4$ layers and hidden dimensions in $\{16, 32, 64, 128, 256\}$. Every model employed global add pooling and was optimized with Adam (learning rate $=10^{-3}$, weight decay $=10^{-5}$, batch size $=128$) for up to $500$ epochs with cross-entropy loss. We evaluated each configuration using $5$-fold cross-validation, selecting the epoch with the lowest validation loss in each fold. Each dataset had both node features (atom type) and edge features (chemical bond type). We used the latter whenever the architecture could use them. We report the average test accuracy of the best configurations for all datasets in \cref{tab:all-datasets}.

Almost none of the model-dataset pairs achieve near-perfect accuracy. This suggests that these models are often unable to reliably recognise the ground-truth motifs, indicating that \cref{ass:shortcut} may not hold; otherwise, the models would be expected to achieve near-perfect discriminative performance.
We provide additional discussions on why this does not happen for each of the datasets in the following paragraphs.

\begin{wraptable}{r}{0.45\textwidth}
    \centering
    \caption{Percentage of prediction changes on \Mutagenicity across perturbed graphs for positive, mutagenic graphs containing $\mathrm{NO}_2$ or $\mathrm{NH}_2$ groups.}
    \label{tab:MUTAG-exp}
    \scalebox{0.90}{%
    \begin{tabular}{@{}lcc@{}}
    \toprule
    \textbf{Model} & \textbf{\makecell{Pred. change (\%)}} \\ \midrule
        GCN      & 39\% \\
        GIN      & 43\% \\
        GraphSAGE  & 46\% \\
        PNA        & 37\% \\
        GAT      & 41\% \\
    \bottomrule
    \end{tabular}
    }
\end{wraptable}

\textbf{\Mutagenicity.} We verified the model behaviour through an additional experiment. We first retrained the best-performing model configurations from scratch under identical settings on a $70\%/15\%/15\%$ train/val/test split; then, we selected all test mutagenic graphs with a $\mathrm{NO}_2$ or $\mathrm{NH}_2$ group and substituted them with a single hydrogen $\mathrm{H}$. This targeted perturbation is specifically designed to generate realistic organic molecules, avoiding out-of-distribution graphs \citep{zheng2023robust}. 
If the mutagenicity of these molecules is exclusively linked to the presence of the $\mathrm{NO}_2$ or $\mathrm{NH}_2$ groups, we expect most perturbed graphs to have a different prediction than the original one.
However, as shown in \cref{tab:MUTAG-exp}, more than \textit{half} of the graphs did not yield a prediction change, meaning that the $\mathrm{NO}_2$ or $\mathrm{NH}_2$ groups are not the exclusive factor influencing predictions.
Reinforcing this result, \citet{Juntao2022CF2} points out that the original \MUTAG paper \citep{debnath1991structure} explicitly excluded the $\mathrm{NH}_2$ functional group as a cause of mutagenicity, as it is active only under specific conditions. Thus, $\mathrm{NH}_2$ is likely a spurious signal, whereas $\mathrm{NO}_2$ is only one of multiple possible drivers of mutagenicity in the dataset, as shown by our analysis.

\textbf{\Benzene.} Our experiments reported in \cref{tab:all-datasets} show that the GNNs achieving perfect classification is PNA \citep{corso2020principal}, which utilizes edge features. However, if all atoms in the datasets were explicitly represented as nodes, aromatic bonds such as those in the benzene rings could be identified purely from node degrees, as aromatic carbons connect to fewer neighbour atoms. Unfortunately, because hydrogen atoms are not explicitly represented in this dataset, graph connectivity alone cannot capture these differences, and aromaticity is encoded via edge features. This directly implies that GNNs that do not process edge features, like GCN \citep{kipf2016semi}, GIN \citep{xu2018powerful}, and GraphSAGE \citep{hamilton2017inductive}, need to rely on alternative signals outside the true ground-truth six-carbon aromatic ring subgraph, increasing their reliance on spurious correlations. Consequently, even an accuracy $\ge 0.90$ does not indicate a model aligned with the intended ground truth \citep{faber2021comparing}. Models without edge features were used with this dataset in the past \citep{sanchez2020evaluating,muschalik2025exact, pmlr-v235-bui24b, agarwal2023evaluating}, but they could not achieve perfect accuracy.

\textbf{\FluorideCarbonyl and \AlkaneCarbonyl.} 
Although the ground truth motifs are fully discriminative and the datasets contain no label noise, no GNN could reach perfect performance despite hyperparameter optimization (see \cref{tab:all-datasets}). 

\section{Additional details on compiling GNNs from logic}
\label[appendix]{appx:compiler-details}
\subsection{Practical example of a compiled GNN}
\label[appendix]{appx:example-compiled}



In this section, we provide a practical example of how a graded modal logic sentence can be compiled into a GNN, described next.

\begin{example}
\label{ex:compile}
    Let $\Phi(u)=\mathrm{red}(u)\wedge\Diamond_{\ge 2}\,\mathrm{blue}$, \ie \textit{a red node with at least two blue neighbours}.
    Its subformulas are $\mathrm{sub}(\Phi)=[\,\mathrm{red},\; \mathrm{blue},\; \Diamond_{\ge 2}\,\mathrm{blue},\; \Phi\,]$, of depths $0,0,1,2$, so the compiled network has width $d=4$ and $L=2$ layers.
    Initialisation copies the two colours into coordinates $1$ and $2$.
    Layer $1$ carries them forward ($C^{(1)}_{11}=C^{(1)}_{22}=1$) and evaluates the diamond ($A^{(1)}_{23}=1$, $b^{(1)}_3=-1$), while layer $2$ carries all three forward and evaluates the conjunction ($C^{(2)}_{14}=C^{(2)}_{34}=1$, $b^{(2)}_4=-1$), all remaining entries being zero.
    Unrolling \cref{eq:acgnn} at coordinate $4$ gives
    \[
        \big(\vh^{(2)}_u\big)_4 \;=\; \sigma\Big(\underbrace{\big(\vh^{(0)}_u\big)_1}_{\mathrm{red}(u)}
        \;+\; \sigma\Big(\sum_{v \in \calN(u)}\underbrace{\big(\vh^{(0)}_v\big)_2}_{\mathrm{blue}(v)} \;-\; 1\Big) \;-\; 1\Big) ,
    \]
    which equals $1$ exactly when $u$ is red and at least two of its neighbours are blue.
\end{example}

The same classifier is obtained in our library by parsing its string representation and compiling it:
\begin{lstlisting}
from compiler import parse, compile_formula_heterogeneous as compile
from compiler.pyg_module import from_compiled_heterogeneous_pyg as to_pyg
matrices = compile(parse("Col(red) & <2> Col(blue)")) # $\smash{\{C^{(\ell)}\}_L,\{A^{(\ell)}\}_L,\{b^{(\ell)}\}}_L$
gnn = to_pyg(matrices)                                # torch.nn.Module
\end{lstlisting}

\subsection{Desugaring the derived connectives}
\label[appendix]{appx:desugaring}

\cref{eq:gml} lists only the core operators the compiler writes into weights.
Every other connective we use is rewritten into that core before compilation.
Propositional connectives are standard, $\varphi\vee\psi\equiv\neg(\neg\varphi\wedge\neg\psi)$, $\varphi\rightarrow\psi\equiv\neg(\varphi\wedge\neg\psi)$ and $\varphi\leftrightarrow\psi\equiv(\varphi\rightarrow\psi)\wedge(\psi\rightarrow\varphi)$, and so is the dual modality $\Box\varphi\equiv\neg\Diamond_{\ge 1}\neg\varphi$.
The counting modalities are obtained from the graded diamond alone:
%
\begin{align*}
    \Diamond_{\le n}\varphi &\;\equiv\; \neg\,\Diamond_{\ge n+1}\varphi , \\
    \Diamond_{=n}\varphi    &\;\equiv\; \Diamond_{\ge n}\varphi \wedge \neg\,\Diamond_{\ge n+1}\varphi
        \quad\text{(if } n>0\text{)}, \\
    \Diamond_{=0}\varphi    &\;\equiv\; \neg\,\Diamond_{\ge 1}\varphi ,\\
    \Diamond_{[n,m]}\varphi &\;\equiv\; \Diamond_{\ge n}\varphi \wedge \neg\,\Diamond_{\ge m+1}\varphi .
\end{align*}
The comparative operator $\Diamond_{>}(\varphi,\psi)$, reading \textit{``more neighbours satisfy $\varphi$ than $\psi$''}, is the one construct that is not a simple rewriting, since a comparison between two unbounded counts is not expressible in graded modal logic, the fragment of $\text{FOC}_2$ used in \citet{barcelo2020logical}.
We therefore fix a bound $n$ on the counts, which in our datasets is the maximum degree, and write
\[
\label{eq:comparative}
    \Diamond_{>}(\varphi,\psi) \;\equiv\; \bigvee_{m=1}^{n}\Big(\Diamond_{\ge m}\varphi \wedge \neg\,\Diamond_{\ge m}\psi\Big) ,
\]
\ie there is a threshold $m$ met by the $\varphi$-neighbours and missed by the $\psi$-neighbours. 
The variant $\Diamond_{=}(\varphi,\psi)$ replaces the clause in parenthesis with $\Diamond_{=m}\varphi\wedge\Diamond_{=m}\psi$ and starts the disjuction from $m=0$.
Written with binary connectives, each disjunction costs three layers, since $\varphi\vee\psi\equiv\neg(\neg\varphi\wedge\neg\psi)$, so \cref{eq:comparative} would compile to a GNN of depth $3n$.
We avoid this by compiling flat $n$-ary conjunction and disjunction gates directly, in a single column each,
\[
\label{eq:flat-gates}
    \bigwedge_{k=1}^{n}\varphi_k : \; C^{(\ell)}_{ki}=1 \;\;\forall k, \;\; b^{(\ell)}_i = -(n-1) ,
    \qquad
    \bigvee_{k=1}^{n}\varphi_k : \; C^{(\ell)}_{ki}=1 \;\;\forall k, \;\; b^{(\ell)}_i = 0 ,
\]
which are correct because the incoming coordinates are Boolean and $\sigma$ saturates at $1$.
The comparative operator then compiles to a GNN of depth $4$ regardless of $n$: one layer aggregates all the thresholds $\Diamond_{\ge m}\varphi$ and $\Diamond_{\ge m}\psi$ in parallel, one negates, one forms the clauses, and one disjoins them with the flat gate of \cref{eq:flat-gates}.
%
The $n$ thresholds are thus absorbed into the width of a layer rather than the depth of the compiled model.
Note that the thresholds can share a single layer because each occupies its own column of $A^{(\ell)}$: these columns all read the same neighbour sum and differ only in their bias entry $b^{(\ell)}_i = 1-m$.

\subsection{Ablation on comparative operator implementation}
\label{appx:ablation-comparative}

The comparative operator, as previously mentioned, is not expressible in the original fragment of graded modal logic used in \citet{barcelo2020logical}. 
However, during the writing of this manuscript, we found out that \citet{Nunn2024logic} extended the results of \citet{barcelo2020logical} to express comparisons and, more generally, linear combinations of counting modalities. 
Therefore, we ablated our original implementation (which is exact only when the bound $n$ is equal to or higher than the maximum degree of the graph) with one derived directly from \citet{Nunn2024logic}:
\[
\label{eq:alt-comparative}
    \Diamond_{>}(\varphi_j,\psi_k) : \; A^{(\ell)}_{ji}=1 ,\;\; A^{(\ell)}_{ki}=-1,  \;\; b^{(\ell)}_i = 0.
\]
We report the ablation results in \cref{tab:alt-comparative}.

\textbf{Results \& Discussion.}
Most explainers in \cref{tab:alt-comparative} are unaffected by the change. \gnnexplainer shows minor variations on \MAJTWOHOP, and while \graphmask improves in recall, its overall $F_1$ score remains within the margin of noise.
These fluctuations are due to the fact that explainers rely on the underlying model: \graphmask initialization depends on the input/output dimensions of each layer; \gnnexplainer, instead, is affected by the internal gradients of the model and the output of the layers. \saliency and \lrp, instead, are noticeably affected by the change because the alternative implementation introduces negative gradients.
Because \saliency returns absolute gradient values, edges with large negative gradients (e.g., from a blue node to the target in \MAJONEHOP) are incorrectly assigned high importance, reducing precision in both tasks.
For \lrp, the final relevance of an edge is computed as the product of its subsequent edges' relevancies.
An edge from a red node to a blue-ish intermediary has a negative gradient for that intermediary; however, because the intermediary itself has a negative gradient with respect to the final target, their product results in a false positive relevance. This compounding effect explains why the performance drops slightly on \MAJTWOHOP but remains identical on \MAJONEHOP. 

\begin{table}[t]
    \centering
    \small
    \setlength{\tabcolsep}{4pt}
    \scalebox{0.90}{
    \begin{tabular}{l ccc ccc}
        \toprule
         & \multicolumn{3}{c}{Original ($\Diamond_{>}$ expanded)} & \multicolumn{3}{c}{Alternative ($\Diamond_{>}$ direct)} \\
        \cmidrule(lr){2-4}\cmidrule(lr){5-7}
        Explainer & precision & recall & $F_1$ & precision & recall & $F_1$ \\
        \midrule
        \multicolumn{7}{c}{\MAJONEHOP} \\
        \midrule
        \saliency                & $\nentry{1.00}{0.00}$ & $\nentry{1.00}{0.00}$ & $\nentry{1.00}{0.00}$ & $\nentry{0.90}{0.01}$ & $\nentry{1.00}{0.00}$ & $\nentry{0.94}{0.00}$ \\
        \intgrad                 & $\nentry{1.00}{0.00}$ & $\nentry{1.00}{0.00}$ & $\nentry{1.00}{0.00}$ & $\nentry{1.00}{0.00}$ & $\nentry{1.00}{0.00}$ & $\nentry{1.00}{0.00}$ \\
        \deconv                  & $\nentry{1.00}{0.00}$ & $\nentry{1.00}{0.00}$ & $\nentry{1.00}{0.00}$ & $\nentry{1.00}{0.00}$ & $\nentry{1.00}{0.00}$ & $\nentry{1.00}{0.00}$ \\
        \guidedbp                & $\nentry{1.00}{0.00}$ & $\nentry{1.00}{0.00}$ & $\nentry{1.00}{0.00}$ & $\nentry{1.00}{0.00}$ & $\nentry{1.00}{0.00}$ & $\nentry{1.00}{0.00}$ \\
        \gnnexplainer            & $\nentry{1.00}{0.00}$ & $\nentry{1.00}{0.00}$ & $\nentry{1.00}{0.00}$ & $\nentry{1.00}{0.00}$ & $\nentry{1.00}{0.00}$ & $\nentry{1.00}{0.00}$ \\
        \graphmask               & $\nentry{0.00}{0.00}$ & $\nentry{1.00}{0.00}$ & $\nentry{0.00}{0.00}$ & $\nentry{0.00}{0.00}$ & $\nentry{0.99}{0.01}$ & $\nentry{0.00}{0.00}$ \\
        \subgraphx               & $\nentry{0.36}{0.00}$ & $\nentry{0.88}{0.01}$ & $\nentry{0.49}{0.01}$ & $\nentry{0.36}{0.00}$ & $\nentry{0.88}{0.01}$ & $\nentry{0.49}{0.01}$ \\
        \pgmexplainer            & $\nentry{0.43}{0.01}$ & $\nentry{0.99}{0.00}$ & $\nentry{0.59}{0.01}$ & $\nentry{0.43}{0.01}$ & $\nentry{0.99}{0.00}$ & $\nentry{0.59}{0.01}$ \\
        \graphsvx                & $\nentry{0.50}{0.00}$ & $\nentry{1.00}{0.00}$ & $\nentry{0.66}{0.00}$ & $\nentry{0.50}{0.00}$ & $\nentry{1.00}{0.00}$ & $\nentry{0.66}{0.00}$ \\
        \flowx                   & $\nentry{1.00}{0.00}$ & $\nentry{1.00}{0.00}$ & $\nentry{1.00}{0.00}$ & $\nentry{1.00}{0.00}$ & $\nentry{1.00}{0.00}$ & $\nentry{1.00}{0.00}$ \\
        \lrp                     & $\nentry{1.00}{0.00}$ & $\nentry{1.00}{0.00}$ & $\nentry{1.00}{0.00}$ & $\nentry{1.00}{0.00}$ & $\nentry{1.00}{0.00}$ & $\nentry{1.00}{0.00}$ \\
        \cmidrule(lr){1-7}
        \randombaseline          & $\nentry{0.33}{0.17}$ & $\nentry{0.81}{0.06}$ & $\nentry{0.41}{0.18}$ & $\nentry{0.33}{0.17}$ & $\nentry{0.81}{0.06}$ & $\nentry{0.41}{0.18}$ \\
        \midrule
        \multicolumn{7}{c}{\MAJTWOHOP} \\
        \midrule
        \saliency                & $\nentry{1.00}{0.00}$ & $\nentry{0.87}{0.02}$ & $\nentry{0.92}{0.02}$ & $\nentry{0.67}{0.01}$ & $\nentry{0.87}{0.02}$ & $\nentry{0.73}{0.01}$ \\
        \intgrad                 & $\nentry{0.96}{0.01}$ & $\nentry{1.00}{0.00}$ & $\nentry{0.97}{0.01}$ & $\nentry{0.96}{0.01}$ & $\nentry{1.00}{0.00}$ & $\nentry{0.97}{0.01}$ \\
        \deconv                  & $\nentry{0.94}{0.02}$ & $\nentry{1.00}{0.00}$ & $\nentry{0.96}{0.01}$ & $\nentry{0.94}{0.02}$ & $\nentry{1.00}{0.00}$ & $\nentry{0.96}{0.01}$ \\
        \guidedbp                & $\nentry{1.00}{0.00}$ & $\nentry{0.87}{0.02}$ & $\nentry{0.92}{0.02}$ & $\nentry{1.00}{0.00}$ & $\nentry{0.87}{0.02}$ & $\nentry{0.92}{0.02}$ \\
        \gnnexplainer            & $\nentry{0.96}{0.01}$ & $\nentry{0.99}{0.00}$ & $\nentry{0.97}{0.01}$ & $\nentry{0.94}{0.02}$ & $\nentry{1.00}{0.00}$ & $\nentry{0.96}{0.01}$ \\
        \graphmask               & $\nentry{0.02}{0.01}$ & $\nentry{0.78}{0.12}$ & $\nentry{0.03}{0.01}$ & $\nentry{0.01}{0.00}$ & $\nentry{0.96}{0.06}$ & $\nentry{0.02}{0.00}$ \\
        \subgraphx               & \multicolumn{3}{c}{\OOT} & \multicolumn{3}{c}{\OOT} \\
        \pgmexplainer            & $\nentry{0.30}{0.02}$ & $\nentry{0.83}{0.01}$ & $\nentry{0.41}{0.02}$ & $\nentry{0.30}{0.02}$ & $\nentry{0.83}{0.01}$ & $\nentry{0.41}{0.02}$ \\
        \graphsvx                & $\nentry{0.56}{0.02}$ & $\nentry{0.74}{0.02}$ & $\nentry{0.61}{0.02}$ & $\nentry{0.56}{0.02}$ & $\nentry{0.74}{0.02}$ & $\nentry{0.61}{0.02}$ \\
        \flowx                   & $\nentry{0.91}{0.01}$ & $\nentry{0.95}{0.01}$ & $\nentry{0.92}{0.01}$ & $\nentry{0.91}{0.01}$ & $\nentry{0.95}{0.01}$ & $\nentry{0.92}{0.01}$ \\
        \lrp                     &  $\nentry{1.00}{0.01}$ & $\nentry{0.89}{0.02}$ & $\nentry{0.93}{0.01}$ & $\nentry{0.96}{0.01}$ & $\nentry{0.87}{0.02}$ & $\nentry{0.90}{0.02}$ \\
        \cmidrule(lr){1-7}
        \randombaseline          & $\nentry{0.22}{0.01}$ & $\nentry{0.51}{0.04}$ & $\nentry{0.29}{0.01}$ & $\nentry{0.22}{0.01}$ & $\nentry{0.51}{0.04}$ & $\nentry{0.29}{0.01}$ \\
        \bottomrule
    \end{tabular}}
    \caption{
        \textbf{Plausibility under the ablation of the comparative operator}, averaged over $5$ seeds.
        The original columns compile $\Diamond_{>}(\varphi,\psi)$ as the capped disjunction of \cref{eq:comparative}, the alternative ones as a single aggregation layer defined in \cref{eq:alt-comparative}.
        The graphs, the labels, and \UPI are identical in both cases and equal those of \cref{tab:explainers-nc-majority}; only the compiled GNN changes, from $4$ to $1$ layers on \MAJONEHOP and from $8$ to $2$ on \MAJTWOHOP, and from $203$ to $3$ and $505$ to $5$ hidden dimensions respectively.
        %
        A `\OOT' marks an explainer that exceeded the per-explanation time budget on every seed.
    }
    \label{tab:alt-comparative}
\end{table}

\subsection{Keeping the backward pass of a compiled model informative}
\label[appendix]{appx:compiled-backward}

A compiled network computes with integers, as by \cref{eq:invariant} every coordinate is Boolean, and every weight in \cref{tab:weights} is an integer.
Consequently $\sigma$ is always evaluated on the boundaries of the activations ($z\in\{0,1\}$) or on a saturated flat ($z\le-1$ or $z\ge 2$), where the gradient $\sigma'(z)$ is $0$, and this may create problems for gradient-based explainers.
To avoid issues with meaningless gradients, we therefore replace the derivative of $\sigma$ with
\[
\label{eq:ste}
    \widetilde{\sigma}'(z) \;\defeq\; \Ind{z>0} ,
\]
and leave the forward pass untouched.
This is applied for every explainer, except for \intgrad for the reasons described below.
Note that \cref{eq:ste} is the derivative of a \textsc{ReLU} rather than of a truncated one, and provides the backward gradient that a conventional architecture would provide at that unit.
\cref{eq:ste} is also semantically meaningful, since for every pattern in \cref{tab:weights} the condition $z>0$ coincides with the truth of the subformula the unit computes: $z=1-\alpha_k$ for $\neg\varphi_k$, $z=\alpha_j+\alpha_k-1$ for $\varphi_j\wedge\varphi_k$, $z=\nu-N+1$ for $\Diamond_{\ge N}\varphi_k$, and $z=\alpha_k$ for a carried-forward unit, where $\alpha_k$ is the incoming value of coordinate $k$ and $\nu$ the number of neighbours satisfying $\varphi_k$.
Gradient therefore flows along the subformulas that are satisfied.

\textbf{\intgrad is evaluated with the exact gradient.}
\intgrad is the only explainer for which we keep the unmodified derivative of $\sigma$.
This is because, otherwise, the change may violate the Completeness axiom \citep{sundararajan2017axiomatic}, as attributions may not sum up to the difference in the output between the original samples and the baseline of \intgrad.
Our ablation study in \cref{tab:ig-backward} confirms that \intgrad is nonetheless robust to this choice: the two backward passes yield identical scores on five tasks out of six, and on \MAJTWOHOP the surrogate gradient gains $0.04$ precision in exchange for $0.13$ in recall, the final $F_1$ being $0.92$ instead of the original $0.97$ reported in \cref{tab:explainers}.
%
%

\begin{table}[h!]
    \centering
    \small
    \setlength{\tabcolsep}{4pt}
    \scalebox{0.90}{
    \begin{tabular}{l ccc}
        \toprule
        Task & precision & recall & $F_1$ \\
        \midrule
        \REDBLUELEAF     & $\nentry{1.00}{0.00}$ & $\nentry{1.00}{0.00}$ & $\nentry{1.00}{0.00}$ \\
        \REDUNDANCY      & $\nentry{1.00}{0.00}$ & $\nentry{1.00}{0.00}$ & $\nentry{1.00}{0.00}$ \\
        \COMPREDUNDANCY  & $\nentry{1.00}{0.00}$ & $\nentry{0.50}{0.00}$ & $\nentry{0.67}{0.00}$ \\
        \midrule
        \REDBLUELEAFNC   & $\nentry{1.00}{0.00}$ & $\nentry{1.00}{0.00}$ & $\nentry{1.00}{0.00}$ \\
        \MAJONEHOP       & $\nentry{1.00}{0.00}$ & $\nentry{1.00}{0.00}$ & $\nentry{1.00}{0.00}$ \\
        \MAJTWOHOP       & $\nentry{1.00}{0.00}$ & $\nentry{0.87}{0.02}$ & $\nentry{0.92}{0.02}$ \\
        \bottomrule
    \end{tabular}}
    \caption{
        \textbf{\intgrad under the surrogate backward pass} of \cref{eq:ste}, averaged over $5$ seeds.
        The graphs, the labels, \UPI, and the forward pass are identical to those in \cref{tab:explainers} and \cref{tab:explainers-nc-majority}.
    }
    \label{tab:ig-backward}
\end{table}

\subsection{Ablating how the formula is written}
\label[appendix]{appx:compilation-choices}

In standard training, we cannot control how a model internally parametrizes a learned function.
Therefore, when functionally equivalent models exist for a given task, a reliable explainer should ideally perform consistently regardless of the underlying architecture \citep{sundararajan2017axiomatic, nauta2023anecdotal}. 
Because \METHODNAME models are compiled rather than trained, we have full control over the underlying implementation. 
This allows us to turn conditions that are usually inextricably confounded in trained models -- such as vanishing gradients, layer saturation, or specific weight configurations -- into controlled experiments. 
In this section, we test the robustness of explainers by observing how they react when they are evaluated on different but logically equivalent compiled GNNs. 
In the following, we test an alternative choice to compile the disjunction operator for \COMPREDUNDANCY, and an alternative logic formula for the \REDBLUELEAF task.

\textbf{Alternative disjunction operator for \COMPREDUNDANCY.}
Our original implementation of \COMPREDUNDANCY encodes the ``$\lor$'' operator via De Morgan's laws ($\varphi\vee\psi\equiv\neg(\neg\varphi\wedge\neg\psi)$, see \cref{appx:desugaring}), which requires intermediate subformulas with nested ``$\land$'' and ``$\neg$'' operations. 
To evaluate the sensitivity of explainers to alternative compilation choices, we compare the original implementation reported in the main text against an alternative formulation using the $n$-ary, flat definition of ``$\lor$'', described in \cref{eq:flat-gates}. 
While input graphs, labels, and \UPI remain identical to those in \cref{tab:explainers}, this alternative compilation shrinks the GNN from $8$ to $3$ layers and from $16$ to $9$ hidden dimensions, as it does not require nested ``$\land$'' and ``$\neg$'' operations. 

As \cref{tab:compilations-disjunction} shows, the two implementations yield drastically different results: five explainers (\saliency, \deconv, \guidedbp, \gnnexplainer, and \lrp) jump from random-level to perfect performance.
%
%
This is because, by \cref{eq:ste}, the gradient can only travel along satisfied subformulas, while a negation is satisfied exactly when its argument is not, and the two are never satisfied together. 
Crucially, De Morgan's disjunction puts a negation on every path from the output to its branches, so no gradient ever reaches them.\footnote{The same absence of signal would arise with the exact derivative of $\sigma$ in \cref{eq:acgnn}, but for both compilations alike, since for positive instances of \COMPREDUNDANCY all three colour-group clauses hold and therefore gates saturate, yielding zero gradient.} Conversely, the flat gate disjunction of \cref{eq:flat-gates} contains no negation at all, and explainers may benefit from this.
Only \intgrad, \flowx, and \graphsvx maintain consistent results across both formulations, indicating a more robust behaviour. 

\begin{table}[t]
    \centering
    \small
    \setlength{\tabcolsep}{4pt}
    \caption{
        \textbf{Plausibility on \COMPREDUNDANCY over $5$ seeds, for two compilations of the same disjunction}, see \cref{appx:compilation-choices}.
        The original columns report the implementation of the disjunction ``$\varphi\vee\psi$'' with De Morgan's law ``$\neg(\neg\varphi\wedge\neg\psi)$''.
        The alternative ones use the flat $n$-ary gate (\cref{eq:flat-gates}), which puts no negation on any path.
    }
    \scalebox{1.0}{
    \begin{tabular}{l ccc ccc}
        \toprule
         & \multicolumn{3}{c}{Original ($\vee$ via De Morgan)} & \multicolumn{3}{c}{Alternative ($\vee$ flat gate)} \\
        \cmidrule(lr){2-4}\cmidrule(lr){5-7}
        Explainer & precision & recall & $F_1$ & precision & recall & $F_1$ \\
        \midrule
        \saliency                & $\nentry{0.04}{0.00}$ & $\nentry{1.00}{0.00}$ & $\nentry{0.09}{0.00}$ & $\nentry{1.00}{0.00}$ & $\nentry{1.00}{0.00}$ & $\nentry{1.00}{0.00}$ \\
        \intgrad                 & $\nentry{1.00}{0.00}$ & $\nentry{0.50}{0.00}$ & $\nentry{0.67}{0.00}$ & $\nentry{1.00}{0.00}$ & $\nentry{0.50}{0.00}$ & $\nentry{0.67}{0.00}$ \\
        \deconv                  & $\nentry{0.04}{0.00}$ & $\nentry{1.00}{0.00}$ & $\nentry{0.09}{0.00}$ & $\nentry{1.00}{0.00}$ & $\nentry{1.00}{0.00}$ & $\nentry{1.00}{0.00}$ \\
        \guidedbp                & $\nentry{0.04}{0.00}$ & $\nentry{1.00}{0.00}$ & $\nentry{0.09}{0.00}$ & $\nentry{1.00}{0.00}$ & $\nentry{1.00}{0.00}$ & $\nentry{1.00}{0.00}$ \\
        \gnnexplainer            & $\nentry{0.04}{0.00}$ & $\nentry{1.00}{0.00}$ & $\nentry{0.09}{0.00}$ & $\nentry{1.00}{0.00}$ & $\nentry{1.00}{0.00}$ & $\nentry{1.00}{0.00}$ \\
        \graphmask               & $\nentry{0.04}{0.00}$ & $\nentry{1.00}{0.00}$ & $\nentry{0.09}{0.00}$ & $\nentry{0.05}{0.00}$ & $\nentry{1.00}{0.00}$ & $\nentry{0.09}{0.01}$ \\
        \subgraphx               & $\nentry{0.19}{0.01}$ & $\nentry{0.69}{0.01}$ & $\nentry{0.15}{0.00}$ & $\nentry{0.18}{0.01}$ & $\nentry{0.72}{0.02}$ & $\nentry{0.14}{0.00}$ \\
        \pgmexplainer            & $\nentry{0.35}{0.16}$ & $\nentry{0.56}{0.08}$ & $\nentry{0.39}{0.15}$ & $\nentry{0.25}{0.18}$ & $\nentry{0.49}{0.35}$ & $\nentry{0.25}{0.18}$ \\
        \graphsvx                & $\nentry{0.47}{0.03}$ & $\nentry{0.74}{0.16}$ & $\nentry{0.56}{0.06}$ & $\nentry{0.47}{0.03}$ & $\nentry{0.74}{0.16}$ & $\nentry{0.56}{0.06}$ \\
        \flowx                   & $\nentry{1.00}{0.00}$ & $\nentry{0.97}{0.07}$ & $\nentry{0.98}{0.04}$ & $\nentry{1.00}{0.00}$ & $\nentry{0.97}{0.07}$ & $\nentry{0.98}{0.04}$ \\
        \lrp                     & $\nentry{0.04}{0.00}$ & $\nentry{1.00}{0.00}$ & $\nentry{0.09}{0.00}$ & $\nentry{1.00}{0.00}$ & $\nentry{1.00}{0.00}$ & $\nentry{1.00}{0.00}$ \\
        \midrule
        \randombaseline          & $\nentry{0.06}{0.01}$ & $\nentry{0.73}{0.15}$ & $\nentry{0.10}{0.01}$ & $\nentry{0.06}{0.01}$ & $\nentry{0.73}{0.15}$ & $\nentry{0.10}{0.01}$ \\
        \bottomrule
    \end{tabular}}
    \label{tab:compilations-disjunction}
\end{table}


\textbf{Alternative formula for \REDBLUELEAF.}
To further investigate how the syntactical structure of a formula impacts explanations, we compare two distinct compiled models for the \REDBLUELEAF task (see \cref{appx:datasets}). In the original formulation reported in the main text, the node predicate is defined simply as:
\[
    \varphi_\wedge(u) \;=\; \mathrm{midend}(u) \wedge\Diamond_{\ge 1}\,\mathrm{midend} ,
    \qquad \Phi \;=\; \exists u.\,\varphi_\wedge(u) ,
\]
which fires on the two interior nodes of the chain; see \cref{appx:datasets} for the full formula. We compare it against a second, more complex predicate for the same motif, built from the same vocabulary:
\begin{align*}
    \mathrm{leaf}(u)   &\;=\; \mathrm{blue}(u) \wedge\Diamond_{=1}\top
        & &\text{\footnotesize a blue node of degree one,} \\
    \mathrm{core}(u)   &\;=\; \mathrm{red}(u) \wedge\Diamond_{=2}\,\mathrm{leaf}
        & &\text{\footnotesize a red node with exactly two blue leaves,} \\
    \mathrm{end}(u)    &\;=\; \mathrm{core}(u) \wedge\Diamond_{=1}\,\mathrm{core}
        & &\text{\footnotesize endpoint of the chain,} \\
    \mathrm{mid}(u)    &\;=\; \mathrm{core}(u) \wedge\Diamond_{=2}\,\mathrm{core}
        & &\text{\footnotesize interior node of the chain,} \\
    \mathrm{midend}(u) &\;=\; \mathrm{mid}(u) \wedge\Diamond_{\ge 1}\,\mathrm{end}
        & &\text{\footnotesize interior node touching an endpoint,} \\
    \mathrm{midmid}(u) &\;=\; \mathrm{mid}(u) \wedge\Diamond_{\ge 1}\,\mathrm{midend}
        & &\text{\footnotesize interior node beside such an interior node,} \\[2pt]
    \varphi_{\mathrm{end}}(u) &\;=\; \mathrm{end}(u) \wedge\Diamond_{\ge 1}\,\mathrm{midmid}
        & &\text{\footnotesize the two endpoints, along the whole chain,} \\
    \varphi_{\mathrm{int}}(u) &\;=\; \mathrm{mid}(u) \wedge\Diamond_{\ge 1}\,\mathrm{end} \wedge\Diamond_{\ge 1}\,\mathrm{midend}
        & &\text{\footnotesize the two interior nodes}\\
\end{align*}
and closed with the same graph-level quantifier,
\[
    \varphi_\vee(u) \;=\; \varphi_{\mathrm{end}}(u) \vee\varphi_{\mathrm{int}}(u) ,
    \qquad \Phi \;=\; \exists u.\,\varphi_\vee(u) ,
\]
which instead fires on all four nodes of the chain. Despite having different depths, the two models agree with the intended label on all graphs of the dataset, and re-deriving the ground truth from either of them with \cref{alg:main} returns the same \UPI on every graph. 

When evaluating this second ablation (\cref{tab:compilations}), we observe a phenomenon remarkably similar to the disjunction ablation above.
%
\saliency, \guidedbp, and \lrp drop to chance-level performance under the deeper $\varphi_\vee$ formulation, meaning they are not robust to different model implementations.
Conversely, \intgrad, \graphsvx, and \gnnexplainer remain unaffected. \flowx, instead, runs out of time due to the deeper model architecture.

\begin{table}[t]
    \centering
    \small
    \setlength{\tabcolsep}{4pt}
    \scalebox{0.99}{
    \begin{tabular}{l ccc ccc}
        \toprule
         & \multicolumn{3}{c}{Original formula ($\varphi_\wedge$)} & \multicolumn{3}{c}{Alternative formula ($\varphi_\vee$)} \\
        \cmidrule(lr){2-4}\cmidrule(lr){5-7}
        Explainer & precision & recall & $F_1$ & precision & recall & $F_1$ \\
        \midrule
        \saliency                & $\nentry{1.00}{0.00}$ & $\nentry{1.00}{0.00}$ & $\nentry{1.00}{0.00}$ & $\nentry{0.22}{0.00}$ & $\nentry{1.00}{0.00}$ & $\nentry{0.37}{0.00}$ \\
        \intgrad                 & $\nentry{1.00}{0.00}$ & $\nentry{1.00}{0.00}$ & $\nentry{1.00}{0.00}$ & $\nentry{1.00}{0.00}$ & $\nentry{1.00}{0.00}$ & $\nentry{1.00}{0.00}$ \\
        \deconv                  & $\nentry{0.22}{0.00}$ & $\nentry{1.00}{0.00}$ & $\nentry{0.37}{0.00}$ & $\nentry{0.22}{0.00}$ & $\nentry{1.00}{0.00}$ & $\nentry{0.37}{0.00}$ \\
        \guidedbp                & $\nentry{1.00}{0.00}$ & $\nentry{1.00}{0.00}$ & $\nentry{1.00}{0.00}$ & $\nentry{0.22}{0.00}$ & $\nentry{1.00}{0.00}$ & $\nentry{0.37}{0.00}$ \\
        \gnnexplainer            & $\nentry{1.00}{0.00}$ & $\nentry{1.00}{0.00}$ & $\nentry{1.00}{0.00}$ & $\nentry{1.00}{0.00}$ & $\nentry{1.00}{0.00}$ & $\nentry{1.00}{0.00}$ \\
        \graphmask               & $\nentry{0.25}{0.03}$ & $\nentry{0.93}{0.15}$ & $\nentry{0.38}{0.02}$ & $\nentry{0.25}{0.03}$ & $\nentry{0.89}{0.19}$ & $\nentry{0.38}{0.02}$ \\
        \subgraphx               & $\nentry{0.40}{0.01}$ & $\nentry{0.84}{0.01}$ & $\nentry{0.38}{0.00}$ & $\nentry{0.40}{0.01}$ & $\nentry{0.84}{0.01}$ & $\nentry{0.38}{0.00}$ \\
        \pgmexplainer            & $\nentry{0.33}{0.16}$ & $\nentry{0.89}{0.19}$ & $\nentry{0.45}{0.15}$ & $\nentry{0.33}{0.16}$ & $\nentry{0.89}{0.19}$ & $\nentry{0.45}{0.15}$ \\
        \graphsvx                & $\nentry{1.00}{0.00}$ & $\nentry{1.00}{0.00}$ & $\nentry{1.00}{0.00}$ & $\nentry{1.00}{0.00}$ & $\nentry{1.00}{0.00}$ & $\nentry{1.00}{0.00}$ \\
        \flowx                   & $\nentry{0.90}{0.03}$ & $\nentry{0.99}{0.01}$ & $\nentry{0.94}{0.01}$ & \multicolumn{3}{c}{\OOT} \\
        \lrp                     & $\nentry{1.00}{0.00}$ & $\nentry{1.00}{0.00}$ & $\nentry{1.00}{0.00}$ & $\nentry{0.22}{0.00}$ & $\nentry{1.00}{0.00}$ & $\nentry{0.37}{0.00}$ \\
        \midrule
        \randombaseline          & $\nentry{0.23}{0.02}$ & $\nentry{0.93}{0.16}$ & $\nentry{0.37}{0.00}$ & $\nentry{0.23}{0.02}$ & $\nentry{0.93}{0.16}$ & $\nentry{0.37}{0.00}$ \\
        \bottomrule
    \end{tabular}}
    \caption{\textbf{Plausibility on \REDBLUELEAF over $5$ seeds, for two different formulas}, see \cref{appx:compilation-choices}. The graphs, the labels, and \UPI are identical for both formulas and equal to those evaluated in \cref{tab:explainers}. Only the compiled GNN changes.     Given that the alternative, more complex formula yields a GNN of greater depth, execution is slower and \flowx exceeds the time budget.
    }
    \label{tab:compilations}
\end{table}

\textbf{Discussion.}
Taken together, these two ablations highlight a fundamental vulnerability in many popular explainers: simple gradient-based and \lrp are highly susceptible to different implementations of the model being explained, meaning they do not fulfill the implementation invariant \citep{sundararajan2017axiomatic, nauta2023anecdotal}. 
While they may achieve perfect scores on a specific formulation (e.g., the flat $\lor$ operator or the shallow $\varphi_\wedge$ predicate), they fail when the same logic is parametrized differently. 
In contrast, \intgrad and \graphsvx remain stable across implementations, indicating more robust behaviours.
%
Because our models are compiled rather than trained, we have full control over what is being explained.
This turns conditions -- like vanishing gradients -- that are usually inextricably confounded on a trained model into controlled experiments, and lets one state and test a specific hypothesis rigorously.
For instance, one can test how an explainer copes with saturated activations by testing equivalent models compiled at different depths, or with logically equivalent but syntactically different connectives, while ensuring the same ground truth holds.
We leave a systematic ablation over these experimental conditions to future work.

\section{Additional details about \ROBDDs and how to extract \UPI}
\label[appendix]{appx:robdd-details}

\subsection{Definition of \ROBDDs and a practical example}
\label[appendix]{appx:robdd-example}

In this section, we aim to provide a compact overview of how the \ROBDDs used in \cref{sec:upi} are defined, and then provide a practical example of an \ROBDD over graphs.
We start by reviewing what \ROBDDs are.

\begin{definition}[\OBDD and \ROBDD]
    \label{def:robdd}
    Let $X=\{x_1, \dots, x_n\}$ be a set of binary variables and $\pi$ an ordering of the elements of $X$.
    An Ordered Binary Decision Diagram (\OBDD) over $X$ with respect to $\pi$ is a rooted, directed acyclic graph, with two sinks called the $1$-sink and $0$-sink. 
    Every node $u$ except the sinks
    is labelled with a tuple $(\mathrm{var}(u), \mathrm{lo}(u), \mathrm{hi}(u))$, where $\mathrm{var}(u) \in X$ indicates the node's variable name, and $\mathrm{lo}(u)$ and $\mathrm{hi}(u)$ indicate the two successors of $u$ in the diagram.
    If there is an edge from a node labelled $x_i$ to a node labelled $x_j$, then $i <_\pi j$.
    The \OBDD is Reduced (\ROBDD) if in addition:
    
    \; (no irrelevant branches) $\mathrm{lo}(u)\neq\mathrm{hi}(u)$ for every internal node $u$, and

    \; (no duplicated sub-diagram) no distinct internal nodes carry the same triple            $\bigl(\mathrm{var}(u),\mathrm{lo}(u),\mathrm{hi}(u)\bigr)$.
\end{definition}

At a high level, an \ROBDD is a tractable representation of a Boolean function over a set of Boolean variables, allowing efficient operations and queries on the function.
In our setting, in particular, we use \ROBDDs to efficiently encode a Boolean function -- i.e., the graded modal logic formula we wish to convert into a GNN; see \cref{sec:compiled-gnn} -- over subgraphs of $G$, allowing us to extract the ground truth explanation \UPI efficiently.
Below, we report an example of an \ROBDD defined using the graph symbolic encodings described in \cref{sec:upi}.
For more on \ROBDDs and its applications, see \citet{lee1959representation, bryant1986graph, bryant1992symbolic, meinel1998obdds, wegener2000branching}.

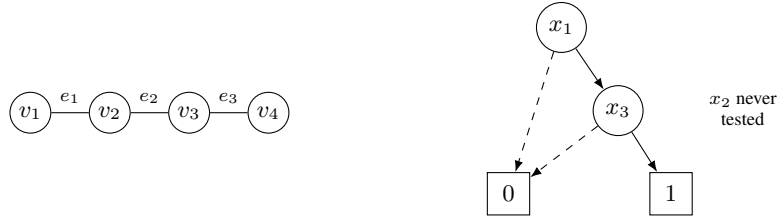
\begin{figure}[t]
    \centering
    \begin{minipage}[c]{0.34\textwidth}
    \centering
    \begin{tikzpicture}
      \foreach \i/\x in {1/0, 2/1.05, 3/2.10, 4/3.15}
        \node[vtx] (v\i) at (\x,0) {$v_{\i}$};
      \draw (v1) -- node[above,font=\scriptsize] {$e_1$} (v2);
      \draw (v2) -- node[above,font=\scriptsize] {$e_2$} (v3);
      \draw (v3) -- node[above,font=\scriptsize] {$e_3$} (v4);
    \end{tikzpicture}
    \end{minipage}%
    \begin{minipage}[c]{0.44\textwidth}
    \centering
    \begin{tikzpicture}
      \node[bnode] (n1) at (0,0)         {$x_1$};
      \node[bnode] (n3) at (0.75,-1.10)  {$x_3$};
      \node[term]  (F)  at (-0.70,-2.20) {$0$};
      \node[term]  (T)  at (1.45,-2.20)  {$1$};
      \draw[lo] (n1) -- (F);
      \draw[hi] (n1) -- (n3);
      \draw[lo] (n3) -- (F);
      \draw[hi] (n3) -- (T);
      \node[font=\scriptsize,align=center] at (2.40,-1.10) {$x_2$ never\\tested};
    \end{tikzpicture}
    \end{minipage}
    \caption{
        The path graph of \cref{ex:path} (left) and the ROBDD of its
        edge-cover function $f=x_1\wedge x_3$ under the order $x_1<x_2<x_3$
        (right).  Dashed arcs are $\mathrm{lo}$ (the variable is $0$, the edge is
        deleted), solid arcs are $\mathrm{hi}$ (the variable is $1$, the edge is
        kept).
    }
\label{fig:robdd-path}
\end{figure}

\begin{example}[Edge cover on a path]
\label{ex:path}
    Let $G$ be the path on four nodes, $\calV=\{v_1,\dots,v_4\}$ and
    $\calE=\{e_1,e_2,e_3\}$ with $e_i=(v_i,v_{i+1})$, and let $f$ be the Boolean
    function holding on $R$ iff $R$ is an edge cover of $G$, i.e., every
    node is incident to at least one edge of $R$.  Writing one clause per node,
    \[
      f \;=\; \bigwedge_{v\in \calV}\ \bigvee_{e\ni v} x_e
      \;=\; \underbrace{x_1}_{v_1}\wedge\underbrace{(x_1\vee x_2)}_{v_2}
            \wedge\underbrace{(x_2\vee x_3)}_{v_3}\wedge\underbrace{x_3}_{v_4}
      \;=\; x_1\wedge x_3 ,
    \]
    because $v_1$ and $v_4$ have degree one, so their clauses are unit and need to be assigned $1$ for $f$ to be evaluated to $1$, resulting in the simplification above.
    The resulting \ROBDD for $f$ over $G$ encodes the following Boolean function, and is shown in \cref{fig:robdd-path}:
    \[
    \begin{array}{@{}l|cccccccc@{}}
      R & \emptyset & \{e_1\} & \{e_2\} & \{e_3\} & \{e_1,e_2\} & \{e_1,e_3\} & \{e_2,e_3\} & \calE \\
      \midrule
      f(R) & 0 & 0 & 0 & 0 & 0 & 1 & 0 & 1
    \end{array}
    \]
    The possible edge covers are two, but the minimal one is $\{e_1,e_3\}$.
    Hence, $\UPI(f, G) = \{e_1,e_3\}$.
\end{example}

To keep the example above compact, we considered the undirected edge $e=(u,v)$ as a single variable boolean $x_e$, but in our implementation we instantiate a single variable per directionality.

\subsection{The $\mathrm{BUILD}$ algorithm for \ROBDDs}
\label[appendix]{appx:build-alg}

In \cref{sec:upi}, we introduced the method $\mathrm{BUILD}$ for constructing an \ROBDD for each formula and graph. In this section, we discuss this procedure in more detail.
In practice, we rely on the $\mathrm{dd}$ Python package \citep{dd} for \ROBDD construction and inference.

Four fundamental operations are used to construct \ROBDDs \citep{bryant1986graph}: \emph{Reduce}, \emph{Apply}, \emph{Restrict}, and \emph{Composition}. Any other operation on \ROBDDs can be expressed as a combination of these procedures \citep{bryant1986graph}. \emph{Reduce} transforms an arbitrary Binary Decision Diagram (\BDD) into a canonical \ROBDD by eliminating duplicate subgraphs and redundant tests. Given two Boolean functions $f$ and $h$ implemented as \ROBDDs and a binary Boolean operator $\langle{op}\rangle$, \emph{Apply} generates a \BDD implementing $f\langle{op}\rangle h$ (and can also be used to implement logical negation and induction, as detailed in \citet{bryant1986graph}). \emph{Restrict} simplifies an input \BDD into another \BDD by fixing a specific variable $x$ to a constant value $b$. Finally, \emph{Composition} replaces a variable $x$ in a \ROBDD $f$ with another function implemented by a \ROBDD $h$. Because \emph{Apply}, \emph{Restrict}, and \emph{Composition} yield general \BDDs, \emph{Reduce} must always be applied post-operation. In our implementation, this workflow is handled automatically by the $\mathrm{dd}$ package \citep{dd}. In our paper, we treat Boolean operations directly as operations on the \ROBDDs themselves, using standard logical notation (e.g., $\land, \neg, \cof{}{}{1}$) to implicitly denote the execution of these underlying graph-construction algorithms. We denote the constant true and false functions by $1$ and $0$, respectively.


There are two variations of the $\mathrm{BUILD}$ algorithm: one for node tasks (\proc{BuildNode} in \cref{alg:buildnode}), and one for graph tasks (\proc{BuildGraph} in  \cref{alg:buildgraph}). In the latter case, we take as input only $\Phi$ and the graph $G$, while in the former we also need the target node $v$. Since \proc{BuildGraph} internally uses \proc{BuildNode}, we start our description from \proc{BuildNode}. 


\proc{BuildNode} (\cref{alg:buildnode}) is recursive: when a formula $\Phi$ is the conjunction of two subformulas $\varphi_j$ and $\varphi_k$, the procedure recursively constructs the \ROBDDs for $\varphi_j$ and $\varphi_k$ separately, then conjoins them. Analogous recursive steps handle disjunction and negation. For base cases, i.e., the constants $\top$, $\bot$, and the atomic predicates $\varphi(v) = \mathrm{Col}(v)$, the \ROBDD is initialized directly as $1$, $0$, or the constant $\mathrm{Col}(v)$. The most nuanced step is handling the counting quantifier $\Diamond_{\ge N}\,\psi$, which requires constructing a Boolean function $C_N$ that evaluates to true if and only if at least $N$ neighbours of $v$ satisfy $\psi$.
For each incident edge $e=(v,u)$, the algorithm computes $t = x_e \wedge \proc{BuildNode}(\psi, u, G)$, indicating edge existence and that $u$ satisfies $\psi$, and then iteratively updates the intermediate counting functions $C_k,\,k\in \{1,\dots,N\}$ (line \ref{line:edgecount}).
After evaluating all incident edges, it then returns the counting function $C_N$, which now indicates whether node $v$ has at least $N$ neighbours satisfying $\psi$.

\proc{BuildGraph} (\cref{alg:buildgraph}) builds directly upon \cref{alg:buildnode}. The graph-level classifier is framed as a counting quantifier $\exists^{[\mathrm{lo}, \mathrm{hi}]}$ that verifies whether a node-level formula $\varphi$ is satisfied by at least $\mathrm{lo}$ and at most $\mathrm{hi}$ nodes, i.e. it express the formula: $\Phi = \left(\exists^{\ge \mathrm{lo}}\,u.\,\varphi(u)\right)\wedge\neg\left(\exists^{\ge \mathrm{hi+1}}\,u.\,\varphi(u)\right)$.
The case where at least $N$ nodes are fired is then defined by setting $\mathrm{lo}=N$ and $\mathrm{hi}=\infty$ (symmetrically, the case where maximum $N$ nodes are fired is defined by setting $\mathrm{lo}=0$ and $\mathrm{hi}=N$).
Therefore, we iterate over each node $v\in\calV$, construct its corresponding \ROBDD for formula $\varphi$ via \proc{BuildNode}, and iteratively update the global threshold functions $C_1, \dots, C_{\vert{}\calV\vert{}}$ (\cref{line:nodecount}), analogously to the counting quantifier case of \cref{alg:buildnode} (which instead iterates over the edges incident to a node).
The graph-level \ROBDD is finally computed as $f \leftarrow C_{\mathrm{lo}} \wedge \neg C_{\mathrm{hi}+1}$, which returns true exactly if between $\mathrm{lo}$ and $\mathrm{hi}$ nodes in the input graph satisfy $\varphi$.

In both cases, to speed up the algorithms, we implement a lookup table, where we save each of the \ROBDD previously constructed, such that a given function is never implemented twice.

\begin{figure}[t]
    \centering
    \begin{minipage}[t]{0.45\textwidth}
    \begin{algorithm}[H]\footnotesize
    \caption{\proc{BuildNode}$(\varphi,v,G)$}
    \label{alg:buildnode}
    \begin{algorithmic}[1]
    \Require formula $\varphi$; node $v\in \calV$; graph $G=(\calV,\calE)$;
    \Require global variable $B$ indicating an \ROBDD object;
    \Ensure \ROBDD $f$ with $f(R)=\Ind{(R,v)\models\varphi}$ for any $R\subseteq G$
    \If{$\varphi = \top$}
      \State $f \gets 1$
    \ElsIf{$\varphi = \bot$}
      \State $f \gets 0$
    \ElsIf{$\varphi = \mathrm{Col}$, a colour}
      \State $f \gets \mathrm{Col}(v)$ 
    \ElsIf{$\varphi = \neg\psi$}
      \State $f \gets \neg\,\proc{BuildNode}(\psi,v,G)$
    \ElsIf{$\varphi = \bigwedge_{i}\psi_i$}
      \State $f \gets \bigwedge_{i} \proc{BuildNode}(\psi_i,v,G)$
    \ElsIf{$\varphi = \bigvee_{i}\psi_i$}
      \State $f \gets \bigvee_{i} \proc{BuildNode}(\psi_i,v,G)$
    \ElsIf{$\varphi = \Diamond_{\ge N}\,\psi$}
      \State $C_0 \gets 1$
      \State $C_k \gets 0$ for $k=1,\dots,N$
      \For{$u \in \mathcal{N}(v)$}
        \State $e \gets (v,u)$
        \State $t \gets x_e \wedge \proc{BuildNode}(\psi,u,G)$
        \For{$k = N,\dots,1$}
          \State $C_k \gets C_k \vee (C_{k-1} \wedge t)$ \label{line:edgecount}
        \EndFor
      \EndFor
      \State $f \gets C_N$
    \EndIf
    \State \Return $f$
    \end{algorithmic}
    \end{algorithm}
    \end{minipage}\hfill
    \begin{minipage}[t]{0.45\textwidth}
    \begin{algorithm}[H]\footnotesize
    \caption{\proc{BuildGraph}$(\Phi,G)$}
    \label{alg:buildgraph}
    \begin{algorithmic}[1]
    \Require formula $\Phi = \exists^{[\mathrm{lo}, \mathrm{hi}]}u.\,\varphi(u)$, with the counting quantifier
             standing for the readout bounds $\mathrm{lo} \le \mathrm{hi}$, so that
             $G \models \Phi$ iff $\mathrm{lo} \le |\{u : (G,u) \models \varphi\}| \le \mathrm{hi}$;
             graph $G=(\calV,\calE)$,
    \Ensure \ROBDD $f$ with $f(R)=\Ind{R\models\Phi}$ for any $R\subseteq G$
    \State $C_0 \gets 1$
    \State $C_k \gets 0$ for $k=1,\dots,|\calV|$
    \For{$v \in \calV$}
      \State $f_v \gets \proc{BuildNode}(\varphi,v,G)$ \label{line:reuse}
      \For{$k = |\calV|,\dots,1$}
        \State $C_k \gets C_k \vee (C_{k-1} \wedge f_v)$ \label{line:nodecount}
      \EndFor
    \EndFor
    \State $f \gets C_{\mathrm{lo}}$
    \If{$\mathrm{hi}+1 \le |\calV|$}
      \State $f \gets f \wedge \neg\,C_{\mathrm{hi}+1}$
    \EndIf
    \State \Return $f$
    \end{algorithmic}
    \end{algorithm}
    \end{minipage}
\end{figure}

\subsection{Unrolling the recursive rule in \cref{alg:Extend}}
\label[appendix]{appx:example-unrolling}

Consider a generic \ROBDD over a graph $G$, \ie a Boolean function mapping each subgraph $R$ of $G$ to a truth value, as described in \cref{sec:upi}.
\cref{alg:Extend} initialises $\fext$ to this function and then updates it in place, executing one conjunction per edge, each time against a copy of $\fext$ shifted by that edge.
Writing $\fext$ for the accumulated function and $f$ for the one the \ROBDD encodes before the first iteration, unfolding the assignments in order gives, at an arbitrary row $R$:

\begin{align*}
  \text{start}   &: & \fext(R) \;&=\; f(R) \\[2pt]
  \text{add }e_1 &: & \fext(R) \;&\leftarrow\; \fext(R)\wedge \fext(R\cup \{e_1\})
                          \;=\;  f(R)\wedge f(R\cup \{e_1\}) \\[2pt]
  \text{add }e_2 &: & \fext(R) \;&\leftarrow\; \fext(R)\wedge \fext(R\cup \{e_2\}) \\
                 & & &\phantom{\leftarrow}\;=\;
                     \bigl[ f(R)\wedge f(R\cup \{e_1\})\bigr]\wedge
                     \bigl[ f(R\cup \{e_2\})\wedge f(R\cup \{e_1\}\cup \{e_2\})\bigr] \\[2pt]
  \text{add }e_3 &: & \fext(R) \;&\leftarrow\; \fext(R)\wedge \fext(R\cup \{e_3\}) \\
                 & & &\phantom{\leftarrow}\;=\;
                     \bigl[ f(R)\wedge f(R\cup \{e_1\})\wedge f(R\cup \{e_2\})\wedge f(R\cup \{e_1\}\cup \{e_2\})\bigr] \\
                 & & &\phantom{\leftarrow\;=\;}\wedge\;
                     \bigl[ f(R\cup \{e_3\})\wedge f(R\cup \{e_1\}\cup \{e_3\}) \\
                 & & &\phantom{{\leftarrow\;=\;}\wedge\;\bigl[}
                     \wedge f(R\cup \{e_2\}\cup \{e_3\})\wedge f(R\cup \{e_1\}\cup \{e_2\}\cup \{e_3\})\bigr] \\[2pt]
  \dots
\end{align*}

Each step conjoins the \emph{accumulated} $\fext$ with a shifted copy of itself, so every term already present is duplicated with the new edge added and the clause doubles.
After all $|\calE|$ steps, the accumulated function is $\fext(R)=\bigwedge_{T\subseteq\calE}f(R\cup T)$; as $T$ ranges over the subsets of $\calE$, the set $R\cup T$ ranges over every superset of $R$, so $\fext(R)=1$ exactly when $R$ is sufficient in the sense of \cref{def:piexpl}.

\subsection{Computational complexity of \cref{alg:main}}
\label[appendix]{appx:complexity}

We proceed to analyse the time complexity of \cref{alg:main} in terms of the number of symbolic operations on the ROBDD. Given a graph $G$ and a formula $\Phi$, we first perform the construction of the \ROBDD. Consider a node-level formula $\Phi$ and a target node $v$, for which we are building the corresponding ROBDD. While $\land$, $\lor$, and $\neg$ all require a single operation on the \ROBDD to be built (lines $7$-$12$ of \cref{alg:buildnode}), the existential quantifier $\Diamond_{\ge N}$ can require up to $N|\calE_T|$ operations, with $\calE_T$ the receptive field of target node $v$ for formula $\Phi$ (lines $13$-$23$). Calling $N_{max}$ the largest $N$ between counting quantifiers in $\Phi$, the upper bound for the computational cost for building a node-level formula is then $O(|\calE| \cdot |\Phi| \cdot N_{max})$ symbolic operations on the \ROBDD, where $|\Phi| \defeq |\mathrm{sub}(\Phi)|$ is the number of subformulas of $\Phi$, as defined in \cref{sec:compiled-gnn}. For graph-level formulas (\cref{sec:compiled-gnn}), the node counting quantifier $\exists^{[\mathrm{lo}, \mathrm{hi}]}$ additionally requires $O(|\calV|^2)$ symbolic operations (lines $3$-$8$ of \cref{alg:buildgraph}) to be constructed. 

Then, \cref{alg:Extend} requires $O(|\calE|)$ symbolic operations to construct the function $\fext$: as shown in \cref{appx:example-unrolling}, for each edge in $\calE$ we modify $\fext$ iteratively by performing a symbolic operation with the \ROBDD ($\fext \leftarrow \fext \wedge \cof{\fext}{e}{1}$). \cref{alg:extract} similarly performs $O(|\calE|)$ symbolic operations to examine if each $e \in \calE$ is also in $\UPI$ (for each edge, it always performs a fixed number of symbolic operations on the \ROBDD, and then checks its satisfiability, which is an $O(1)$ operation when using \ROBDD). The final time complexity, for a single target of a graph classifier, is then $O(| \calE | \cdot |\Phi|\cdot N_{max}+|\calV|^2)$ symbolic operations.

The real bottleneck of our approach is that the time complexity of a symbolic operation on a \ROBDD depends on its size, which is in the worst case exponential in the number of variables \citep{bryant1986graph} which, in our case, is the number of edges $|\calE|$ of the input graph. However, in practice, the tasks and datasets we have chosen rarely fall into this worst case, and the $\mathrm{dd}$ Python package \citep{dd} provides an efficient implementation of the ROBDD.

We complement the analysis with a measurement of the actual cost of \cref{alg:main} over random graphs at varying size.

\textbf{Setup.}
For a given number of nodes $N$, we take $G$ to be the complete directed graph on $N$ nodes, self-loops included, so that the number of candidate subgraphs is $2^{N^{2}}$.
The classifier is the directed edge-cover function $f$ of \cref{ex:path}, i.e.\ $f(R)=1$ iff every node of $G$ retains at least one incoming edge in $R$.
We pick this family for two main reasons: its ground truth is $\UPI(f,G)=\calE$, so no instance degenerates into a trivial \ROBDD; and both $f$ and the sufficiency-extended function of \cref{alg:Extend} keep growing with $|\calE|$.

We report the wall-clock time of \emph{one} complete run of \cref{alg:main} for three \ROBDD engines exposed by the $\mathrm{dd}$ package \citep{dd}: the Cython-accelerated \CUDD backend with a static variable order, the same backend with dynamic variable reordering enabled, and the pure-Python \texttt{autoref} backend (no Cython/CUDD acceleration).
Alongside each time, we report the number of nodes in the \ROBDD to rack memory consumption.
The last column is the brute-force baseline, namely enumerating all $2^{|\calE|}$ subgraphs of $G$, querying the compiled GNN on each of them, representing a naive implementation of \cref{def:piexpl}.
All measurements are taken on a single core of an AMD Ryzen 9 7950X.
An engine exceeding a budget of $1$ second on a given size is not run on the larger ones and is marked \OOT. 
All engines are also asserted to return the same edge mask.

\textbf{Results.}
As reported in \cref{tab:cost-bdd-vs-brute}, the brute-force baseline is already at $48$ seconds on $N=5$ nodes and is out of reach immediately after, whereas \cref{alg:main} stays in the millisecond range over the whole range we tested, up to $|\calE|=144$ edges, i.e.\ $2^{144}$ subgraphs, with a diagram of a few tens of thousands of nodes.
The comparison between the two \CUDD columns is also informative: dynamic reordering never pays off here, and costs up to two orders of magnitude at the larger sizes, which indicates that our default ordering $\pi$ based on the edge index is already good enough.
Given the results, we adopt \CUDD with a static variable order as the default implementation in all our experiments.

\begin{table}[h!]
    \centering
    \small
    \begin{tabular}{rr rr rr rr r}
        \toprule
         & & \multicolumn{2}{c}{\CUDD, static $\pi$} & \multicolumn{2}{c}{\CUDD, dynamic $\pi$} & \multicolumn{2}{c}{\texttt{autoref}} & brute force \\
        \cmidrule(lr){3-4}\cmidrule(lr){5-6}\cmidrule(lr){7-8}\cmidrule(lr){9-9}
        $N$ & $N^2$ & time (s) & nodes & time (s) & nodes & time (s) & nodes & time (s) \\
        \midrule
         3 &   9 & 0.006 & $<$1k & 0.006 & $<$1k & 0.001 & $<$1k & \phantom{0}0.004 \\
         4 &  16 & 0.006 & $<$1k & 0.006 & $<$1k & 0.001 & $<$1k & \phantom{0}0.050 \\
         5 &  25 & 0.007 &    1k & 0.007 &    1k & 0.003 &    1k & 48.041 \\
         6 &  36 & 0.002 &    2k & 0.002 &    2k & 0.006 &    2k & \OOT \\
         8 &  64 & 0.002 &    7k & 0.022 &    7k & 0.020 &    7k & \OOT \\
        10 & 100 & 0.003 &   18k & 0.198 &   18k & 0.044 &   18k & \OOT \\
        12 & 144 & 0.006 &   38k & 0.502 &   36k & 0.096 &   38k & \OOT \\
        \bottomrule
    \end{tabular}
    \caption{
        Cost of extracting \UPI\ with \cref{alg:main} on complete graphs on $N$ nodes for the edge-cover classifier of \cref{ex:path}, against a brute-force baseline implementing \cref{def:piexpl}.
        Times are for one full extraction, \emph{nodes} is the number of nodes in the \ROBDD, and \OOT\ marks a configuration that was dropped after exceeding the budget of $1$ second on a smaller size.
    }
    \label{tab:cost-bdd-vs-brute}
\end{table}

\subsection{Cost of the ground truth against the cost of the explainers}
\label[appendix]{appx:cost-vs-explainers}

In this section, we compare the execution time of the ground truth extraction in \cref{alg:main} with the cost of running explainers.
As the results shall show, the total runtime is largely dominated by the explainers.

\textbf{Setup.}
We break the wall clock into the three stages of \cref{alg:main}: the construction of the \ROBDD, the sufficiency extension of \cref{alg:Extend}, and the final extraction of \cref{alg:extract}.
We report average and standard deviations over the first $200$ instances of each dataset, on a single core of an AMD Ryzen 9 7950X.
For explainers, we report the time that each explainer takes on average to explain a single explanation. In this analysis, we only report the timing for the most expensive explainers, as the others run at negligible cost.

\textbf{Results.}
As reported in \cref{tab:cost-datasets}, one extraction takes up to $0.15$ seconds across the whole benchmark. Note that this cost is paid once and then the ground truth explanations are stored together with the dataset.
For motif-based tasks, such as \REDBLUELEAF, or for simple dataset designs like \REDUNDANCY and \COMPREDUNDANCY, the algorithm can, in principle, be replaced by manual annotation of the target motif, therefore further reducing the cost.
This is possible because the formula is designed to recognize the target motif exclusively, with every relevant edge belonging to that motif by construction.
\cref{tab:cost-explainers} reports the numbers for explanation extraction, showing how the overall runtime of our experiments is largely dominated by the explainers.

\begin{table}[h!]
    \centering
    \small
    \setlength{\tabcolsep}{4pt}
    \resizebox{\textwidth}{!}{%
    \begin{tabular}{l rrr r rrrr}
        \toprule
         & & & & & \multicolumn{4}{c}{time (s)} \\
        \cmidrule(lr){6-9}
        Dataset & $|\calV|$ & $|\calE|$ & $|\calE_T|$ & Num. nodes & $\mathrm{BUILD}$ & \cref{alg:Extend} & \cref{alg:extract} & total \\
        \midrule
        \multicolumn{9}{l}{\textit{graph classification}} \\
        \REDBLUELEAF      &   32 &   98 &   98.0 & 4.6k (155k)  & \nentry{0.0124}{0.0185} & \nentry{0.0002}{0.0000} & \nentry{0.0003}{0.0000} & \nentry{0.0130}{0.0185} \\
        \REDUNDANCY       &   37 &  134 &  134.0 & 338 (338)    & \nentry{0.0009}{0.0001} & \nentry{0.0004}{0.0000} & \nentry{0.0005}{0.0000} & \nentry{0.0018}{0.0002} \\
        \COMPREDUNDANCY   &   37 &  134 &  134.0 & 342 (342)    & \nentry{0.0022}{0.0047} & \nentry{0.0004}{0.0000} & \nentry{0.0005}{0.0000} & \nentry{0.0030}{0.0047} \\
        \midrule
        \multicolumn{9}{l}{\textit{node classification}} \\
        \REDBLUELEAFNC    &  750 & 2192 & \phantom{0}56.0 & 14k (1.2M)  & \nentry{0.0041}{0.0306} & \nentry{0.0002}{0.0003} & \nentry{0.0003}{0.0004} & \nentry{0.0045}{0.0311} \\
        \MAJONEHOP        &  500 & 1992 & \phantom{0}\phantom{0}6.0 & 620 (43k)   & \nentry{0.0057}{0.0099} & \nentry{0.0000}{0.0003} & \nentry{0.0001}{0.0004} & \nentry{0.0058}{0.0105} \\
        \MAJTWOHOP        &  500 & 1992 & \phantom{0}73.7 & 109k (2.9M) & \nentry{0.1374}{0.1625} & \nentry{0.0062}{0.0212} & \nentry{0.0082}{0.0235} & \nentry{0.1517}{0.2024} \\
        \bottomrule
    \end{tabular}}
    \caption{
        Cost of extracting \UPI\ with \cref{alg:main} on the benchmark datasets.
        Times are wall-clock seconds for one complete extraction on a single explained instance (one graph for graph classification, one target node for node classification), split into the three stages of the algorithm, as mean $\pm$ standard deviation over the first $200$ explained instances of each dataset.
        $|\calV|$ and $|\calE|$ are the size of the graph, and $|\calE_T|$ is the average size of the receptive field of the target.
        The `\textit{Num. nodes}' column indicates the mean number of nodes in the \ROBDD, with the maximum reported in parentheses.
    }
    \label{tab:cost-datasets}
\end{table}

\begin{table}[h!]
    \centering
    \small
    \setlength{\tabcolsep}{6pt}
    \begin{tabular}{l rrr}
        \toprule
        & \multicolumn{3}{c}{time (s)} \\
        \cmidrule(lr){2-4}
        Dataset & \gnnexplainer & \graphmask & \subgraphx \\
        \midrule
        \multicolumn{4}{l}{\textit{graph classification}} \\
        \REDBLUELEAF     & \nentry{0.554}{0.380} & \nentry{0.609}{0.292} & \nentry{1.523}{0.409} \\
        \REDUNDANCY      & \nentry{0.066}{0.002} & \nentry{0.019}{0.002} & \nentry{42.464}{1.846} \\
        \COMPREDUNDANCY  & \nentry{1.001}{0.029} & \nentry{0.023}{0.001} & \nentry{44.986}{1.890} \\
        \midrule
        \multicolumn{4}{l}{\textit{node classification}} \\
        \REDBLUELEAFNC   & \nentry{0.323}{0.070} & \nentry{0.109}{0.012} & \nentry{4.502}{12.769}$^\dagger$ \\
        \MAJONEHOP       & \nentry{0.380}{0.027} & \nentry{0.678}{0.037} & \nentry{0.104}{0.555} \\
        \MAJTWOHOP       & \nentry{1.203}{0.101} & \nentry{4.113}{0.206} & \nentry{1904.211}{1805.311}$^\dagger$ \\
        \bottomrule
    \end{tabular}
    \caption{
        Wall-clock seconds to explain \emph{one} instance, per explainer and per dataset, averaged over random seeds.
        A $\dagger$ marks an explainer that did not finish every instance within the per-explanation time-out, so that its mean is computed on the instances that did finish.
    }
    \label{tab:cost-explainers}
\end{table}

\section{Choosing the right evaluation metric}
\label[appendix]{appx:metrics}

In this section we motivate our choice to report precision and recall separately, and scoring against the \emph{union} of the PIs rather than against each PI individually. We do this considering previous approaches:

\textbf{Why not plain Accuracy \citep{ying2019gnnexplainer}.}
On the tasks we consider, \UPI covers a small minority of the edges of a graph (c.f.\  \cref{tab:datasets}).
Accuracy would therefore be dominated by the negative class: the constant explanation that declares every edge irrelevant would score very high, while carrying no information whatsoever.

\textbf{Why not \AUROC \citep{luo2020parameterized} and Precision-Recall Area Under the Curve (\PRAUC) \citep{nerini2025true}.}
These two metrics measure different things, but they share one key pitfall: each compresses the precision/recall trade-off into a single scalar, and because \UPI is a union of alternative, individually valid explanations, that trade-off is informative.
Consider an explainer that returns one single PI, correctly and in full.
This is alone a desirable, high-quality explanation \citep{darwiche2020reasons, marques2022delivering, bassan2025briefly, azzolin2026degenerate}, as everything it selects is relevant and minimal.
Its recall, however, is $|PI| / |\UPI|$, which can be arbitrarily small.
Hence, the \PRAUC will also be small.
Since such an explainer has settled on one explanation, it has no reason to rank the edges of the \emph{other} PIs any higher than the irrelevant ones, and therefore \AUROC also drops.
In every case, the fact that each and every edge the explainer selected is correct may get overwhelmed by the fact that alternative valid explanations exist when the explainer preferred to focus on just one.

\textbf{Why not a weighted intersection over union (WIoU) \citep{azzolin2025reconsidering}.}
A further alternative is to compare the two masks with a weighted intersection over union,
$$
\mathrm{WIoU}(\mu, \mu^*) = \frac{\sum_{e \in \calE} \min(\mu_e, \mu^*_e)}{\sum_{e \in \calE} \max(\mu_e, \mu^*_e)},
$$
which has the appeal of requiring no threshold.
The problem is that \textrm{WIoU} may charge the explanation for the total score mass placed on irrelevant edges, which however accumulates with the size of the graph even when the ranking of the edges is perfect.
To see this, consider a task whose ground truth is a single edge per graph, and an explainer that assigns importance $1.0$ to that edge and a negligible $0.001$ to all others, \ie a perfect ordering.
On a graph with $|\calE| = 10$ we get $\mathrm{WIoU} = 1 / (1 + 9\cdot 0.001) = 0.991$, whereas on a graph with $|\calE| = 5000$ we get $1 / (1 + 4999\cdot 0.001) = 0.167$.
This shows that the same explainer, with the same ranking, is scored an order of magnitude apart, meaning scores are not comparable across graph size.

\textbf{Why the union of the PIs, and not the disjoint union \citep{agarwal2023evaluating}.}
Evaluating against every single PI individually would require retaining one explanatory mask per PI and scoring the explanation against the best matching one, similarly to what is proposed in \citet{agarwal2023evaluating}.
We avoid this for two reasons:
First, the number of PIs can grow exponentially in $|\calE|$ so both storing the masks and matching against them may become intractable;
Second, a best-matching-PI score is easy to inflate, as the explainer returning a single PI is awarded a perfect score while telling the user nothing about how well the explainer covers all the other alternative explanations.

In summary, our design choice of tracking precision and recall separately allows us to measure: \textit{i)} how well the explainer recovers exactly the relevant elements, via the precision score; \textit{ii)} how well it can cover the range of all possible alternative explanations, via the recall score.

Two additional caveats are of interest: 
\textit{First,} we note that the $F_1$ reported in \cref{tab:explainers} and \cref{tab:explainers-nc-majority} does not match the precision and recall columns; this is because the $F_1$ column is computed as the average of sample-wise precisions and recalls, and the average of $F_1$ scores does not necessarily match the $F_1$ of the averages.
\textit{Second,} by sweeping over all possible thresholds, we are providing an \textit{optimistic} estimate of precision, recall, and $F_1$.

\section{Dataset details}
\label[appendix]{appx:datasets}

Throughout, $\Diamond_{\geq N}\,\varphi$ reads ``at least $N$ neighbours satisfy $\varphi$'' and $\Diamond_{=N}\,\varphi$ ``exactly $N$ of them'', $\top$ is the always-true predicate (so $\Diamond_{=1}\top$ is ``degree exactly one''), and $\mathrm{Col}(u)$ is true when node $u$ carries colour $\mathrm{Col}$.
For graph classification tasks, every task is specified by a node predicate $\varphi$ together with the graph-level quantifier $\Phi = \exists u.\,\varphi(u)$ that acts as the graph-level readout.
We call \emph{palette} the ordered list of colours $\calC$ a task uses.
It plays two roles at once: it is the vocabulary out of which the colour predicates $\mathrm{Col}(u)$ are drawn, and it is the node feature space, as a node carrying colour $\mathrm{Col}$ is encoded by the indicator vector $\vc_v$, here one-hot, selecting the position of $\mathrm{Col}$ in the palette.

\textbf{Common protocol.}
Every dataset contains $1000$ graphs with a class ratio of $0.5$, built on a Barabási--Albert background with attachment parameter $2$ \citep{barabasi1999emergence}, of $20$ nodes for \REDBLUELEAF and of $30$ nodes for \REDUNDANCY and \COMPREDUNDANCY. Every background node receives a colour drawn uniformly from the colours of the palette that are reserved for the background in the task.
At generation time, we apply rejection sampling against $\varphi$, so that a candidate graph is kept only when $\varphi$ fires exactly where the annotation says it should, i.e., only inside the motif for a positive graph, nowhere at all for a negative one. This guarantees that the annotation is unambiguous.

\textbf{\REDBLUELEAF.}
The palette is $[\mathrm{red}, \mathrm{blue}, \mathrm{green}]$ and the motif is a chain of four red nodes, each carrying exactly two blue leaves, for $12$ nodes and $11$ edges in total.
The motif is attached to the background by linking the two chain endpoints to two random background nodes, which are recoloured green first.
This is done because a green anchor is neither a red core nor a blue leaf, so the attachment cannot perturb the local counts of the motif.
The two attachment edges are not part of the ground truth, which is exactly the $11$ motif edges. If the motif is attached, the graph is annotated as of positive class; otherwise, of negative one.
The underlying formula is built bottom-up from the local structure of the motif,
\begin{align*}
    \mathrm{leaf}(u) &\;=\; \mathrm{blue}(u)\wedge\Diamond_{=1}\top
        & &\text{a blue node of degree one,} \\
    \mathrm{core}(u) &\;=\; \mathrm{red}(u)\wedge\Diamond_{=2}\,\mathrm{leaf}
        & &\text{a red node with exactly two blue leaves,} \\
    \mathrm{end}(u)  &\;=\; \mathrm{core}(u)\wedge\Diamond_{=1}\,\mathrm{core} & & \text{endpoint of the chain,}\\
    \mathrm{mid}(u)  &\;=\; \mathrm{core}(u) \wedge\Diamond_{=2}\,\mathrm{core}
        & &\text{interior node of the chain,}\\
    \mathrm{midend}(u)  &\;=\; \mathrm{mid}(u) \wedge\Diamond_{\geq 1}\,\mathrm{end}
        & &\text{interior node touching an end point,}
\end{align*}
and closed with the following node predicate and graph-level quantifier,
\[
    \varphi(u) \;=\; \mathrm{midend}(u) \land \Diamond_{\geq 1}\, \mathrm{midend} ,
    \qquad \Phi \;=\; \exists u.\,\varphi(u),
\]
which fires on the two interior nodes of the chain.
The reason for this construction is to avoid counts by backtracking: the naive ``four red cores next to each other'' is satisfied even by a chain of length $2$, since a walk can simply bounce between two adjacent red cores.
%
%
%
We further provide in \cref{appx:compilation-choices} an ablation study on the impact of how the formula above is written on the plausibility evaluation.

\textbf{\REDUNDANCY.}
The palette is $[\mathrm{black}, \mathrm{red}, \mathrm{grey}, \mathrm{silver}]$, where the last two shades are reserved for the background. The node predicate, and the graph-level quantifier closing it, are
\[
    \varphi(u) \;=\; \mathrm{black}(u) \wedge \Diamond_{\geq 1}\,\mathrm{red} ,
    \qquad \Phi \;=\; \exists u.\,\varphi(u) ,
\]
i.e., `a black node with at least one neighbour of colour $\mathrm{red}$'.
A positive graph carries exactly one black node, together with $6$ red neighbours and further connected to the background by $5$ edges to random background nodes, which are the irrelevant edges an explainer has to reject.
Since no background node carries the anchor colour, the planted anchor is the only node that can ever satisfy $\varphi$.
Note that any single one of the $6$ edges between black and red nodes constitutes a valid PI explanation. Hence, their union \UPI is composed of the whole group of $6$ edges.

\textbf{\COMPREDUNDANCY.}
The palette is $[\mathrm{black}, \mathrm{red}, \mathrm{blue}, \mathrm{green}, \mathrm{grey}, \mathrm{silver}]$ and the node predicate, with the graph-level quantifier closing it, is
\[
    \varphi(u) \;=\; \mathrm{black}(u) \wedge \big( \Diamond_{\geq 1}\,\mathrm{red} \;\vee\; \Diamond_{\geq 2}\,\mathrm{blue} \;\vee\; \Diamond_{\geq 3}\,\mathrm{green} \big) ,
    \qquad \Phi \;=\; \exists u.\,\varphi(u) .
\]
The construction follows the one above, except that the black node now contains three groups of candidate neighbours in three different colours, planted with $1$, $2$ and $3$ members respectively, and each group has to be complete for its disjunct to be true.
A negative graph starves every group in the same way as above, planting one member fewer than its threshold.
Differently from \REDUNDANCY, there are exactly $3$ alternative explanations, one per colour-group. \UPI is their union, meaning $6$ edges in total, the same as in \REDUNDANCY.

\textbf{Common protocol for node classification.}
A node-classification task has no graph-level quantifier: the formula $\Phi(u)$ itself labels the target node, so what the explainer scores is the receptive field of the target rather than the whole graph.
Each of these datasets is composed of several Barabási--Albert graphs with attachment parameter $2$ \citep{barabasi1999emergence} and palette $[\mathrm{red}, \mathrm{blue}, \mathrm{green}]$; we randomly generate and accumulate these graphs until we collect at least $500$ nodes on which $\varphi$ fires. We do so in order to limit the size and the density of the receptive field, and therefore the time required to extract the ground truth for each target node.
Below, $\Diamond_{>}(\varphi,\psi)$ is the comparative operator of \cref{appx:desugaring}, and it reads `strictly more neighbours satisfy $\varphi$ than $\psi$' (see \cref{appx:desugaring} for the details regarding its implementation).

\textbf{\MAJONEHOP.}
A single graph is $500$ nodes, each coloured uniformly at random from the palette $[\mathrm{red}, \mathrm{blue}, \mathrm{green}]$, and the node predicate is
\[
    \Phi(u) \;=\; \Diamond_{>}(\mathrm{red}, \mathrm{blue}) ,
\]
i.e., `strictly more red neighbours than blue ones'.
A node can be assigned no colour, in which case we refer to it as of grey colour.

\textbf{\MAJTWOHOP.}
Same graph and same construction as \MAJONEHOP, with the predicate applied twice \citep{nerini2025true},
\[
    \mathrm{red\text{-}ish}(u) \;=\; \Diamond_{>}(\mathrm{red}, \mathrm{blue}) ,
    \,\,
    \mathrm{blue\text{-}ish}(u) \;=\; \Diamond_{>}(\mathrm{blue}, \mathrm{red}) ,
    \,\,
    \Phi(u) \;=\; \Diamond_{>}(\mathrm{red\text{-}ish}, \mathrm{blue\text{-}ish}) .
\]
The inner predicates give every node a derived label, and the outer one compares those derived labels across the neighbourhood, which is exactly two rounds of label propagation.
The receptive field grows accordingly, from one hop to two, and with it the cost of extracting \UPI (\cref{tab:cost-datasets}).

\textbf{\REDBLUELEAFNC.}
This is the node-level counterpart of \REDBLUELEAF, built from the same vocabulary,
\begin{align*}
    \mathrm{leaf}(u) &\;=\; \mathrm{blue}(u) \wedge\Diamond_{=1}\top
        & &\text{a blue node of degree one,} \\
    \mathrm{core}(u) &\;=\; \mathrm{red}(u) \wedge\Diamond_{=2}\,\mathrm{leaf}
        & &\text{a red node with exactly two blue leaves,}
\end{align*}
and closed into the classifier
\[
    \Phi(u) \;=\; \mathrm{core}(u) \wedge \Diamond_{\geq 1}\,\mathrm{core} ,
\]
which fires on a red core adjacent to another red core, and therefore needs three hops to be decided: the neighbour's colour, that neighbour's blue leaves, and the degree of each leaf. This represents a minor modification with respect to \REDBLUELEAF: instead of distinguishing between $\mathrm{end}$, $\mathrm{mid}$ or $\mathrm{midend}$ nodes, we simplified the formula in order to have multiple nodes firing for each planted motif. 
On a background of $300$ nodes we plant $50$ units, each a path of three red cores carrying two blue leaves apiece, whose two endpoints are tied to two random background nodes recoloured green first, for $750$ nodes in total. Because the model uses a 3-hop receptive field, even the central core node integrates multiple background edges outside the motif due to the pair of anchor points.
As in \REDBLUELEAF, green is inert.
Note that the planted unit and the formula deliberately do not coincide: in a path of three cores, the middle core and the two endpoint cores fire for structurally different reasons, so targets inside the same unit have different explanations.

\begin{table}[h!]
    \centering
    \small
    \begin{tabular}{lccccc}
        \toprule
        Dataset & graphs & nodes & edges & \UPI\ edges & \UPI\ fraction \\
        \midrule
        \REDBLUELEAF    & 1000 & 32 & \phantom{0}98.0 & 22.0 & 0.22 \\
        \REDUNDANCY     & 1000 & 37 & 134.0 & \phantom{0}6.0 & 0.04 \\
        \COMPREDUNDANCY & 1000 & 37 & 134.0 & \phantom{0}6.0 & 0.04 \\
        \midrule
        \REDBLUELEAFNC  & \phantom{00}4 & 750 & 2192.0 & 10.7 & 0.29 \\
        \MAJONEHOP      & 3--4 & 500 & 1992.0 & \phantom{0}2.3 & 0.60 \\
        \MAJTWOHOP      & 3--4 & 500 & 1992.0 & 17.1 & 0.43 \\
        \bottomrule
    \end{tabular}
    \caption{
        Size of the datasets, averaged over the positive samples of the $5$ evaluation seeds.
        Nodes and edges are counted per graph, and edges are counted as directed, \ie an undirected edge contributes two.
        For node classification we keep sampling graphs until the dataset holds at least $500$ explainable targets, so the number of graphs is a range over the seeds rather than a fixed size.
        The last column is the fraction of edges that the explainers are asked to recover. For node classification, the denominator is the receptive field of the target.
    }
    \label{tab:datasets}
\end{table}

\section{Additional Details on the Survey}
\label[appendix]{app:survey}

In this section, we describe the procedure by which we collected the relevant information from previous literature summarized in \cref{fig:eval-papers}.
Our literature search was restricted to peer-reviewed venues, specifically CORE $A^*$ conferences, leading journals, and key domain-specific conferences. The complete set of targeted venues comprises:
\begin{itemize}[leftmargin=1.25em]
    \item \textit{Conferences}: AAAI, ACL, AISTATS, CVPR, FAccT, ICDM, ICDE, ICLR, ICML, IJCAI, KDD, LOG, NeurIPS, VLDB, and WWW.
    \item \textit{Journals}: ACM Computing Surveys, ACM Transactions on Intelligent Systems and Technology (ACM TIST), Proceedings of the ACM on Programming Languages (PACMPL), Nature Communications, Scientific Data, Transactions on Machine Learning Research (TMLR), and IEEE Transactions on Pattern Analysis and Machine Intelligence (TPAMI)
\end{itemize}
We intentionally included LOG and FAccT given their direct domain relevance to graph machine learning and algorithmic transparency.

%
We discarded works targeting non-standard explanation modalities (e.g., global decision boundaries \citep{wang2024gnnboundary}) where faithfulness and plausibility metrics cannot be meaningfully applied. 
%
%
We divided each entry according to the type of evaluation used in the paper: plausibility, faithfulness, both, or none of them (a minority of papers used ad-hoc metrics or entirely different objectives). For each paper, we annotate the following information: type of paper (whether it is a methodological paper, a survey, a benchmark, a critique, applied, or a mix thereof), whether the authors make sure the underlying model has learned the intended ground truth, the type of explanation provided (global vs local, post-hoc vs ante-hoc), the target tasks (node classification, graph classification, etc.), and the datasets used in the evaluation.

\section{Implementation details}
\label[appendix]{appx:impl-details}

\begin{table}[h!]
    \centering
    \small
    \caption{\textbf{Hyper-parameter search space of each explainer}, shared across all tasks. $[a,b]_{\log}$ denotes a log-uniform range, $[a,b]$ a uniform one, and $\{\cdot\}$ a categorical choice. The other explainers are run in their default configuration.}
    \label{tab:hparams}
    \begin{tabular}{@{}l p{0.76\linewidth}@{}}
        \toprule
        Explainer & Hyper-parameters and ranges \\
        \midrule
        \gnnexplainer & \texttt{epochs} $\in\{50, 100, 1000\}$; \texttt{lr} $\in[10^{-4}, 10^{-1}]_{\log}$; \texttt{edge\_size} $\in[10^{-4}, 10^{-1}]_{\log}$; \texttt{edge\_ent} $\in[10^{-4}, 10^{-1}]_{\log}$ \\
        \addlinespace
        \graphmask & \texttt{epochs} $\in\{5, 10, 25, 50\}$; \texttt{lr} $\in[10^{-5}, 10^{-2}]_{\log}$; \texttt{penalty\_scaling} $\in\{1,\dots,10\}$; \texttt{allowance} $\in[0, 0.3]$ \\
        \addlinespace
        \subgraphx & \texttt{max\_nodes} $\in\{5,\dots,20\}$; \texttt{min\_atoms} $\in\{3,\dots,\texttt{max\_nodes}\}$; \texttt{sample\_num} $\in\{10, 20, 50\}$; \texttt{rollout} $\in\{3, 5, 10, 20\}$; \texttt{expand\_atoms} $\in\{2,\dots,10\}$; \texttt{local\_radius} $\in\{1, 2, 3\}$ \\
        \addlinespace
        \pgmexplainer & \texttt{num\_samples} $\in\{50, 100, 200, 500\}$; \texttt{perturbation\_mode} $\in\{\texttt{randint}, \texttt{zero}, \texttt{mean}, \texttt{max}\}$; \texttt{significance\_threshold} $\in[10^{-3}, 10^{-1}]_{\log}$; \texttt{pred\_threshold} $\in[0, 0.5]$; \texttt{is\_perturbation\_scaled} $\in\{\texttt{false}, \texttt{true}\}$; \texttt{max\_subgraph\_size} $\in\{1,\dots,50\}$ (graph tasks only) \\
        \addlinespace
        \graphsvx & \texttt{num\_samples} $\in\{100, 200, 400, 800\}$; \texttt{S} $\in\{1,\dots,4\}$; \texttt{args\_coal} $\in\{\texttt{Smarter}, \texttt{Smart}, \texttt{Random}\}$; \texttt{args\_g} $\in\{\texttt{WLS}, \texttt{WLR\_sklearn}\}$; \texttt{use\_abs} $\in\{\texttt{true}, \texttt{false}\}$ \\
        \bottomrule
    \end{tabular}
\end{table}

\textbf{\textit{Hyperparameters tuning of explainers}}:
Hyperparameters are tuned for every task with Optuna \citep{akiba2019optuna} maximising the mean $F_1$ on a dataset generated with a held-out seed, so that no explainer is tuned on the instances it is later scored on. We run $50$ trials per explainer and task, each scored on $50$ instances, and report the searched ranges in \cref{tab:hparams}.

\textbf{\textit{Explainer time budget}}:
Given that some explainers are particularly computationally expensive, we set a time budget of $180$ seconds for explaining individual instances.
During the evaluation of explainers in \cref{sec:experiments}, whenever an instance takes more than $180$ seconds to be explained, the process is halted, and no further samples are explained for that specific seed.
When an explainer does not finish any of the random seeds, we report `\textit{OOT}' in the corresponding table entry.

\subsection{Implementation details of explainers}

For some explainers, we performed some modifications to their implementation in order to adapt them to our evaluation setting. We detail each of the modifications in the following paragraphs, along with the corresponding justification. In every case, these modifications are applied so as not to change the original scope or functionality of the explainer.
Implementations of explainers are taken from PyG \citep{fey2025pyg}, except when otherwise specified.

\textbf{SubgraphX.} \citep{yuan2021explainability} We adopt the implementation provided in \citet{agarwal2023evaluating}\footnote{\url{https://github.com/mims-harvard/GraphXAI}}. For node tasks, after computing the explanation, this implementation applies a filter and return only the attribution on the computational graph of the node. Since it is the only explainer to do so, we removed this filter and compared the explanation on the full graph instead, to make the evaluation identical for all explainers. Because underlying computation of the explainer is inherently restricted to the computation graph, edges falling outside of it are simply zero-padded. Note that this modification does not affect execution time, and the explainer infers the depth of the computation graph directly from the number of message-passing layers in the target model.

\textbf{GNNExplainer.} \citep{ying2019gnnexplainer} GNNExplainer scales every message between node $u$ and $v$ by $\sigma(m_{u,v})$ for a learned logit $m_{u,v}$, and the default initialisation sets $m_{u,v}\approx 0$, so that each message is halved.
We noticed that this can create instabilities for some datasets (like \REDBLUELEAF) in the training of the logits, as halving the count of incoming messages can greatly change the total count of the neighbours satisfying a given sub-formula, even at initialization time. To reduce this effect, we use the bias term introduced in the original implementation\footnote{We note that this bias term is only present in the original codebase available at \url{https://github.com/RexYing/gnn-model-explainer}, and not in the implementation provided in PyG \citep{fey2025pyg}.} provided by \citet{ying2019gnnexplainer}, such that logits are initialized at values around $\sigma(m_{u,v})\sim1$. In this way, the count of satisfying neighbours is roughly preserved at initialisation time. Contrary to the original implementation, this bias term is only used for initializing the logits, and is not injected at every forward step. Then, during training, GNNExplainer will attempt to sparsify the logits as much as possible, following the original implementation. 
We verified that the value of this bias does not drive our insight regarding the sparsification ability of \gnnexplainer described in \cref{sec:existing-insights} by sweeping the additional bias term over $\{0, 1, 2.5, 5\}$ on \REDUNDANCY, where $0$ recovers the unbiased initialisation.
Precision, recall and $F_1$ all remain equal to $1.00\pm0.00$ over the same $5$ seeds, exactly the value reported in \cref{tab:explainers}.

\textbf{\flowx.} \citep{gui2023flowx} We adopt the original \flowx implementation\footnote{\url{https://github.com/CM-BF/FlowX}}. \flowx is a decomposition explainer, which decomposes the contribution of each part of the graphs through \textit{flows}. Besides the standard method, which we adopt, the authors of the paper propose variations that have an additional training step to refine the original explanations. 
We observed that the original implementation of the baseline method had a training loop of a single iteration, inherited from its variations, which did not perform any actual training. Moreover, inside the loop, there were additional operations which produced OOM errors when running \flowx on models with more than three message-passing layers. As such, we modified these operations and removed the calculation of the gradients in order to avoid OOM issues. We compared the original and the new version of the explainer on a single seed of the tasks \MAJTWOHOP and \REDBLUELEAFNC in \cref{tab:flowx_comparison}, showing no distinguishable behaviour besides a small rounding error on the recall in \MAJTWOHOP.

\begin{table}[h!]
    \centering
    \small
    \caption{Comparison in performance of the new implementation of \flowx, reporting mean $\pm$ standard deviation for explanations across sample nodes, for a single seed.}
    \label{tab:flowx_comparison}
    \begin{tabular}{l rr rr}
        \toprule
         & \multicolumn{2}{c}{\MAJTWOHOP} & \multicolumn{2}{c}{\REDBLUELEAFNC} \\
        \cmidrule(lr){2-3}\cmidrule(lr){4-5}
        Explainer & precision & recall & precision & recall \\
        \midrule
        \flowx        & $\nentry{0.888}{0.150}$ & $\nentry{0.942}{0.103}$ & $\nentry{0.994}{0.028}$ & $\nentry{0.998}{0.014}$ \\
        \flowx (original) & $\nentry{0.888}{0.150}$ & $\nentry{0.943}{0.103}$ & $\nentry{0.994}{0.028}$ & $\nentry{0.998}{0.014}$ \\
        \bottomrule
    \end{tabular}
\end{table}

\textbf{\pgmexplainer.} \citep{vu2020pgm} We adapted the implementation of \pgmexplainer provided in \textsc{PyG} \citep{fey2025pyg}. Originally, \pgmexplainer provides as the explanation the $p$-value scores associated with each node for its statistical relevance, together with a binary mask on the $k$ nodes with \textit{lowest} $p$-value (where $k$ is a parameter of the explainer). Since in this work we evaluate exclusively edge-level explanations, we assigned to each edge $e=(u,v)$ an importance score $1 - \max(p_u, p_v)$: in this way, an edge is important only as much as its less important endpoint.

\textbf{\graphsvx.} \citep{duval2021graphsvx} We adopt the original \graphsvx implementation\footnote{\url{https://github.com/AlexDuvalinho/GraphSVX}}. \graphsvx is a perturbation-based explainer that performs attribution for each node in the $k$-hop neighbourhood of the explained node and for each of its features. Since in this work we evaluate exclusively edge-level explanations, we disable the feature players so that the whole coalition budget is spent on the graph structure, and, as for \pgmexplainer, we assign to each edge $e=(u,v)$ the score $\min(|\phi_u|, |\phi_v|)$, where $\phi$ are the estimated Shapley values; the explained node, which is not a player, receives the maximum score of its neighbourhood. Moreover, the original implementation assumes dense adjacency matrices for graph classification tasks, incompatible with the sparse encoding of \textsc{PyG}; we adapted the code to work with sparse matrices as well, keeping the original coalition sampling and regression. Because it assigns equal importance to both edge directions, its maximum achievable precision is bounded at $50\%$ on \MAJONEHOP, \REDUNDANCY, and \COMPREDUNDANCY, as relevant edges in these datasets contribute exclusively in one direction. Nonetheless, across all these tasks, the explainer reaches a precision of $\approx50\%$, implying that it successfully localizes relevant nodes, but its inherent design limitations prevent it from distinguishing fine-grained edge directionality. Similar reasoning applies to \MAJTWOHOP, where the explainer reaches $56\%$ precision: although both directions of an edge can be important, typically only one contributes, meaning its theoretical maximum is only marginally higher. The same theoretical cap applies to \pgmexplainer and \subgraphx as well, though both achieve significantly lower precision than \graphsvx in practice.

\textbf{\lrp.} \citep{bach2015pixel} We adopt the \lrp implementation provided in \textsc{Captum} \citep{kokhlikyan2020captum}, through the \textsc{Captum} interface of \textsc{PyG}. \lrp propagates the model output backwards through the network according to layer-specific rules.
We adopted the baseline epsilon rule for all the layers of our compiled models, the same rule \textsc{Captum} applies by default.

\section{Proof of \cref{thm:extract}}
\label[appendix]{appx:proof-extract}

\algupi*
\begin{proof}
    We will analyze the two directions separately:

    \textbf{($\leftarrow$)}
    
    We aim to show that if $e \in \UPI$, then $\fext \wedge x_e \wedge \neg\cof{\fext}{e}{0} \; \ne 0$.
    Then, by assumption, $e$ belongs to a PI, say $M$.
    We note that $M$ is by itself a witness of $\fext \wedge x_e \wedge \neg\cof{\fext}{e}{0} \; \ne 0$:
    \begin{itemize}
        \item $M$ is sufficient by definition, hence $\fext(M) = 1$
        \item $e \in M$ by assumption, hence $x_e=1$
        \item $M \setminus \{e\}$ is not sufficient, since by definition of PI there exists no subset of $M$ that is also sufficient, hence $\cof{\fext}{e}{0}$ evaluates to $0$ at $M$, and $\neg\cof{\fext}{e}{0}$ to $1$
    \end{itemize}
    Hence, $\fext \wedge x_e \wedge \neg\cof{\fext}{e}{0} \; \ne 0$ since $M$ is a valid witness.

    \textbf{($\rightarrow$)}

    We aim to show that if $\kappa \defeq (\fext \wedge x_e \wedge \neg\cof{\fext}{e}{0}) \; \ne 0$, then $e \in \UPI$.
    The fact that $\kappa \ne 0$ means there must exist a subgraph $R$ such that the following hold together:
    \begin{itemize}
        \item[(H1)] $R$ is sufficient,
        \item[(H2)] edge $e$ is part of $R$,
        \item[(H3)] $R \setminus \{e\}$ is not sufficient.
    \end{itemize}
    Let us define $D \defeq \{S: S \subseteq R \land \fext(S)=1\}$. Note that $D \ne \emptyset$ as $R \in D$.
    Since $D$ is finite, there must be an element $R'$ with the smallest cardinality. Since $R'$ has the smallest cardinality, it is also the smallest sufficient subgraph in $D$, hence $R'$ is a PI.

    We now want to show that $e$ must be in $R'$, meaning $e \in \UPI$.
    Let us assume to the contrary that $e \not \in R'$. Then, $R' \subseteq R \setminus \{e\}$. However, $R'$ is sufficient since it is a PI, and therefore any of its supergraphs must also be sufficient, including $R\setminus \{e\}$.
    This contradicts the hypothesis (H3), meaning it was absurd to assume $e \not \in R'$. Hence, $e \in \UPI$.
\end{proof}

Note that the restriction to positive graphs is without loss of generality.
For a negative graph, \proc{UnionPI} runs \proc{Extend} on $\neg f$ (line~\ref{line:ok} of \cref{alg:main}), for which $G$ is positive; since \cref{def:piexpl} compares $g(R')$ with $g(G)$, a subgraph is sufficient for $\neg f$ iff it is for $f$, hence $\UPI(\neg f, G)=\UPI(f, G)$.
An interesting edge case of \cref{thm:extract} is whenever \UPI is empty. Note that this can only happen when the prediction on graph $G$ is invariant to the removal of any subset of edges, meaning $\emptyset$ is a PI. Cases where this happens are negative instances of \REDUNDANCY, or tasks like \textit{the graph has an isolated node}. The following Corollary shows that \cref{alg:extract} will return precisely an empty mask in this case:

\begin{corollary}
\label{cor:well-behaved}
    Under the same assumptions of Theorem~\ref{thm:extract}:
    \begin{equation}
    \label{eq:corollary}
        \big ( \forall e\in \mathcal{E}.\, \fext \wedge x_e \wedge \neg\cof{\fext}{e}{0} \; = 0 \big ) \iff \UPI(f, G) = \emptyset.
    \end{equation}
\end{corollary}
\begin{proof}
    Both directions are immediately proved by contradiction with Theorem~\ref{thm:extract}.
    On one hand, let us assume $\UPI(f, G) = \emptyset$ to be true.
    Assume for the sake of contradiction that \cref{eq:corollary} is false, meaning that, $ \exists e \in \calE . \fext \wedge x_e \wedge \neg\cof{\fext}{e}{0} \; \ne 0$, then by \cref{thm:extract} $\UPI(f, G)$ must be non empty, leading to a contradiction.
    On the other hand, let us assume the left side of \cref{eq:corollary} to be true.
    Assume for the sake of contradiction that the right side is false, meaning that $\UPI(f, G) \ne \emptyset$, then there must exist an $e \in \UPI(f, G)$: but by \cref{thm:extract}, this implies $\fext \wedge x_e \wedge \neg\cof{\fext}{e}{0} \; \ne 0$, which represents a contradiction.
\end{proof}

\end{document}